\documentclass{article} 
\usepackage{iclr2027_conference,times}

\usepackage{amsmath,amsfonts,bm}

\def\eqref#1{equation~\ref{#1}}

\def\1{\bm{1}}

\DeclareMathAlphabet{\mathsfit}{\encodingdefault}{\sfdefault}{m}{sl}
\SetMathAlphabet{\mathsfit}{bold}{\encodingdefault}{\sfdefault}{bx}{n}

\usepackage{hyperref}
\usepackage{url}
\usepackage{xspace}

\usepackage[utf8]{inputenc} 
\usepackage[T1]{fontenc}    
\usepackage{hyperref}       
\usepackage{url}            
\usepackage{subcaption}
\usepackage{makecell}
\usepackage{booktabs}       
\usepackage{amsfonts}       
\usepackage{nicefrac}       
\usepackage{microtype}      
\usepackage{xcolor} 
\usepackage{algorithm}
\usepackage{algorithmic}
\usepackage{amsmath}
\usepackage{amssymb}
\usepackage[toc,page,header]{appendix}
\usepackage{minitoc}

\usepackage{hyperref}
\usepackage{amsthm}
\usepackage{graphicx}
\usepackage{booktabs}

\usepackage{minitoc}
\usepackage{multirow}
\usepackage{hyperref}

\definecolor{gain}{RGB}{0,128,0}
\definecolor{loss}{RGB}{180,0,0}

\iclrfinalcopy 

\usepackage{orcidlink}
\newcommand{\ourmodel}
{TopoEnhance\xspace}
\newcommand{\best}[1]{\textcolor[HTML]{4e4fd1}{\textbf{#1}}}

\title{Refine Connections, Close the Gap: A Reliable Enhancement Framework for Driving Scene Topology}

\author{%
Xiaoqi Wang$^{1,2}$\thanks{Equal contribution.}\qquad
Dingyi Zhuang$^{3}$\footnotemark[1]\hspace{0.5em}\thanks{This work was done during an internship at Bosch.}\qquad
David Paz$^{1,2}$\qquad
Wenbin He$^{1,2}$\qquad
Yucai Bai$^{4}$\\
\textbf{Peng Zhou}$^{4}$\qquad
\textbf{Rui Zhang}$^{4}$\qquad
\textbf{Jinhua Zhao}$^{3}$\qquad
\textbf{Liu Ren}$^{1,2}$\\[0.2em]
$^{1}$Bosch Research North America \\
$^{2}$Bosch Center for Artificial Intelligence (BCAI) \\
$^{3}$Massachusetts Institute of Technology \\
$^{4}$Bosch XC China\\
[0.2em]
\texttt{\{xiaoqi.wang,david.pazruiz,wenbin.he2,liu.ren\}@us.bosch.com} \\
\texttt{\{dingyi,jinhua\}@mit.edu}\\
\texttt{\{yucai.bai,peng.zhou3,frank.zhang2\}@cn.bosch.com}
}

\begin{document}

\maketitle
\lhead{Preprint}

\doparttoc
\faketableofcontents
\begin{abstract}
In autonomous driving, understanding scene topology—the connectivity between lanes and traffic elements—is critical for safe path planning and motion control. While current methods excel at detecting individual map elements, their connectivity reasoning often falls short of its theoretical potential, leaving a significant performance gap relative to the theoretical upper-bound (Oracle) achievable given the underlying detections. To bridge this gap, we propose \ourmodel, a novel topology enhancement framework designed to unlock the latent potential of existing methods by refining their connectivity predictions. We formulate topology enhancement as a denoising-based reconstruction process, where the model learns to recover structural consistency from stochastically corrupted ground-truth graphs. This training strategy does not require predictions from upstream topology models. Experiments on OpenLane-V2 show that \ourmodel consistently improves TOP$_{ll}$ and TOP$_{lt}$ across diverse prediction sources, narrowing the gap to the detection-conditioned Upper-Bound without source-specific retraining. Further evaluations demonstrate improvements in ego-route recoverability, downstream trajectory prediction, and temporal consistency.
\end{abstract}

\section{Introduction}
\label{sec:intro}

In autonomous driving, understanding driving-scene topology is essential, as it encodes the structure of drivable regions and their interactions with traffic elements. For instance, when executing a left turn at an intersection, accurate topology understanding enables the ego vehicle to determine which lane to follow and which traffic signals govern its movement. Achieving this requires detecting both lanes and traffic elements (map element detection), inferring connectivity between lanes to form valid driving paths (lane–lane topology), and associating traffic elements with lanes to identify relevant control signals (lane–traffic topology). The resulting topology graph, in which nodes represent lanes or traffic elements and edges encode their relations, serves as the foundation for downstream tasks such as path planning~\citep{chai2019multipath} and trajectory prediction~\citep{gu2024producing}.

While accurate detection of map elements lays the foundation, the ultimate success of topology reasoning depends on making correct connectivity predictions. Even with perfectly detected map elements, erroneous connections can still lead to critical failures. For example, a planning module may choose an unsafe route if lane connections are predicted incorrectly, despite flawless detection. Beyond these safety concerns, we observe a substantial performance gap between existing methods and the theoretical upper-bound (Oracle) achievable in connectivity reasoning given their detection results. This gap suggests that existing models do not fully exploit the structural information already present in their detections, leaving significant room for improvement in topology connections. A central goal of our work is to bridge this gap by developing a topology refinement framework that fully exploits the latent potential of existing methods.

To address this gap, we propose a novel driving topology enhancement framework, \ourmodel, which formulates topology refinement as a diffusion-like denoising process to unlock the maximum topological reasoning potential inherent in existing detections. Given a set of detected map elements and topology confidence scores for each pair of detections, \ourmodel acts as an independent judge that refines connectivity predictions by leveraging structural context in the driving scene. To avoid relying on predictions from a particular upstream method for training, we develop a novel forward process (topology noise simulation) that stochastically corrupts ground-truth topology graphs to simulate realistic geometric and structural prediction errors. Complementing this, we introduce a specialized reverse process using a lightweight heterogeneous Graph Neural Network (GNN) to recover structurally consistent connectivity. This module is architected to jointly model the distinct relational dependencies of lane and traffic elements, ensuring both are optimized without one dominating the other. By exploiting this heterogeneous neighborhood information, \ourmodel can potentially learn to identify true positive connections via general traffic rules—for example, a lane associated with right-turn successor lane should have near-zero likelihood of connecting to a left-turn sign. By refining topology scores, \ourmodel leads to more reliable topology connections, bridging the gap toward the theoretical maximum performance.

As a flexible source-agnostic framework, \ourmodel enhances topology predictions from different sources without requiring retraining.  Across predictions generated from methods with diverse input modalities, \ourmodel consistently improves topology metrics (TOP$_{ll}$, TOP$_{lt}$). For instance, although TopoLogic already achieves significant improvements in TOP scores over previous state-of-the-art methods, \ourmodel further increases its TOP$_{lt}$ from 25.4 to 27.2. In addition, while SMART attains the current state-of-the-art performance by leveraging additional auxiliary data (e.g., SD maps and satellite imagery), \ourmodel still boosts its TOP$_{ll}$ from 37.0 to 40.1 without retraining. In summary, our key contributions are threefold:  

\begin{itemize}

    \item We formulate a detection-conditioned Upper-Bound evaluation reference to quantify the remaining potential for topology refinement and reveal performance gaps across existing methods.

    \item We propose \ourmodel, which learns topology refinement from corrupted ground-truth graphs and transfers across upstream methods without using their predictions for training or requiring source-specific retraining.

    \item We demonstrate consistent improvements in topology metrics across diverse prediction sources and geographically non-overlapping splits, together with improved ego-route recoverability, downstream trajectory prediction, and temporal consistency.

\end{itemize}

\section{Related Work}

\textbf{Driving Topology Reasoning.} Driving topology reasoning extracts lane–lane connectivity and lane–traffic associations to form structured road graphs, as formalized in OpenLane-V2~\citep{wang2023openlanev2}. Early DETR-style pipelines like STSU~\citep{stsu} jointly predict lanes and relations, while subsequent works enhance reasoning through diverse architectures. TopoNet~\citep{li2023toponet} utilizes GNNs for joint modeling, whereas TopoMLP~\citep{wu2023topomlp} employs a detect-then-reason pipeline with MLP heads. Further improvements integrate geometric–semantic cues (TopoLogic~\citep{fu2024topologic}), explicit point–lane interactions (TopoPoint~\citep{fu2025topopoint}), and external priors such as SD maps (SMERF~\citep{smerf}, TopoSD~\citep{yang2024toposd}), satellite imagery (SMART~\citep{smart}), or 2D lane priors (Topo2D~\citep{topo2d}). Across these approaches, existing methods typically formulate topology as pairwise confidence scores. However, a significant performance gap remains relative to the theoretical upper-bound achievable given the underlying detections. To bridge this gap, we introduce \ourmodel, a dedicated enhancement framework that refines topology scores using structural context to improve the reliability of topology connections.

\textbf{Graph Neural Networks.} Graph Neural Networks (GNNs) excel at modeling structured data by learning relational dependencies through message-passing~\citep{gnn-survey}. In driving scene topology, TopoNet~\citep{li2023toponet} pioneered using GNNs to capture map element interactions, followed by SMART~\citep{smart} and TopoLogic~\citep{fu2024topologic}, which integrated GNN-based heads for relational reasoning. More recently, TopoPoint~\citep{fu2025topopoint} introduced a Point–Lane GNN for fine-grained feature aggregation. In this work, we propose a heterogeneous GNN-based refinement model featuring a transformer-inspired architecture. Compared with the homogeneous designs that treat all different relationships uniformly, our heterogeneous modeling employs distinct message-passing kernels for lane–lane and lane–traffic relations, acknowledging their unique semantic roles~\citep{h-gnn}. By employing Graph Attention (GAT)~\citep{gat} layers, our model adaptively weights neighbors to selectively emphasize structurally consistent connections while suppressing noise. This heterogeneous attention mechanism allows \ourmodel to implicitly learn traffic-rule-aware patterns, ensuring logically consistent topology enhancement.

\section{\ourmodel}
\label{sec:method}

In this paper, our high-level objective is to bridge the performance gap in driving topology reasoning by refining connectivity predictions while keeping the underlying detections fixed.
To this end, we propose a topology enhancement framework designed to unlock the latent potential of existing perception backbones. By training on simulated topologies containing diverse predictive errors, our framework learns to recover structural consistency, thereby enabling effective generalization to the inaccuracies found in real-world topology predictions.
We frame this as a diffusion-like denoising problem on topology graphs:
the forward process (Section~\ref{tns}) stochastically corrupts a
ground-truth topology $\mathcal{G}^*$ into a noisy graph
$\mathcal{G}_{\mathrm{noise}}$ by injecting geometric and structural prediction
errors; the reverse process (Section~\ref{sec:enhancement_model}) learns
to recover reliable topology from $\mathcal{G}_{\mathrm{noise}}$ through a
lightweight heterogeneous GNN-based refinement model.
A high-level overview of \ourmodel{} is illustrated in \autoref{fig:model}.

\begin{figure*}[t]
    \centering
    \includegraphics[width=\linewidth]{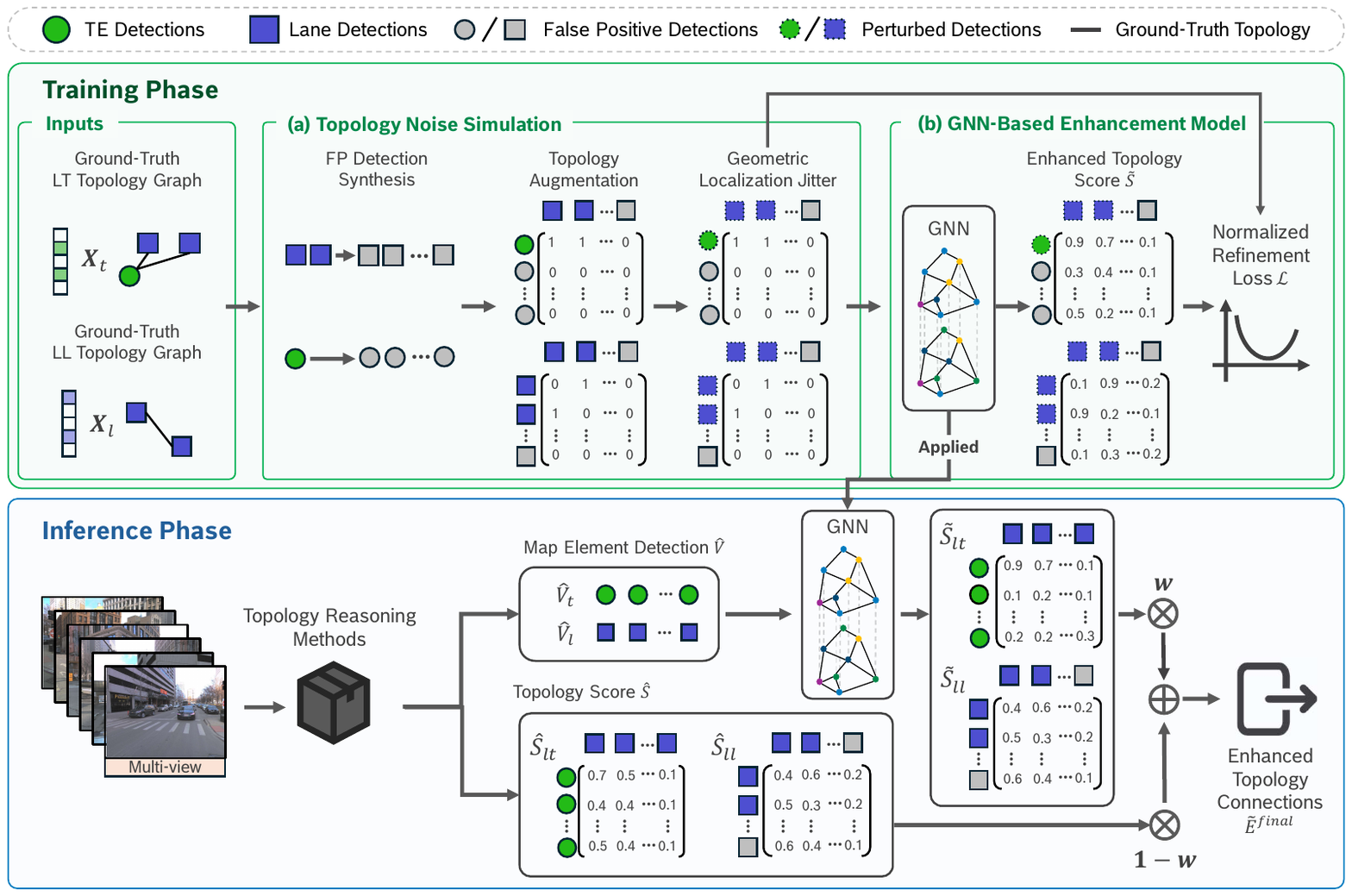}

\caption{
Overall framework of \ourmodel.
Training phase:
(a) Forward process (topology noise simulation): a novel noise simulation
strategy that stochastically corrupts ground-truth topology graphs into noisy
graphs $\mathcal{G}_{\mathrm{aug}}$ by injecting realistic geometric and
structural prediction errors.
(b) Reverse process (topology denoising): a lightweight
heterogeneous GNN learns to recover reliable topology connections from the
simulated noisy graphs.
Inference phase: once trained, the enhancement model is directly
applied to refine topology scores from any topology reasoning method,
without retraining.
}

\label{fig:model}
\end{figure*}

\subsection{Problem Formulation}
\label{sec:problem_formulation}

\textbf{Representation of Driving Scene Topology.} We represent the driving
scene topology as a heterogeneous graph $\mathcal{G} = (\mathcal{V},
\mathcal{E})$. The node set $\mathcal{V} = \mathcal{V}_l \cup \mathcal{V}_t$
consists of two distinct types: lane instances $\mathcal{V}_l =
\{l_i\}_{i=1}^{n_l}$, encoded as 3D polylines, and traffic elements
$\mathcal{V}_t = \{t_j\}_{j=1}^{n_t}$, represented by 2D bounding boxes. The
connectivity $\mathcal{E}$ is categorized into two subsets: (1) lane--lane
relations $\mathcal{E}_{ll} \subseteq \mathcal{V}_l \times \mathcal{V}_l$,
capturing structures such as successors, merges, and splits; and (2)
lane--traffic relations $\mathcal{E}_{lt} \subseteq \mathcal{V}_l \times
\mathcal{V}_t$, capturing lane-to-sign assignments.

\textbf{Enhancing Topology Predictions.} In existing pipelines,
the predicted topology $\hat{\mathcal{G}}$ is represented by detected map elements
$\hat{\mathcal{V}}$ and topology scores
$\hat{\mathcal{S}} = \{s_{uv} \in [0,1] \mid u,v \in \hat{\mathcal{V}}\}$.
The predicted topology $\hat{\mathcal{G}}$ can be viewed as a corrupted
observation of the ground-truth $\mathcal{G}^*$, suffering from false-negative
or false-positive connections. We therefore formulate topology enhancement
as a reverse process: learning a mapping $f_\theta$ that recovers refined
topology scores $\tilde{\mathcal{S}} = f_\theta(\hat{\mathcal{V}}, \hat{\mathcal{S}})$
from the corrupted $\hat{\mathcal{G}}$, such that the refined topology
predictions achieve higher structural fidelity to $\mathcal{G}^*$
by resolving logical inconsistencies through heterogeneous neighborhood
context.

\subsection{Forward Process: Topology Noise Simulation}
\label{tns}

The forward process defines the corruption distribution
$q(\mathcal{G}_{\mathrm{noise}}\mid\mathcal{G}^*)$, which stochastically
transforms a ground-truth topology $\mathcal{G}^*$ into a noisy topology
$\mathcal{G}_{\mathrm{noise}}$.
Unlike image-based diffusion models where Gaussian noise can be added directly
to continuous pixel values, designing a forward process for driving topology
graphs is non-trivial.
A driving topology graph is a heterogeneous graph structure consisting of
3D polylines, 2D bounding boxes, and categorical connectivity relations.
Corruption must be applied separately and meaningfully to each component to simulate errors encountered in real-world topology predictions.
These corruptions aim to help the reverse process
generalize on real-world topology predictions.
To address this, we design the forward process through three components that
jointly simulate errors in real-world topology
predictions.

\textbf{False-Positive Detection Synthesis.} A primary failure in driving
topology reasoning is the detection of false positives, i.e., lane instances
or traffic elements that are either geometrically disparate from the
ground-truth instances or entirely non-existent in the physical scene. We
simulate these by expanding the detection set $\mathcal{V}_{\mathrm{noise}} =
\mathcal{V}^* \cup \mathcal{V}_{fp}$. To generate the false-positive detection
set $\mathcal{V}_{fp}$, we sample from a high-variance distribution, accepting
a candidate detection $v_{fp}$ only if its geometric distance from the source
ground-truth instance $v$ exceeds a category-specific distance threshold
$\delta_{type}$:

\begin{equation}
v_{fp} = v + \epsilon, \quad \text{s.t.} \quad \text{dist}(v, v_{fp}) >
\delta_{type}, \quad \epsilon \sim \mathcal{N}(0, \sigma_{fp}^2\mathbf{I}),
\end{equation}
where $v \in \mathcal{V}^*$ is a ground-truth instance and $type \in \{l, t\}$
refers to the lane or traffic element category.

\textbf{Topology Augmentation.} These false-positive detections are then used
to synthesize unreliable connectivities that $f_\theta$ must learn to prune.
We define the candidate edge set $\mathcal{E}_{\mathrm{noise}}$ by generating
all potential pairs $\{(u,v) \mid u,v \in \mathcal{V}_{\mathrm{noise}}\}$. To
supervise the reverse process, we assign ground-truth edge labels $y_{uv}$ as:

\begin{equation}\label{eq:gt_label}
y_{uv} = \mathbb{I}[(u,v) \in \mathcal{E}^*].
\end{equation}
This labeling scheme naturally ensures that any connection involving at least
one false-positive detection from $\mathcal{V}_{fp}$ is assigned a negative
label ($y_{uv} = 0$), as these nodes by definition do not exist in
$\mathcal{V}^*$. By populating $\mathcal{V}_{fp}$ such that the augmented node
set matches the expected cardinality of detections during inference
($|\mathcal{V}_{\mathrm{noise}}| = \mathbb{E}[|\hat{\mathcal{V}}|]$), the
candidate edge space $\mathcal{E}_{\mathrm{noise}}$ matches the size of the
topology scores $\hat{\mathcal{S}}$ to be refined.

\textbf{Geometric Localization Jitter.} In practice, the map element
predictions $\hat{\mathcal{V}}$ can never perfectly reconstruct ground-truth
geometries; even true-positive detections contain localization jitter in their
3D polylines or 2D bounding boxes. If $f_\theta$ were trained only on perfect
true-positive detections, it generalize less effectively to the noisy coordinates
encountered during inference. To bridge this gap, we apply a constrained
perturbation to all ground-truth detections $v \in \mathcal{V}^*$,
transforming the raw geometry into a jittered representation $v'$:

\begin{equation}
v' = v + \epsilon, \quad \text{s.t.} \quad \text{dist}(v, v') \leq
\delta_{type}, \quad \epsilon \sim \mathcal{N}(0, \sigma_{tp}^2\mathbf{I})
\end{equation}
By suppressing the perturbation within $\delta_{type}$, these detections
remain semantically valid yet distinguishable from the false positives in
$\mathcal{V}_{fp}$, forcing $f_\theta$ to resolve connectivity from
topological context rather than exact coordinate matching.

\textbf{Adaptive Noise Curriculum.} To expose the forward process
$q(\mathcal{G}_{\mathrm{noise}}\mid\mathcal{G}^*)$ to a broad spectrum of
corruption severities, we design the noise variances $\sigma_{tp}^2$ and $\sigma_{fp}^2$ as dynamic functions of training progress. Inspired by curriculum learning~\citep{curriculum}, we design a curriculum-like schedule which enables a gradual transition from coarse error correction to fine-grained discrimination. Specifically, $\sigma_{tp}$ is initialized near zero and progressively increased toward the distance threshold $\delta_{type}$, while $\sigma_{fp}$ starts at a relatively large value and is gradually reduced toward $\delta_{type}$. As both noise profiles converge to $\delta_{type}$, the geometric distinction between true and false detections becomes increasingly subtle. This design can theoretically improve training efficiency and generalization of our framework~\citep{curriculum-benefit}.

The complete forward process is summarized in Algorithm~\ref{alg:tns}.
By design,
the simulated corruption $q(\mathcal{G}_{\mathrm{noise}}\mid\mathcal{G}^*)$
models selected errors in real-world predictions
$\hat{\mathcal{G}}$. This design aims to help error-correction patterns learned on $\mathcal{G}_{\mathrm{noise}}$
naturally transfer to $\hat{\mathcal{G}}$ during inference, without requiring
retraining on any specific prediction source.

\subsection{Reverse Process: Topology Enhancement Model}
\label{sec:enhancement_model}

The reverse process learns the conditional distribution
$p_\theta(\mathcal{Y}\mid\mathcal{G}_{\mathrm{noise}})$ of ground-truth edge labels $\mathcal{Y}$, which recovers
reliable topology from the corrupted topology $\mathcal{G}_{\mathrm{noise}}$.
It is achieved with a heterogeneous GNN-based enhancement model $g_\theta$
that performs type-aware message passing over lane and traffic-element nodes,
enabling distinct processing of lane--lane and lane--traffic relations.
Through attention-based aggregation, $g_\theta$ leverages structural context
to suppress inconsistent edges while reinforcing connectivity patterns aligned
with lane flow and traffic regulations.

\textbf{Heterogeneous GNN Architecture.}
Given the noisy graph $\mathcal{G}_{\mathrm{noise}} = (\mathcal{V}_{\mathrm{noise}},
\mathcal{E}_{\mathrm{noise}})$, we employ a heterogeneous GNN to learn
relational representations.
Node features $(X_l, X_t)$ are initialized by encoding lane instances using
their 3D polyline coordinates, while traffic elements are represented by
visual features extracted from their 2D bounding boxes.
The node embeddings are iteratively updated through heterogeneous message
passing.
Specifically, for each node $v \in \mathcal{V}_{\mathrm{noise}}$, its embedding
$\mathbf{h}_v^{(k)}$ at layer $k$ is updated by aggregating information from
its multi-type neighborhood:

\begin{equation}
\mathbf{h}_v^{(k)} = \sigma \left( \sum_{r \in \mathcal{R}}
\text{AGG}_{(u,v) \in \mathcal{E}_{\mathrm{noise}}} \left( \psi_r \left(
\mathbf{h}_v^{(k-1)}, \mathbf{h}_u^{(k-1)} \right) \right) \right),
\end{equation}
where $\mathcal{R} = \{ll, lt\}$ denotes lane--lane and lane--traffic
relations, and $\psi_r$ is a relation-specific transformation function.
We utilize a lightweight two-layer Heterogeneous Graph Attention Network
(GAT) as aggregation function, whose attention mechanism learns relation-aware coefficients to
emphasize structurally relevant neighbors while suppressing noisy or
false-positive detections.
After obtaining final node embeddings, we compute the refined score
$\tilde{s}_{uv} \in \tilde{\mathcal{S}}$ via a sigmoid-activated bilinear decoder:

\begin{equation}\label{eq:cos_sim}
\begin{aligned}
\tilde{s}_{uv} &= \sigma\!\left(
a_r(\mathbf{P}_r\mathbf{h}_u)^\top
(\mathbf{Q}_r\mathbf{h}_v)+b_r
\right),\\
&\qquad \forall (u,v)\in\mathcal{E}_{\mathrm{noise}}.
\end{aligned}
\end{equation}
where $r\in\{ll,lt\}$ denotes the relation type,
$\mathbf{P}_r$ and $\mathbf{Q}_r$ are learned source and destination
projection matrices, and $a_r$ and $b_r$ are the learned scale and bias.
The separate projections allow the lane--lane scores
$\tilde{s}_{uv}$ and $\tilde{s}_{vu}$ to differ.
$\sigma(\cdot)$ denotes the sigmoid function, ensuring
$\tilde{s}_{uv} \in (0,1)$.

\textbf{Normalized Refinement Loss.}
We derive the training objective directly from the reverse process
$p_\theta(\mathcal{Y}\mid\mathcal{G}_{\mathrm{noise}})$.
Assuming conditional independence of candidate edges given
$\mathcal{G}_{\mathrm{noise}}$, the reverse distribution factorizes as:

\begin{equation}\label{eq:bernoulli}
\begin{aligned}
p_\theta\!\left(\mathcal{Y}\mid\mathcal{G}_{\mathrm{noise}}\right)
&=\prod_{(u,v)\in\mathcal{E}_{\mathrm{noise}}}
\mathrm{Bernoulli}(y_{uv};\,\tilde{s}_{uv})\\
&=\prod_{(u,v)\in\mathcal{E}_{\mathrm{noise}}}
\tilde{s}_{uv}^{y_{uv}}(1-\tilde{s}_{uv})^{1-y_{uv}},
\end{aligned}
\end{equation}
where $y_{uv}$ is the ground-truth edge label from \autoref{eq:gt_label} and
$\tilde{s}_{uv}\in(0,1)$ is the refined topology score from
\autoref{eq:cos_sim}.
Taking the negative log-likelihood of \autoref{eq:bernoulli} yields the
standard binary cross-entropy (BCE) over all candidate edges, as shown in Lemma~\ref{lemma:bce}.
However, applying BCE jointly over all edges is problematic: the much larger
lane--lane candidate space dominates the gradient, starving the lane--traffic
relation of useful learning signal.
We therefore adopt a relation-aware, volume-normalized BCE.
For each relation $r\in\{ll, lt\}$, the per-relation loss is:

\begin{equation}\label{eq:loss_definition}
\mathcal{L}^{(r)}
= - \frac{1}{|\mathcal{E}_r|}
\left(
\sum_{(u,v)\in \mathcal{E}_r^{\mathrm{pos}}} \log \tilde{s}_{uv}
+
\sum_{(u,v)\in \mathcal{E}_r^{\mathrm{neg}}} \log (1-\tilde{s}_{uv})
\right),
\end{equation}
where $|\mathcal{E}_r|$ is the total number of candidate edges for relation
$r$, $\mathcal{E}_r^{\mathrm{pos}} = \{(u,v) \in \mathcal{E}_r \mid
y_{uv}=1\}$, and $\mathcal{E}_r^{\mathrm{neg}} = \{(u,v) \in \mathcal{E}_r
\mid y_{uv}=0\}$.
The overall objective balances the two relation types as:

\begin{equation}\label{eq:loss_all}
\mathcal{L}
= \mathcal{L}^{(ll)}
+ \mathcal{L}^{(lt)}.
\end{equation}
As formalized in Lemma~\ref{lemma:bce} and Proposition~\ref{prop:diffusion} (\autoref{app:diffusion}), $\mathcal{L}$ can be interpreted as a relation-balanced analogue of the reconstruction term in the negative of a single-step $T{=}1$ diffusion evidence lower bound (ELBO)~\citep{sohl2015deep, ho2020denoising}.

\textbf{Training and Inference.}
During training, $g_\theta$ is optimized on simulated noisy graphs
$\mathcal{G}_{\mathrm{noise}}$ from the forward process.
The parameters $\theta$ are optimized against ground-truth labels $\mathcal{Y}$:

\begin{equation}
\theta^* = \arg\min_\theta \mathcal{L}(\tilde{\mathcal{S}}, \mathcal{Y}).
\end{equation}

During inference, given detection predictions $\hat{\mathcal{V}}$ and topology scores $\hat{\mathcal{S}} = \{ s_{uv} \mid u,v \in \hat{\mathcal{V}}
\}$, we construct the inference graph $\hat{\mathcal{G}} = (\hat{\mathcal{V}},
\hat{\mathcal{E}}^{\mathrm{cand}})$, where $\hat{\mathcal{E}}^{\mathrm{cand}}
= \hat{\mathcal{V}} \times \hat{\mathcal{V}}$.
Passing $\hat{\mathcal{G}}$ through $g_\theta$ yields refined scores
$\tilde{\mathcal{S}} = g_\theta(\hat{\mathcal{G}})$, which $f_\theta$ combines with
the original topology scores via a relation-specific fusion weight $w_r$:

\begin{equation}
s_{uv}^{\mathrm{final}} = w_r\,\tilde{s}_{uv} + (1-w_r)\,\hat{s}_{uv},
\quad \forall (u,v) \in \hat{\mathcal{E}}^{\mathrm{cand}}.
\end{equation}

\section{Experiments}
\label{sec:results}

\begin{figure*}[t]
    \centering
    \begin{subfigure}[b]{0.19\textwidth}
        \centering
        \includegraphics[width=1\linewidth]{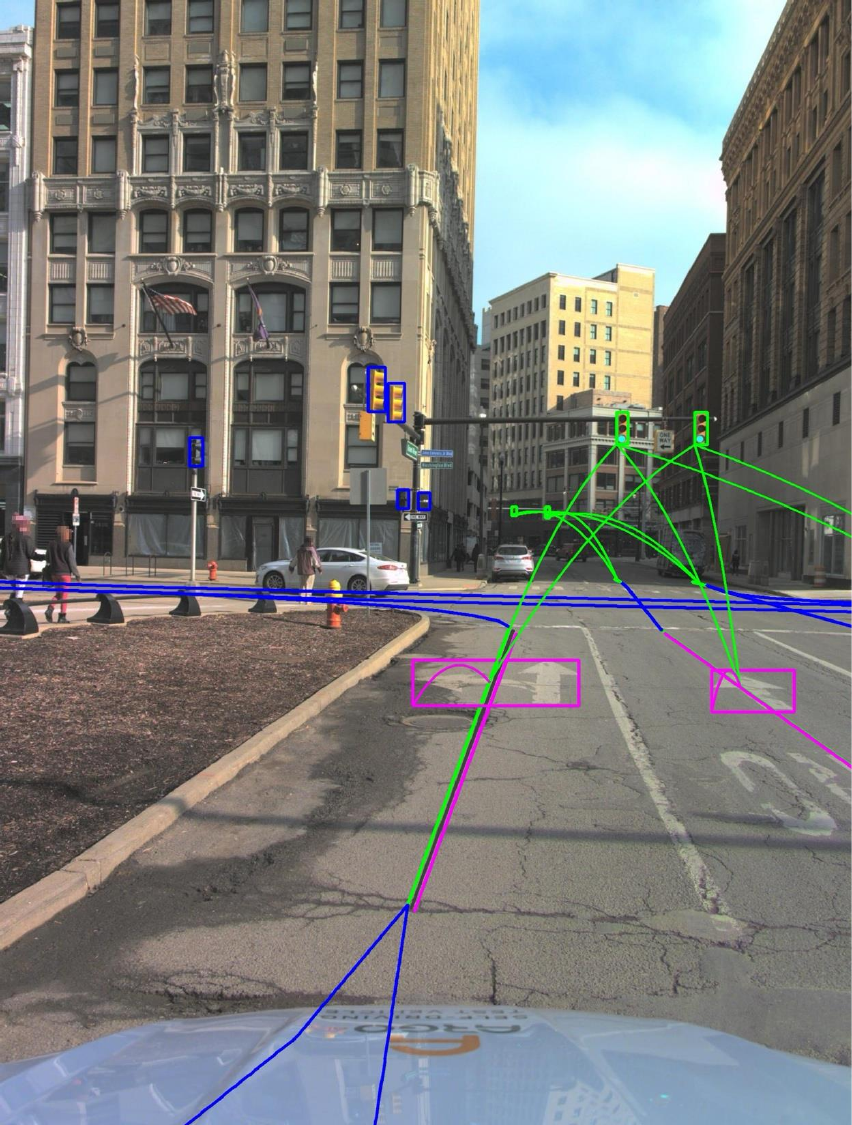}
        \caption{\centering Ground-truth \\ Topology}
        \label{fig:qualitative_gt}
    \end{subfigure}
    \hfill
    \begin{subfigure}[b]{0.19\textwidth}
        \centering
        \includegraphics[width=1\linewidth]{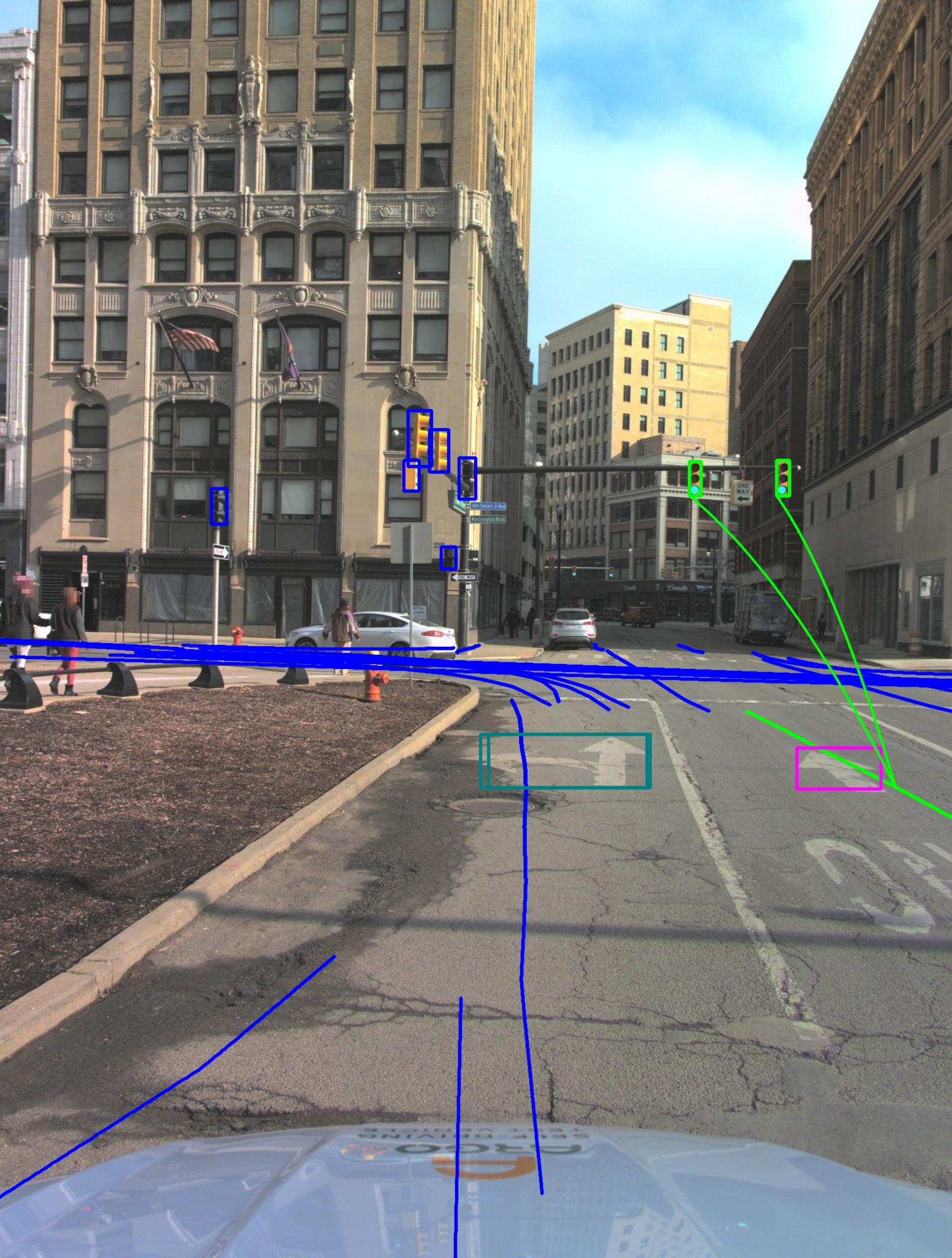}
        \caption{\centering TopoNet \\ Prediction}
        \label{fig:toponet_pred}
    \end{subfigure}
    \hfill
    \begin{subfigure}[b]{0.19\textwidth}
        \centering
        \includegraphics[width=1\linewidth]{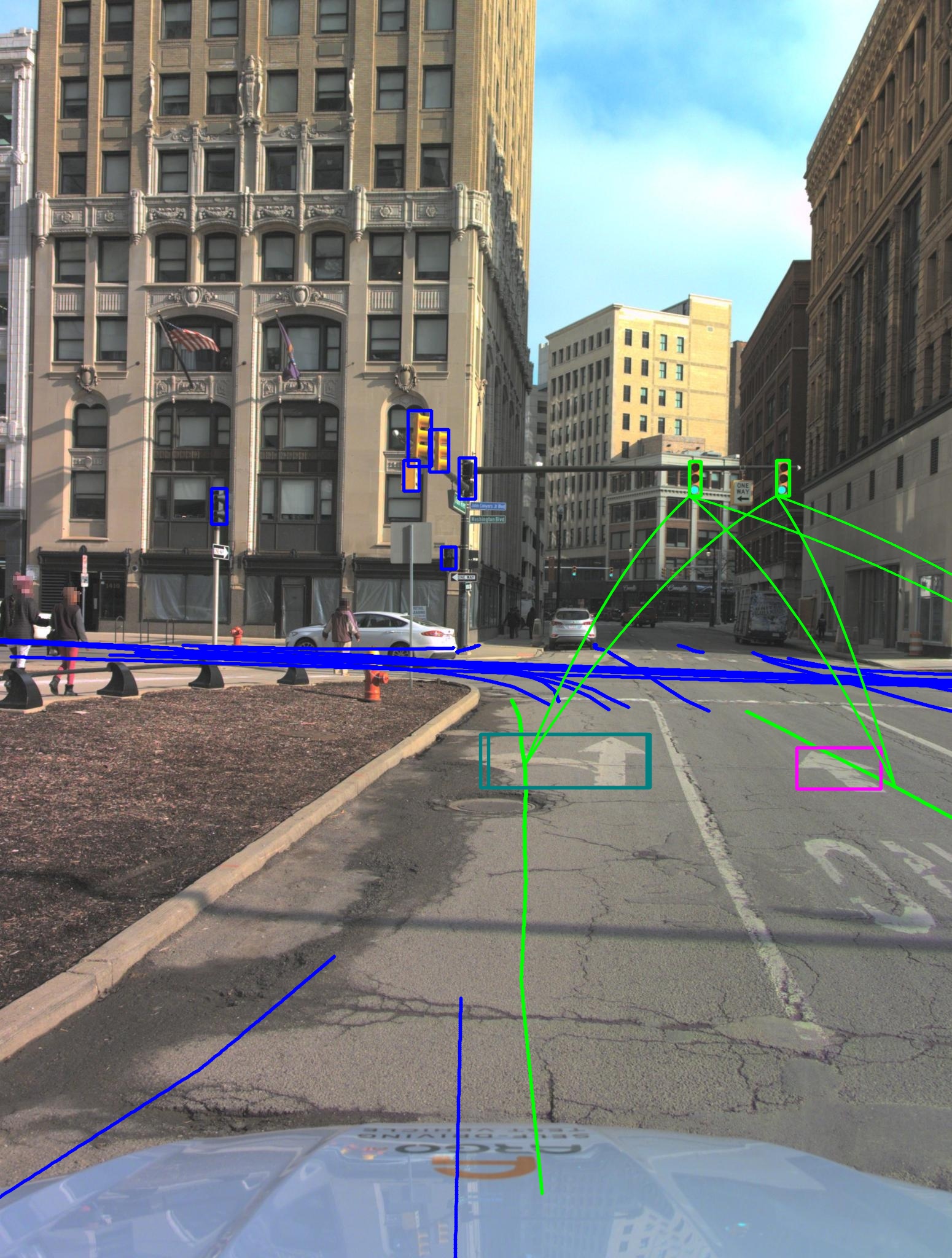}
        \caption{\centering TopoNet  \\+ \ourmodel}
        \label{fig:toponet_ssl}
    \end{subfigure}
    \hfill
    \begin{subfigure}[b]{0.19\textwidth}
        \centering
        \includegraphics[width=1\linewidth]{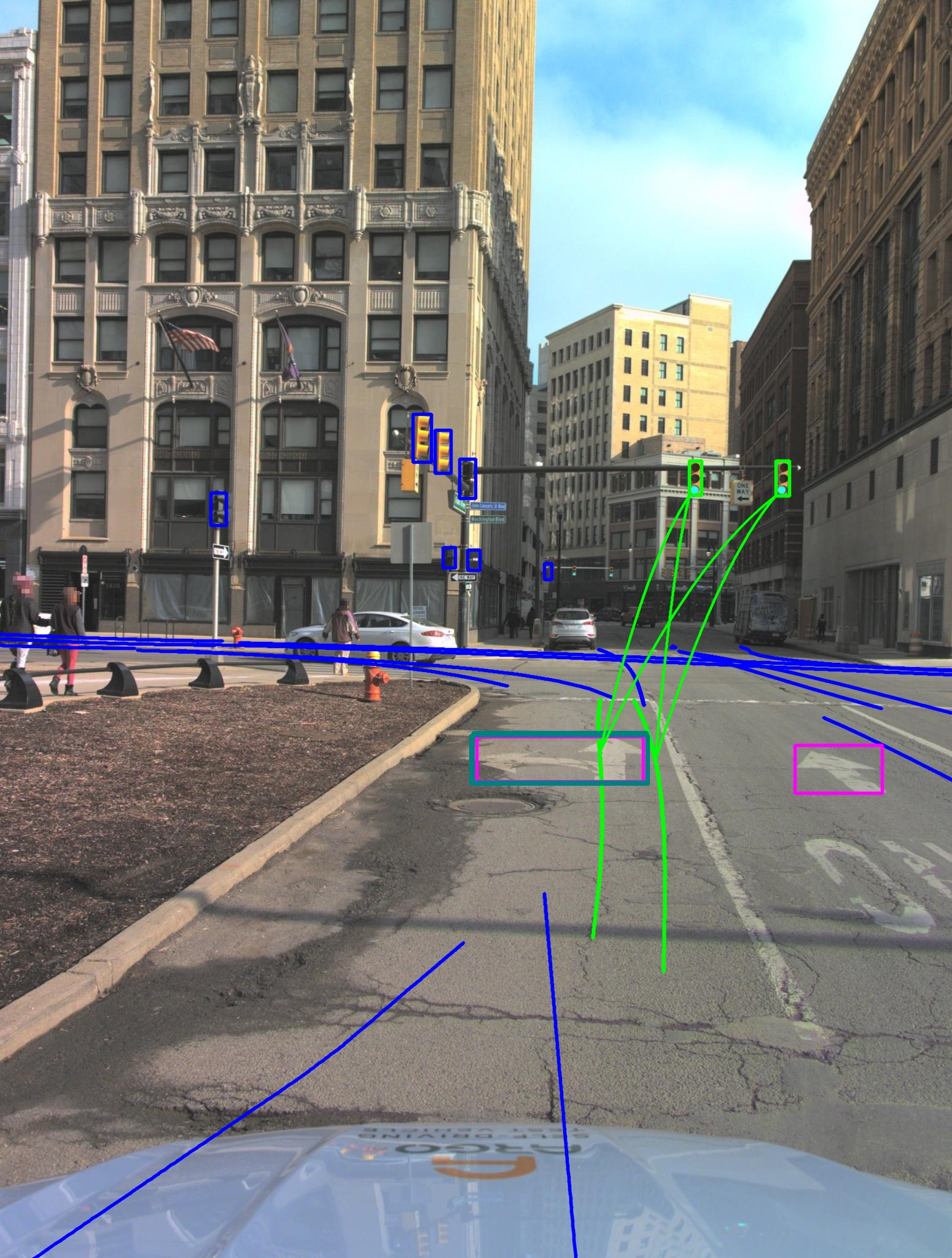}
        \caption{\centering TopoMLP \\Prediction}
        \label{fig:topomlp_pred}
    \end{subfigure}
    \hfill
    \begin{subfigure}[b]{0.19\textwidth}
        \centering
        \includegraphics[width=1\linewidth]{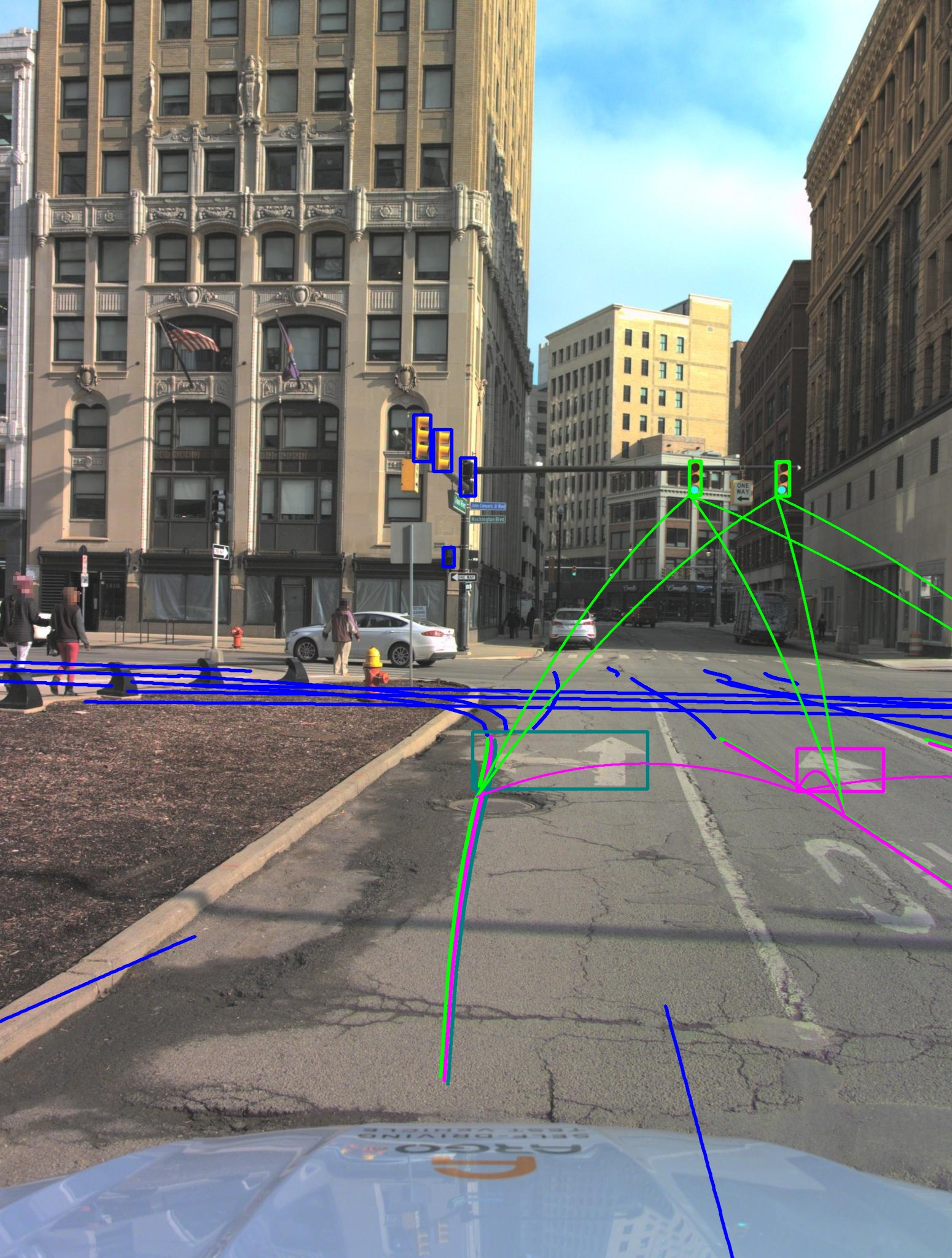}
        \caption{\centering TopoMLP \\+ \ourmodel}
        \label{fig:topomlp_ssl}
    \end{subfigure}
    \caption{Qualitative comparison of topology connections before and after enhancement.  
    }
    \label{fig:qualitative_comparison}
\end{figure*}

\textbf{Dataset.} Our experimental studies are performed on OpenLane-V2~\citep{wang2023openlanev2}, a comprehensive benchmark tailored for driving scene topology understanding and developed on top of the Argoverse2~\citep{wilson2023argoverse} and nuScenes~\citep{nuscenes} datasets. The benchmark supplies rich annotations spanning lane centerline reconstruction, traffic element recognition, and topological relationship modeling. It is organized into two subsets, Argoverse2 (subset\_A) and nuScenes (subset\_B), each containing 1,000 annotated driving sequences captured at 2 Hz with synchronized multi-view camera inputs. Both splits provide supervision for lane geometry, traffic elements, and the connectivity relations between lanes as well as between lanes and traffic elements. \textit{For the main topology comparison, we train one model per dataset and apply it to enhance predictions from different methods at inference time, without any method-specific retraining.} For more implementation details, please refer to \autoref{sec:implement-details}.

\subsection{Evaluation Metrics}

For all topology metrics, we first establish a correspondence between predicted detections $\hat{V}$ and ground-truth detections $V^*$. Following OpenLane-V2~\citep{wang2023openlanev2}, detections are matched using the discrete Fréchet distance for lanes and IoU for traffic elements. We define the set of true-positive detections, $\hat{V}_{\text{TP}} \subseteq \hat{V}$, as those predictions that successfully match a ground-truth counterpart.

\textbf{Topology Metrics (OpenLane-V2 TOP).} To evaluate the impact of our enhancement on the topology scores $\mathcal{S}$, we adopt the TOP metric from the OpenLane-V2 benchmark~\citep{wang2023openlanev2} (version 2.1.0). TOP is an mAP-based metric evaluated using the official OpenLane-V2 protocol. TOP score computes average mAP between $(V^*,E^*)$ and $(\hat{V},\hat{E})$ over all vertices:

\begin{equation}
    \text{TOP} = \frac{1}{|V^*|} \sum_{v \in V^*} \frac{\sum_{\hat{n} \in \hat{N}(v)} P(\hat{n}) \mathbf{1}(\hat{n} \in N(v))}{|N(v)|},
\end{equation}
where $P(\hat{n})$ is the precision of neighbor $\hat{n}$ in the ordered list ranked by $s_{uv}$. We report TOP both before and after enhancement to quantify changes in benchmark performance. Also, $\text{TOP}_{ll}$ and $\text{TOP}_{lt}$ are computed for lane--lane and lane--traffic relations, respectively.

\textbf{Theoretical Upper-Bound (Upper-Bound).} To isolate the performance of topology reasoning from the quality of the upstream detector, we quantify a theoretical upper bound for \text{TOP} through this metric. This bound represents the maximum achievable performance given the existing set of detection predictions $\hat{\mathcal{V}}$. Specifically, we assume an oracle that assigns the final confidence scores $s_{uv}^*$,

\begin{equation}
s_{uv}^* =
\begin{cases}
1.0 & \text{if } (u, v) \in \mathcal{E^*} \text{ and } u, v \in \hat{\mathcal{V}}_{\text{TP}} \\
0.0 & \text{otherwise}
\end{cases}.
\label{eq:upper_bound}
\end{equation}
By evaluating the metrics using $s_{uv}^*$, we establish a performance ceiling where all possible true connectivities are perfectly recovered and ranked, while errors stemming from missing detections or false-positive detections remain. Comparing our results against this bound quantifies the remaining "headroom" for topological refinement and distinguishes between errors caused by connectivity reasoning versus those inherent to the detection backbone.

\subsection{Enhanced Results}

In this section, we present enhanced results obtained by applying \ourmodel on existing topology reasoning methods with publicly available checkpoints to refine their topology connections. To the best of our knowledge, the methods with officially released checkpoints include TopoNet~\citep{li2023toponet}, TopoMLP~\citep{wu2023topomlp}, Topo2D~\citep{topo2d}, TopoLogic~\citep{fu2024topologic}, and SMART~\citep{smart}. Among these methods, SMART achieves state-of-the-art performance in topology reasoning by leveraging additional data modalities, such as SD maps and satellite imagery. In contrast, the other methods rely solely on perspective camera images as input. Beyond methods with publicly released checkpoints, we also attempted to train several additional approaches from scratch. However, we were unable to reliably reproduce their reported results, likely due to differences in training environments or hardware configurations. To ensure fairness and reliability in our comparisons, we therefore restrict our evaluation to methods with official checkpoints. Additionally, please refer to Appendix~\ref{sec:additional_evaluation} for additional evaluation results.

\textbf{Quantitative Comparison on Topology Scores.} Table~\ref{tab:top_results} summarizes the performance enhancement measured by OpenLane-V2 TOP across multiple existing topology reasoning methods. Our enhancer demonstrates consistent and significant improvements in the TOP metric for all evaluated baselines, regardless of their underlying architecture or input modalities. Notably, our model nearly triples the performance of TopoNet ($6.7 \rightarrow 19.5$ $\text{TOP}_{ll}$) on nuScenes and provides a significant boost to TopoMLP ($21.6 \rightarrow 24.4$ $\text{TOP}_{ll}$) on Argoverse2.  Even for the state-of-the-art methods SMART, which utilizes auxiliary data such as satellite imagery and SD maps, our model can still further boosts performance ($37.0 \rightarrow 40.1$ $\text{TOP}_{ll}$) and ($33.0 \rightarrow 35.4$ $\text{TOP}_{lt}$) during inference, \textit{without any retraining and adaptation to the additional input modalities}. This indicates that modeling the structural dependencies between predicted entities improves topology prediction. A key observation from Table~\ref{tab:top_results} is that our refined topology scores closely approach the Theoretical Upper-Bound. Given that the Upper-Bound represents the maximum achievable performance given the fixed detection predictions $\hat{\mathcal{V}}$, our enhancer has substantially narrows the gap within the limits of the provided detections. Specifically, on Argoverse2, our enhancer reduces the margin to the maximum reachable performance in $\text{TOP}_{ll}$ to just 0.6 for TopoMLP and 0.7 for Topo2D.  For the SMART baseline, the enhanced $\text{TOP}_{ll}$ of 40.1 sits within a 0.8 margin of the theoretical ceiling. These results validate that our approach effectively serves as a universal enhancement layer, capable of leveraging heterogeneous neighborhood context to rectify local connectivity errors inherent in existing prediction heads.

\begin{table*}[t]
\centering
\scriptsize
\caption{Enhanced results on OpenLane-V2 measured by the TOP metric (version 2.1.0). Bolded blue
values indicate improved scores over the corresponding baselines before applying \ourmodel.}
\label{tab:top_results}
\resizebox{0.91\textwidth}{!}{
\begin{tabular}{l@{\hskip 0.1in}lcc@{\hskip 0.2in} c c@{\hskip 0.2in} c c}
\toprule
\multirow{2}{*}{\text{Dataset}}  & \multirow{2}{*}{\text{Method}} & \multicolumn{2}{c}{\textbf{Before \ourmodel}} & \multicolumn{2}{c}{\textbf{After \ourmodel}} & \multicolumn{2}{c}{\textbf{Upper-Bound}}  \\
\cmidrule(r){3-4}\cmidrule(r){5-6}\cmidrule(r){7-8}
 &  & TOP$_{ll} \uparrow$ & TOP$_{lt} \uparrow$ & TOP$_{ll} \uparrow$ & TOP$_{lt} \uparrow$ & TOP$_{ll} \uparrow$ & TOP$_{lt} \uparrow$ \\

\midrule
\multirow{5}{*}{{\makecell{Argoverse2\\(Subset\_A)}}}  & TopoNet~\citep{li2023toponet} & 10.9 & 23.8 & \best{21.8} & \best{25.8} & 24.4 & 28.8\\
&  TopoMLP~\citep{wu2023topomlp}   & 21.6 & 26.9 & \best{24.4} & \best{28.7} & 25.0 & 31.4\\

&   Topo2D~\citep{topo2d}   & 22.3 & 26.2 &  \best{24.3}  & \best{27.3} & 25.0 & 29.8\\
 & TopoLogic~\citep{fu2024topologic}  & 23.9 & 25.4 & \best{24.5}  & \best{27.2} & 26.0 & 30.6 \\
 & SMART~\citep{smart} & 37.0 & 33.0 & \best{40.1} & \best{35.4} & 40.9 & 38.9 \\

\specialrule{.1em}{.1em}{.1em}

\multirow{2}{*}{\makecell{nuScenes\\(Subset\_B)}}  & TopoNet~\citep{li2023toponet} & 6.7 & 16.7 & \best{19.5} & \best{17.8} & 21.2 & 20.3\\

 & TopoMLP~\citep{wu2023topomlp}   & 21.0 & 19.8 & \best{22.8} & \best{21.5} & 23.5 & 24.1\\
\bottomrule
\end{tabular}}

\end{table*}

\label{sec:geosplits_main}
\textbf{Geographically Non-Overlapping Splits.} To evaluate generalization without geographic overlap between training and validation, we use the Argoverse2 GeoSplits introduced by~\citet{geosplit}. Since the evaluated baselines do not provide checkpoints for these splits, we train both TopoNet and \ourmodel from scratch. Table~\ref{tab:geosplits} shows improvements in TOP$_{ll}$ from 3.34 to 10.23 and TOP$_{lt}$ from 14.40 to 16.37, approaching the Upper-Bounds of 11.47 and 18.42. These results show that the gains persist on geographically separate scenes.

\begin{table}[t]
\centering
\small
\caption{Evaluation on geographically non-overlapping Argoverse2 GeoSplits. Both TopoNet and \ourmodel are trained on these splits.}
\label{tab:geosplits}
\begin{tabular}{lcc}
\toprule
Method & TOP$_{ll}\uparrow$ & TOP$_{lt}\uparrow$ \\
\midrule
TopoNet & 3.34 & 14.40 \\
+ \ourmodel & \best{10.23} & \best{16.37} \\
Upper-Bound & 11.47 & 18.42 \\
\bottomrule
\end{tabular}

\end{table}

\textbf{Qualitative Comparison.} \autoref{fig:qualitative_comparison} shows that \ourmodel effectively corrects missing and fragmented connections in TopoNet and TopoMLP, such as unlinked traffic lights in TopoNet or inconsistent lane merges in TopoMLP. More qualitative results are provided in \autoref{app:more_qualitative_results}.



\subsection{Downstream Practicability}
\label{sec:downstream_practicability}

We evaluate whether topology refinement improves the practical usability of predicted lane graphs through scene-level ego-route recoverability, downstream trajectory prediction, and temporal consistency. Appendix~\ref{sec:practical_evaluation} provides detailed protocols and complete results, together with additional analyses of corruption robustness, driving-relevant connection errors, and efficiency.

\textbf{Scene-Level Ego-Route Recoverability.}
Downstream navigation and planning require a sequence of valid lane transitions along a driving route. We evaluate whether topology predictions before and after enhancement preserve complete ego routes on Argoverse2 (Subset\_A) by aggregating frame-level predictions into a scene-level lane graph using persistent lane IDs. To isolate connectivity from upstream detection errors, we include only routes whose required lane detections are successfully detected and matched. A route is recoverable only if every transition along the ego vehicle's complete driven trajectory, typically spanning 70--110\,m, is preserved in the predicted topology graphs. Table~\ref{tab:route_recoverability} shows improvements across all four baselines, including 39.19\% to 68.92\% for TopoNet and 94.81\% to 100.00\% for TopoMLP. These results show that the topology graph enhanced by \ourmodel increases the fraction of scene-level driving paths recoverable from the predicted topology graphs, preserving the connectivity needed for route planning. Appendix~\ref{sec:route_recoverability} provides the detailed evaluation.

\begin{table}[t]
\centering
\small
\caption{Scene-level ego-route recoverability (\%) before and after \ourmodel on Argoverse2 (Subset\_A). Higher is better.}
\label{tab:route_recoverability}
\begin{tabular}{lcc}
\toprule
Model & Before \ourmodel & After \ourmodel \\
\midrule
TopoNet & 39.19 & \best{68.92} \\
TopoMLP & 94.81 & \best{100.00} \\
TopoLogic & 92.41 & \best{93.67} \\
SMART & 91.36 & \best{98.77} \\
\bottomrule
\end{tabular}

\end{table}

\textbf{Trajectory Prediction.}
To evaluate whether our enhanced connectivity benefits downstream trajectory prediction, we compare LaneGCN~\citep{liang2020learning} trained with TopoNet connections before and after our enhancement. Both settings use the same detected lane geometry from TopoNet, but with a separate LaneGCN trained for each. In our experiments, LaneGCN predicts six future trajectories over 3\,s from 2\,s of history. We report endpoint error (minFDE$_6$) and average displacement error (minADE$_6$) for the mode with the lowest endpoint error, and the miss rate beyond 2\,m (MR$_6$). On 19,073 validation samples, our enhancement reduces minADE$_6$ from 0.271 to 0.269\,m, minFDE$_6$ from 0.659 to 0.624\,m, and MR$_6$ from 5.09\% to 3.59\% (Table~\ref{tab:lanegcn_all}), showing improved best-of-six forecasting accuracy under this protocol. Appendix~\ref{app:trajectory_prediction} provides the detailed protocol.

\begin{table}[t]
\centering
\small
\caption{LaneGCN performance before and after \ourmodel using the same TopoNet lane detections on Argoverse2 (Subset\_A). Displacement errors are in meters. All metrics are lower-is-better.}
\label{tab:lanegcn_all}
\begin{tabular}{@{}lccc@{}}
\toprule
TopoNet & minADE$_6$ & minFDE$_6$ & MR$_6$ (\%) \\
\midrule
Before \ourmodel & 0.271 & 0.659 & 5.09 \\
After \ourmodel & \best{0.269} & \best{0.624} & \best{3.59} \\
\bottomrule
\end{tabular}
\end{table}

\textbf{Temporal Consistency.} We evaluate temporal stability on topology graphs before and after our enhancement on Argoverse2 (Subset\_A). Specifically, we use persistent lane IDs to track ground-truth lane pairs that remain co-visible for at least five consecutive frames. The edge flip rate measures the fraction of adjacent frame pairs with a change in predicted connectivity, while the temporal standard deviation measures variation in connection confidence. Table~\ref{tab:temporal_consistency} shows improvements in both metrics across TopoNet, TopoMLP, TopoLogic, and SMART. For example, refinement reduces TopoMLP's flip rate from 12.62\% to 7.77\% and SMART's from 10.52\% to 6.38\%; their temporal standard deviations decrease from 0.111 to 0.080 and from 0.104 to 0.075, respectively. These results indicate that \ourmodel can produce more stable connectivity predictions over time. Appendix~\ref{sec:temporal_consistency} provides the detailed evaluation.

\begin{table}[t]
\centering
\small
\caption{Temporal consistency before and after \ourmodel on Argoverse2 (Subset\_A). Ground-truth lane pairs are tracked using persistent lane IDs and remain co-visible for at least five consecutive frames. Lower values indicate greater stability.}
\label{tab:temporal_consistency}
\begin{tabular}{@{}lcccc@{}}
\toprule
& \multicolumn{2}{c}{Flip rate (\%) $\downarrow$} & \multicolumn{2}{c}{Temporal std. $\downarrow$} \\
\cmidrule(lr){2-3}\cmidrule(lr){4-5}
Model & Before & After & Before & After \\
\midrule
TopoNet & 24.48 & \best{24.21} & 0.166 & \best{0.118} \\
TopoMLP & 12.62 & \best{7.77} & 0.111 & \best{0.080} \\
TopoLogic & 10.68 & \best{8.27} & 0.285 & \best{0.228} \\
SMART & 10.52 & \best{6.38} & 0.104 & \best{0.075} \\
\bottomrule
\end{tabular}
\end{table}

\subsection{Ablation Studies}
\label{sec:ablation_main}

We examine the effects of topology noise simulation and traffic-element feature backbones on \ourmodel using OpenLane-V2 Subset\_A. Appendix~\ref{app:ablation} provides more comprehensive ablation studies.

\textbf{Topology Noise Simulation.} To evaluate the effectiveness of our simulated graphs as training data, we compare \ourmodel trained on these graphs with a variant trained on TopoNet's real topology predictions, keeping the GNN architecture and node features fixed. Table~\ref{tab:ablation_summary}(a) shows that training on real predictions yields slightly higher TOP$_{ll}$/TOP$_{lt}$ on TopoNet itself (22.4/26.2 versus 21.8/25.8), but this advantage does not consistently transfer to other methods. Simulated-graph training achieves higher scores on TopoMLP (24.4/28.7 versus 24.1/28.4) and SMART (40.1/35.4 versus 39.7/35.0), while the two training sources yield comparable results on TopoLogic. These results suggest that our noise simulation supports generalization across topology prediction methods without requiring real model predictions during training.

\begin{table}[tbp]
\centering
\small
\caption{Ablation studies on OpenLane-V2 Subset\_A. (a) Noise simulation versus real-prediction training. (b) Traffic-element feature backbones with TopoNet.}
\label{tab:ablation_summary}
\label{tab:training_sources}
\begin{tabular}{@{}lrrrr@{}}
\toprule
\multicolumn{5}{c}{\textbf{(a) Topology noise simulation}} \\
\midrule
& \multicolumn{2}{c}{Simulated graphs} & \multicolumn{2}{c}{Real predictions} \\
\cmidrule(lr){2-3}\cmidrule(lr){4-5}
Source & TOP$_{ll}$ & TOP$_{lt}$ & TOP$_{ll}$ & TOP$_{lt}$ \\
\midrule
TopoNet & 21.8 & 25.8 & \best{22.4} & \best{26.2} \\
TopoMLP & \best{24.4} & \best{28.7} & 24.1 & 28.4 \\
TopoLogic & \best{24.5} & 27.2 & 24.4 & \best{27.4} \\
SMART & \best{40.1} & \best{35.4} & 39.7 & 35.0 \\
\midrule
\multicolumn{5}{c}{\textbf{(b) Traffic-element feature backbone}} \\
\midrule
\multicolumn{3}{l}{Backbone} & TOP$_{ll}$ & TOP$_{lt}$ \\
\midrule
\multicolumn{3}{l}{DINOv2-ViT-L} & 21.8 & \best{25.8} \\
\multicolumn{3}{l}{ResNet-50} & 21.8 & 25.5 \\
\bottomrule
\end{tabular}
\end{table}

\textbf{Traffic-Element Feature Backbone.} We compare DINOv2-ViT-L and ResNet-50 for traffic-element node feature extraction (Table~\ref{tab:ablation_summary}(b)). They achieve TOP$_{ll}$/TOP$_{lt}$ of 21.84/25.80 and 21.80/25.50, respectively, with differences of only 0.04 in TOP$_{ll}$ and 0.30 in TOP$_{lt}$. Even with ResNet-50, the same visual backbone used by TopoNet, \ourmodel substantially outperforms TopoNet's 10.90/23.80. This indicates that the enhancement gains do not rely on a larger visual backbone.

\section{Conclusion}

Reliable topology connectivity is fundamental to autonomous driving, as downstream modules such as trajectory prediction and path planning rely on high-fidelity topology predictions. In this work, we identify a significant performance gap between existing topology reasoning methods and the theoretical upper-bound achievable given their detection predictions. To bridge this gap, we propose \ourmodel, a source-agnostic topology enhancement framework that refines topology connections using heterogeneous structural context. By reconstructing stochastically corrupted graphs, \ourmodel learns to resolve logical inconsistencies and recover structural consistency. Experiments on OpenLane-V2 show that \ourmodel consistently improves OpenLane-V2 TOP across diverse sources without retraining. These results highlight structural refinement as a key step toward closing the performance gap in topology reasoning. We hope our findings encourage future research to focus on producing robust, structurally consistent driving topology connections for safe and reliable autonomous driving.

\bibliography{iclr2027_conference}
\bibliographystyle{iclr2027_conference}

\clearpage
\appendix

\addcontentsline{toc}{section}{Appendix}
\part{Appendix}
\parttoc 
\newtheorem{proposition}{Proposition}
\newtheorem{corollary}{Corollary}
\newtheorem{remark}{Remark}
\newtheorem{lemma}{Lemma}

\section{Theoretical Diffusion-Based Interpretation of \ourmodel}
\label{app:diffusion}

In this section, we provide a theoretical diffusion-based interpretation of
\ourmodel (Sections~\ref{tns} and~\ref{sec:enhancement_model}). We first
show that the normalized relation-aware BCE loss
(Equation~\ref{eq:loss_definition}--Equation~\ref{eq:loss_all}) is a normalized Bernoulli negative log-likelihood (Lemma~\ref{lemma:bce}).
We then derive a full-graph expression for the learnable
reconstruction term of the negative $T{=}1$ diffusion ELBO
(Proposition~\ref{prop:diffusion}). Together, these results motivate
interpreting \ourmodel as a single-step diffusion-based denoising model on
topology graphs.

\paragraph{Setup.}
Let $\mathcal{G}^* = (\mathcal{V}^*, \mathcal{E}^*)$ denote a clean
ground-truth topology graph drawn from $p_{\mathrm{data}}(\mathcal{G})$.
The forward process (Section~\ref{tns}) defines a single-step corruption
distribution $q(\mathcal{G}_{\mathrm{noise}} \mid \mathcal{G}^*)$ that
stochastically transforms $\mathcal{G}^*$ into a noisy graph
$\mathcal{G}_{\mathrm{noise}} = (\mathcal{V}_{\mathrm{noise}},
\mathcal{E}_{\mathrm{noise}})$ by applying constrained Gaussian jitter to
true-positive detections and synthesizing false-positive detections. The
GNN component $g_\theta$ of the reverse process learns the conditional denoising distribution
$p_\theta(\mathcal{Y}\mid\mathcal{G}_{\mathrm{noise}})$ of ground-truth edge labels $\mathcal{Y}$, supervised via
the normalized BCE objective
(Equation~\ref{eq:loss_definition}--Equation~\ref{eq:loss_all}). We further assume that
(i) the per-edge model output $\tilde{s}_{uv}\in(0,1)$ defines a valid
Bernoulli probability, and (ii) candidate edges are conditionally independent
given $\mathcal{G}_{\mathrm{noise}}$.

\begin{lemma}\label{lemma:bce}
The training objective $\mathcal{L}$ of \ourmodel{}
(Equation~\ref{eq:loss_definition}--Equation~\ref{eq:loss_all}) is a relation-balanced,
volume-normalized Bernoulli negative log-likelihood over
the candidate edge set $\mathcal{E}_{\mathrm{noise}}$.
\end{lemma}

\begin{proof}
We define a correspondence map
$\pi:\mathcal{V}_{\mathrm{noise}}\to\mathcal{V}^*\cup\{\varnothing\}$,
where $\pi(v)\in\mathcal{V}^*$ if $v\in\mathcal{V}_{\mathrm{tp}}$ is a
jittered true-positive detection, and $\pi(v)=\varnothing$ if
$v\in\mathcal{V}_{\mathrm{fp}}$ is a synthesized false-positive detection. The
ground-truth edge label is then
\begin{equation}
y_{uv}
=
\mathbb{I}\!\left[
  \pi(u)\neq\varnothing,\;
  \pi(v)\neq\varnothing,\;
  (\pi(u),\pi(v))\in\mathcal{E}^*
\right],
\end{equation}
consistent with Equation~\ref{eq:gt_label}.

Under the conditional independence assumption stated in the Setup, the
reverse distribution factorizes as in Equation~\ref{eq:bernoulli}:
\begin{equation}
p_\theta(\mathcal{Y}\mid\mathcal{G}_{\mathrm{noise}})
=
\prod_{(u,v)\in\mathcal{E}_{\mathrm{noise}}}
\mathrm{Bernoulli}(y_{uv};\tilde{s}_{uv}).
\end{equation}
Using the Bernoulli probability mass function,
\begin{equation}
\mathrm{Bernoulli}(y_{uv};\tilde{s}_{uv})
=
\tilde{s}_{uv}^{\,y_{uv}}
(1-\tilde{s}_{uv})^{1-y_{uv}},
\end{equation}
we obtain
\begin{equation}
p_\theta(\mathcal{Y}\mid\mathcal{G}_{\mathrm{noise}})
=
\prod_{(u,v)\in\mathcal{E}_{\mathrm{noise}}}
\tilde{s}_{uv}^{\,y_{uv}}
(1-\tilde{s}_{uv})^{1-y_{uv}}.
\end{equation}
Taking the negative logarithm yields the Bernoulli negative log-likelihood,
i.e., binary cross-entropy over the candidate edge set:
\begin{equation}\label{eq:full_bce}
-\log p_\theta(\mathcal{Y}\mid\mathcal{G}_{\mathrm{noise}})
=
-\!\sum_{(u,v)\in\mathcal{E}_{\mathrm{noise}}}
\!\left[
  y_{uv}\log\tilde{s}_{uv}
  +
  (1-y_{uv})\log(1-\tilde{s}_{uv})
\right].
\end{equation}

Partitioning Equation~\ref{eq:full_bce} by relation type $r\in\{ll,lt\}$ and
normalizing each partition by its cardinality $|\mathcal{E}_r|$ yields exactly
$\mathcal{L}^{(r)}$ from Equation~\ref{eq:loss_definition}. Therefore, $\mathcal{L}$ is obtained from the Bernoulli negative log-likelihood
in Equation~\ref{eq:full_bce} by a normalization of its per-relation terms, and can be interpreted as a
relation-balanced training objective for the conditional denoising likelihood
$p_\theta(\mathcal{Y}\mid\mathcal{G}_{\mathrm{noise}})$.
\end{proof}

\begin{proposition}\label{prop:diffusion}
For a full-graph reverse likelihood, let $\mathcal{G}^* \sim p_{\mathrm{data}}(\mathcal{G})$ be a clean topology
graph and let $q(\mathcal{G}_{\mathrm{noise}} \mid \mathcal{G}^*)$ be the
single-step corruption process defined above, with a fixed prior
$p(\mathcal{G}_{\mathrm{noise}})$ independent of $\theta$.
Assume the expectations below are finite. Then minimizing
the denoising objective
\begin{equation}\label{eq:denoise_obj}
\mathcal{L}_{\mathrm{denoise}}
=
-\,\mathbb{E}_{%
  \mathcal{G}^*\sim p_{\mathrm{data}},\;
  \mathcal{G}_{\mathrm{noise}}\sim q(\mathcal{G}_{\mathrm{noise}}\mid\mathcal{G}^*)}
\!\left[
\log p_\theta\!\left(\mathcal{G}^*\mid\mathcal{G}_{\mathrm{noise}}\right)
\right]
\end{equation}
over $\theta$ is equivalent, up to an additive constant independent of
$\theta$, to minimizing the learnable reconstruction term of the negative $T{=}1$
diffusion ELBO.
\end{proposition}

\begin{proof}
We assume $p(\mathcal{G}_{\mathrm{noise}})$ is fixed and independent of
$\theta$, mirroring the standard diffusion setting in which the terminal prior
is pre-specified.

A $T$-step diffusion model defines a forward Markov chain
\begin{equation}
q(\mathcal{G}_{1:T}\mid\mathcal{G}^*)
=
\prod_{t=1}^{T} q(\mathcal{G}_t\mid\mathcal{G}_{t-1}),
\end{equation}
gradually corrupting $\mathcal{G}^*\!\equiv\!\mathcal{G}_0$ into
$\mathcal{G}_1,\ldots,\mathcal{G}_T$.
The learned reverse process is
\begin{equation}
p_\theta(\mathcal{G}^*,\mathcal{G}_{1:T})
=
p(\mathcal{G}_T)\prod_{t=1}^{T}p_\theta(\mathcal{G}_{t-1}\mid\mathcal{G}_t).
\end{equation}
Applying Jensen's inequality to the marginal log-likelihood yields the negative ELBO:
\begin{align}
\mathbb{E}_{\mathcal{G}^*\sim p_{\mathrm{data}}}[-\log p_\theta(\mathcal{G}^*)]
&\leq
\mathbb{E}_{\mathcal{G}^*\sim p_{\mathrm{data}},\,q(\mathcal{G}_{1:T}\mid\mathcal{G}^*)}
\!\left[
-\log\frac{p_\theta(\mathcal{G}^*,\mathcal{G}_{1:T})}{q(\mathcal{G}_{1:T}\mid\mathcal{G}^*)}
\right].
\end{align}
Expanding via the Markov factorizations and applying Bayes' rule to rewrite
each forward transition $q(\mathcal{G}_t\mid\mathcal{G}_{t-1})$ in terms of
the posterior $q(\mathcal{G}_{t-1}\mid\mathcal{G}_t,\mathcal{G}^*)$ gives
\begin{align}
\text{negative ELBO}
&=
\mathbb{E}_{\mathcal{G}^* \sim p_{\mathrm{data}}}
\Bigg[
\underbrace{%
  \mathbb{E}_{q(\mathcal{G}_1\mid\mathcal{G}^*)}
  \!\left[-\log p_\theta(\mathcal{G}^*\mid\mathcal{G}_1)\right]
}_{\mathcal{L}_{\mathrm{recon}}}
+
D_{\mathrm{KL}}\!\left(
  q(\mathcal{G}_T\mid\mathcal{G}^*)\,\|\,p(\mathcal{G}_T)
\right)
\notag\\
&\quad+
\sum_{t=2}^{T}
\mathbb{E}_{q(\mathcal{G}_t\mid\mathcal{G}^*)}
\!\left[
D_{\mathrm{KL}}\!\left(
  q(\mathcal{G}_{t-1}\mid\mathcal{G}_t,\mathcal{G}^*)
  \,\|\,
  p_\theta(\mathcal{G}_{t-1}\mid\mathcal{G}_t)
\right)
\right]
\Bigg].
\end{align}
Setting $T{=}1$ and identifying $\mathcal{G}_1\equiv\mathcal{G}_{\mathrm{noise}}$,
the intermediate sum vanishes and the negative ELBO reduces to
\begin{equation}\label{eq:elbo_t1}
\begin{aligned}
\mathbb{E}_{\mathcal{G}^*\sim p_{\mathrm{data}}}
[-\log p_\theta(\mathcal{G}^*)]
\leq &
\underbrace{%
  \mathbb{E}_{\mathcal{G}^*\sim p_{\mathrm{data}},\,
              \mathcal{G}_{\mathrm{noise}}\sim q(\mathcal{G}_{\mathrm{noise}}\mid\mathcal{G}^*)}
  \!\left[-\log p_\theta(\mathcal{G}^*\mid\mathcal{G}_{\mathrm{noise}})\right]
}_{\mathcal{L}_{\mathrm{recon}}}
\\
&+
\underbrace{%
  \mathbb{E}_{\mathcal{G}^*\sim p_{\mathrm{data}}}
  \!\left[
  D_{\mathrm{KL}}\!\left(
    q(\mathcal{G}_{\mathrm{noise}}\mid\mathcal{G}^*)
    \,\|\,
    p(\mathcal{G}_{\mathrm{noise}})
  \right)
  \right]
}_{\text{constant w.r.t.\ }\theta}.
\end{aligned}
\end{equation}
Since the KL term in Equation~\ref{eq:elbo_t1} depends only on the fixed forward
process and the fixed prior,
\begin{equation}
\frac{\partial}{\partial\theta}
\mathbb{E}_{\mathcal{G}^*\sim p_{\mathrm{data}}}
\!\left[
D_{\mathrm{KL}}\!\left(
  q(\mathcal{G}_{\mathrm{noise}}\mid\mathcal{G}^*)\,\|\,p(\mathcal{G}_{\mathrm{noise}})
\right)
\right]
=0.
\end{equation}
Minimizing the negative $T{=}1$ ELBO over $\theta$ therefore reduces to minimizing
$\mathcal{L}_{\mathrm{recon}}$. By the tower property, the nested expectation
$\mathbb{E}_{\mathcal{G}^*}\!\left[\mathbb{E}_{q(\mathcal{G}_{\mathrm{noise}}\mid\mathcal{G}^*)}[\,\cdot\,]\right]$
equals the joint expectation in Equation~\ref{eq:denoise_obj}, so
$\mathcal{L}_{\mathrm{recon}}=\mathcal{L}_{\mathrm{denoise}}$.
\end{proof}

Combining Lemma~\ref{lemma:bce} with Proposition~\ref{prop:diffusion}, the
normalized relation-aware BCE objective used by \ourmodel{} is a
relation-balanced analogue of the learnable reconstruction term in the negative of a
$T{=}1$ diffusion ELBO. Thus, \ourmodel{} can be interpreted as optimizing a
single-step diffusion-based denoising objective on topology graphs.

\section{Algorithmic Summary of Forward Process: Topology Noise Simulation}
\label{sec:algorithm}

The algorithm of the topology noise simulation is presented at Algorithm~\ref{alg:tns}.

\begin{algorithm}
\caption{Forward Process: Topology Noise Simulation}
\label{alg:tns}
\begin{algorithmic}[1]
\REQUIRE Ground-truth $\mathcal{G}^* = (\mathcal{V}^*, \mathcal{E}^*)$,
         thresholds $\delta_{type}$, training step $t$
\ENSURE  Noisy graph $\mathcal{G}_{\mathrm{noise}}$, labels $\mathcal{Y}$
\STATE Update $\sigma_{tp}(t),\,\sigma_{fp}(t)$ via adaptive noise curriculum
\STATE $\mathcal{V}_{tp} \leftarrow \{ v + \epsilon \mid v \in \mathcal{V}^*,\,
       \epsilon \sim \mathcal{N}(0, \sigma_{tp}^2\mathbf{I}),\,
       \mathrm{dist}(v, v') \leq \delta_{type} \}$
\STATE $\mathcal{V}_{fp} \leftarrow \{ v + \epsilon \mid v \in \mathcal{V}^*,\,
       \epsilon \sim \mathcal{N}(0, \sigma_{fp}^2\mathbf{I}),\,
       \mathrm{dist}(v, v_{fp}) > \delta_{type} \}$
\STATE $\mathcal{V}_{\mathrm{noise}} \leftarrow \mathcal{V}_{tp} \cup
       \mathcal{V}_{fp}$ \COMMENT{Scaled to match $\mathbb{E}[|\hat{\mathcal{V}}|]$}
\STATE $\mathcal{E}_{\mathrm{noise}} \leftarrow \mathcal{V}_{\mathrm{noise}}
       \times \mathcal{V}_{\mathrm{noise}}$ \COMMENT{Candidate edge space}
\STATE $y_{uv} \leftarrow \mathbb{I}[(u,v) \in \mathcal{E}^*]$
       \textbf{ for all } $(u,v) \in \mathcal{E}_{\mathrm{noise}}$
\RETURN $\mathcal{G}_{\mathrm{noise}},\, \mathcal{Y} = \{y_{uv}\}$
\end{algorithmic}
\end{algorithm}

\section{Implementation Details}
\label{sec:implement-details}
For topology noise simulation, the dynamic noise scheduler is initialized with $\sigma_{tp} = 1\mathrm{e}{-5}$ and $\sigma_{fp} = 10$. Following the OpenLane-V2 protocol, we set $\delta_l = 3.0$ and $\delta_t = 0.75$ as the distance thresholds for determining true-positive lane and traffic-element detections, respectively. To encode both semantic and geometric context, traffic-element node features $\mathbf{X}_t$ are initialized using 1024-dimensional DinoV2-ViT-L embeddings~\citep{oquab2023dinov2} extracted from front-view images. Lane and traffic features ($\mathbf{X}_l, \mathbf{X}_t$) are then projected into a shared latent space via separate MLP heads, followed by a two-layer GATv2 encoder~\citep{brody2021attentive} implemented in \texttt{PyG}~\citep{Fey/Lenssen/2019} with a hidden dimension of 64, 8 attention heads, ReLU activation, and a dropout rate of 0.1. The model is trained with a batch size of 64 using the AdamW optimizer~\citep{loshchilov2017decoupled}, with an initial learning rate of $0.001$, weight decay of $0.01$, and a CosineAnnealingLR scheduler~\citep{loshchilov2016sgdr} that decays the learning rate to $10^{-4}$. During inference, the fusion weight is fixed at 0.6 for lane–traffic relations and varies between 0.6 and 0.9 for lane–lane relations. We evaluate the refined topology scores following the official OpenLane-V2 protocol~\citep{wang2023openlanev2}. All experiments are conducted on a single NVIDIA H200 GPU. \ourmodel requires approximately 16 hours for training. Inference latency and peak GPU memory, including traffic-element feature extraction, are reported in Appendix~\ref{appendix:latency}.

\section{Additional Evaluation Results}
\label{sec:additional_evaluation}

In this section, we provide additional experiments to further evaluate \ourmodel from more comprehensive perspectives. We extend our evaluation to
TopoPoint~\citep{fu2025topopoint}, a recent strong topology reasoning method in Appendix~\ref{sec:topopoint}. Additionally, in Appendix~\ref{sec:complete_metrics}, although \ourmodel does not modify or refine the upstream detection
results, we provide the complete OpenLane-V2 evaluation metrics, including
detection metrics (DET) and the overall OpenLane-V2 score (OLS), as a complement to the topology results reported in the main paper. Finally, the ablation studies on different node feature extractor, loss function, geometric localization jitter, and the noise simulation module are presented in Appendix~\ref{app:ablation}.

\subsection{Enhancing TopoPoint Predictions}
\label{sec:topopoint}

The experiments presented in the main paper evaluate methods with officially released checkpoints,
since applying \ourmodel requires sample-level lane detections, traffic-element
detections, and pairwise topology confidence matrices rather than only the
aggregate metrics reported in the corresponding papers. TopoPoint~\citep{fu2025topopoint} does not
provide pretrained checkpoints for both OpenLane-V2 subsets. Nevertheless,
because TopoPoint is a recent strong topology reasoning method, we additionally
train it on both Argoverse2 (Subset\_A) and nuScenes (Subset\_B) using its
official implementation\footnote{https://github.com/Franpin/TopoPoint} and evaluate \ourmodel on the resulting predictions.
\begin{table}[t]
\centering
\small
\caption{Additional evaluation of \ourmodel on TopoPoint. ``Reported''
denotes the results reported in the TopoPoint paper, while ``our reproduction''
uses checkpoints trained by us with the official implementation. }
\label{tab:topopoint}
\begin{tabular}{llcccc}
\toprule
\text{Dataset} &
\text{Method} &
\textbf{DET$_l$} &
\textbf{DET$_t$} &
\textbf{TOP$_{ll}$} &
\textbf{TOP$_{lt}$} \\
\midrule

\multirow{4}{*}{\makecell{Argoverse2\\(Subset\_A)}}
& TopoPoint, reported
& 31.40 & 55.30 & 28.70 & 30.00 \\
& TopoPoint, our reproduction
& 28.82 & 54.20 & 28.10 & 30.36 \\
& + \ourmodel
& 28.82 & 54.20 & \best{28.21} & \best{32.28}
\\
& Upper-Bound
& 28.82 & 54.20 & 29.65 & 36.58 \\

\midrule

\multirow{4}{*}{\makecell{nuScenes\\(Subset\_B)}}
& TopoPoint, reported
& 31.20 & 60.20 & 28.30 & 27.10 \\
& TopoPoint, our reproduction
& 23.67 & 60.66 & 24.73 & 22.42 \\
& + \ourmodel
& 23.67 & 60.66 & \best{25.00} & \best{23.78}
\\
& Upper-Bound
& 23.67 & 60.66 & 26.73 & 27.08 \\

\bottomrule
\end{tabular}
\end{table}

On Argoverse2, our reproduction closely matches the topology performance
reported by TopoPoint, obtaining 28.10 versus 28.70 on TOP$_{ll}$ and 30.36
versus 30.00 on TOP$_{lt}$.  In particular, TOP$_{lt}$ increases from 30.36 to 32.28. This demonstrates that \ourmodel remains complementary to TopoPoint's endpoint-aware topology reasoning architecture and can further outperform TopoPoint on both topology metrics under the same experimental setting.

We also attempted to reproduce TopoPoint on nuScenes (Subset\_B). Unfortunately, the released repository does not
provide a Subset\_B configuration, and the corresponding code path contains
dataset-specific inconsistencies in lane dimensionality and image processing.
After implementing the necessary fixes, our reproduced traffic-element
detection performance closely matches the reported result
(DET$_t=60.66$ versus $60.20$), while lane detection remains lower
(DET$_l=23.67$ versus $31.20$). We therefore report these results explicitly
as our reproduction rather than an exact reproduction of the published
Subset\_B result. Despite the weaker reproduced lane detections, \ourmodel
still improves TOP$_{ll}$ and TOP$_{lt}$.

\subsection{Complete OpenLane-V2 Metrics}
\label{sec:complete_metrics}

Table~\ref{tab:complete_top_results} complements
Table~\ref{tab:top_results} in the main paper by reporting the complete
OpenLane-V2 metric suite, including DET$_l$, DET$_t$, and OLS in addition to
TOP$_{ll}$ and TOP$_{lt}$. Since \ourmodel only refines topology connections
and does not modify the upstream map-element detections, DET$_l$ and DET$_t$
remain identical before and after enhancement. We therefore report each
detection metric once for every baseline.

OLS jointly evaluates detection and topology reasoning under the standard
OpenLane-V2 protocol. We recompute OLS using the same detection results and
the topology scores reported in Table~\ref{tab:top_results}, ensuring that the
overall scores are consistent with our OpenLane-V2 v2.1.0 evaluation. Because
the detection terms remain fixed, any improvement in OLS after applying
\ourmodel is entirely attributable to improved topology reasoning.

\begin{table}[t]
\centering
\caption{Complete OpenLane-V2 evaluation corresponding to
Table~\ref{tab:top_results}. DET$_l$ and DET$_t$ measure lane and
traffic-element detection, respectively. Since \ourmodel does not modify
detections, these metrics remain unchanged. OLS is computed from the displayed
detection and topology metrics following the standard OpenLane-V2 evaluation
protocol.}
\label{tab:complete_top_results}
\resizebox{\textwidth}{!}{
\begin{tabular}{llccccc@{\hskip 0.10in}ccc@{\hskip 0.10in}ccc}
\toprule
\multirow{2}{*}{\text{Dataset}} &
\multirow{2}{*}{\text{Method}} &
\multicolumn{2}{c}{\textbf{Detection}} &
\multicolumn{3}{c}{\textbf{Before \ourmodel}} &
\multicolumn{3}{c}{\textbf{After \ourmodel}} &
\multicolumn{3}{c}{\textbf{Upper-Bound}} \\
\cmidrule(r){3-4}
\cmidrule(r){5-7}
\cmidrule(r){8-10}
\cmidrule(r){11-13}
&
&
DET$_l \uparrow$ &
DET$_t \uparrow$ &
TOP$_{ll} \uparrow$ &
TOP$_{lt} \uparrow$ &
OLS $\uparrow$ &
TOP$_{ll} \uparrow$ &
TOP$_{lt} \uparrow$ &
OLS $\uparrow$ &
TOP$_{ll} \uparrow$ &
TOP$_{lt} \uparrow$ &
OLS $\uparrow$ \\
\midrule

\multirow{5}{*}{\makecell{Argoverse2\\(Subset\_A)}}
& TopoNet~\citep{li2023toponet}
& 28.6 & 48.6
& 10.9 & 23.8 & 39.8
& \best{21.8} & \best{25.8} & \best{43.7}
& 24.4 & 28.8 & 45.1 \\

& TopoMLP~\citep{wu2023topomlp}
& 28.5 & 49.5
& 21.6 & 26.9 & 44.1
& \best{24.4} & \best{28.7} & \best{45.2}
& 25.0 & 31.4 & 46.0 \\

& Topo2D~\citep{topo2d}
& 29.1 & 50.6
& 22.3 & 26.2 & 44.5
& \best{24.3} & \best{27.3} & \best{45.3}
& 25.0 & 29.8 & 46.1 \\

& TopoLogic~\citep{fu2024topologic}
& 29.9 & 47.2
& 23.9 & 25.4 & 44.1
& \best{24.5} & \best{27.2} & \best{44.7}
& 26.0 & 30.6 & 45.9 \\

& SMART~\citep{smart}
& 46.6 & 47.7
& 37.0 & 33.0 & 53.1
& \best{40.1} & \best{35.4} & \best{54.3}
& 40.9 & 38.9 & 55.2 \\

\specialrule{.1em}{.1em}{.1em}

\multirow{2}{*}{\makecell{nuScenes\\(Subset\_B)}}
& TopoNet~\citep{li2023toponet}
& 24.3 & 55.0
& 6.7 & 16.7 & 36.5
& \best{19.5} & \best{17.8} & \best{41.4}
& 21.2 & 20.3 & 42.6 \\

& TopoMLP~\citep{wu2023topomlp}
& 26.6 & 58.3
& 21.0 & 19.8 & 43.8
& \best{22.8} & \best{21.5} & \best{44.8}
& 23.5 & 24.1 & 45.6 \\

\bottomrule
\end{tabular}}
\end{table}

As shown in Table~\ref{tab:complete_top_results}, improving topology reasoning
also consistently improves the overall OpenLane-V2 score despite leaving the
detection outputs unchanged. For example, \ourmodel increases OLS from 36.5
to 41.4 for TopoNet on nuScenes and from 53.1 to 54.3 for SMART on
Argoverse2. Moreover, the enhanced OLS values consistently move toward their
corresponding detection-conditioned Upper-Bounds. These results further
demonstrate that existing detection outputs contain substantial unrealized
potential and that improving connectivity alone can translate into meaningful
gains in overall driving-scene topology performance.

\subsection{Ablation Studies}
\label{app:ablation}
\label{sec:training_sources}
To evaluate the effectiveness of our core component, we conduct ablation studies to examine the impact of different visual feature extractors, loss functions, geometric localization jitter, and the proposed noise simulation module.

\begin{table}[ht]
\centering
\renewcommand{\arraystretch}{0.9}
\caption{Ablation studies on different visual backbones for traffic-element node feature extraction. All experiments are conducted using TopoNet as the baseline on Argoverse2 (Subset\_A).}
\label{tab:ablation_backbone}
\begin{tabular}{l|c|cc}
\toprule
\text{Backbone Model} & \text{Dim} & \textbf{$\text{TOP}_{ll}$} & \textbf{$\text{TOP}_{lt}$} \\
\midrule

DINOv2-ViT-S  & 384  & 20.30 & 24.64 \\
DINOv2-ViT-B  & 768  & 21.80 & 25.06 \\
DINOv2-ViT-L  & 1024 & \best{21.84} & \best{25.80} \\
DINOv2-ViT-G  & 1536 & \best{21.84} & 25.75 \\

\midrule
DINOv3-ViT-S  & 384  & 21.56 & 24.48 \\
DINOv3-ViT-B  & 768  & 21.83 & 24.35 \\
DINOv3-ViT-L  & 1024 & 21.83 & 24.49 \\

\midrule
ResNet-50     & 256  & 21.80 & 25.50 \\
\bottomrule
\end{tabular}
\end{table}

\textbf{Traffic-Element Node Feature Extractor.} We investigate the impact of various visual backbones for extracting traffic-element (TE) node features, as the semantic richness of these embeddings is critical for establishing lane--traffic associations. We compare several variants of DINOv2 and DINOv3 across different parameter scales (Small, Base, Large, and Giant), as well as a standard ResNet-50 baseline. As shown in Table~\ref{tab:ablation_backbone}, the performance on lane--lane topology ($\text{TOP}_{ll}$) remains relatively stable across most high-capacity backbones, suggesting that lane-to-lane connectivity is primarily driven by geometric and structural priors rather than deep visual semantics of traffic signs. However, lane--traffic topology ($\text{TOP}_{lt}$) exhibits higher sensitivity to the choice of feature extractor. While DINOv3 variants provide competitive results, the DINOv2-ViT-L backbone achieves the highest overall performance ($\text{TOP}_{lt}$ of 25.8). Interestingly, increasing the model capacity further to DINOv2-ViT-G does not yield additional gains, indicating a point of diminishing returns where the feature dimensionality of 1536 may begin to overfit the relatively sparse traffic-element distribution. At the same time, DINOv2-ViT-L exceeds ResNet-50 by 0.04 on $\text{TOP}_{ll}$ (21.84 versus 21.80) and 0.30 on $\text{TOP}_{lt}$ (25.80 versus 25.50). This demonstrates that our method is largely backbone agnostic and remains effective even with lightweight feature extractors. Based on these results, we adopt DINOv2-ViT-L as our default TE feature extractor to balance semantic depth with computational efficiency. It is also worth noting that TopoNet uses ResNet 50 as its visual backbone. Even under the same backbone setting, our method significantly outperforms TopoNet by 10.90 on $\text{TOP}_{ll}$ (21.80 versus 10.90) and 1.70 on $\text{TOP}_{lt}$ (25.50 versus 23.80), as shown in Table~\ref{tab:ablation_backbone} and Table~\ref{tab:top_results}. Overall, these results suggest that while stronger backbones provide marginal gains, the main performance improvement comes from our noise driven training strategy and refinement design, making the method robust and broadly applicable across different feature sources.

\begin{table}[ht]
\centering
\renewcommand{\arraystretch}{0.9}
\caption{Ablation studies on loss functions. This experiment is conducted to enhance TopoNet's predictions on Argoverse2 (Subset\_A) under a consistent training configuration.}
\label{tab:ablation_loss}
\begin{tabular}{c}
\begin{tabular}{l|cc}

\toprule
\text{Loss} & \textbf{TOP$_{ll}$} & \textbf{TOP$_{lt}$} \\
\midrule

BCE                 & 12.1 & 22.0 \\

BCE + Focal   & \best{22.2} & 22.0 \\

Normalized BCE (Ours) & 21.8 & \best{25.8} \\

\bottomrule

\end{tabular}
\end{tabular}
\end{table}

\textbf{Loss Function.} We compare our normalized BCE loss against standard Binary Cross-Entropy (BCE) and a hybrid BCE–Focal variant. As shown in Table~\ref{tab:ablation_loss}, standard BCE performs poorly due to the extreme sparsity and structural imbalance of the topology graph. In standard formulations, the total loss is dominated by the more populous lane--lane candidate space, causing the model to neglect the sparser but equally critical lane--traffic associations. While the hybrid BCE–Focal loss successfully stabilizes lane--lane topology ($\text{TOP}_{ll}$), it lacks a mechanism to balance the relative importance between heterogeneous relationship types, leading to suboptimal performance in traffic-element associations. In contrast, our normalized BCE achieves the most balanced overall performance by introducing relation-aware scaling. By normalizing the loss of each category by its specific candidate edge volume, we ensure that gradient updates are not overwhelmed by the density of the lane-lane relationships. This balancing mechanism ensures that the enhancer effectively captures the subtle structural dependencies of both relation types, preventing the optimization process from being biased toward the high-frequency connectivity patterns. By automatically adapting to varying connectivity densities across different scene geometries, this approach minimizes the reliance on sensitive, manually tuned weighting hyperparameters, making it a more robust choice for a universal topology enhancer.

\begin{table}[ht]
\centering

\renewcommand{\arraystretch}{0.9}
\caption{Ablation studies on geometric localization jitter, while keeping other components from topology noise simulation unchanged. This experiment is conducted to enhance TopoNet's predictions on Argoverse2 (Subset\_A) under a consistent training configuration.}
\label{tab:ablation_jitter}
\begin{tabular}{c}
\begin{tabular}{l|cc}

\toprule

\text{Strategy} & \textbf{TOP$_{ll}$} & \textbf{TOP$_{lt}$} \\

\midrule

\ourmodel & \best{21.8} & \best{25.8} \\

\ourmodel without Jitter & 21.7 & 25.4 \\

\bottomrule
\end{tabular}
\end{tabular}
\end{table}

\textbf{Geometric Localization Jitter.} Although our primary objective is to rectify structural connectivity, the geometric precision of map elements significantly influences the learning process.  In our framework, geometric localization jitter serves as a crucial regularizer. By introducing constrained, semantics-preserving perturbations to the node coordinates of ground-truth detections ($v \in \mathcal{V}^*$), we prevent the model from overfitting to idealized geometries. This jittering forces the heterogeneous GNN to prioritize invariant structural patterns and relational context over exact coordinate matching, which is essential for generalizing to the noisy predictions encountered during real-world inference. We emphasize that Table~\ref{tab:ablation_jitter} isolates the effect of the geometric jitter component, rather than ablating the entire noise simulation strategy. Since the full framework already incorporates adaptive noise curriculum and false positive detection synthesis to simulate realistic perception errors, the additional gains from geometric jitter are expected and reflect its complementary role. Nevertheless, we observe consistent improvements across both metrics, including a notable +0.4 gain in $\text{TOP}_{lt}$. This indicates that geometric jitter effectively complements the modeling of structural inconsistencies by introducing realistic localization noise. As a result, the \ourmodel becomes more robust to the coordinate inaccuracies commonly present in diverse perception systems, without compromising the learned topological priors.


\textbf{Topology Noise Simulation.} To evaluate how the training source affects topology enhancement and generalization across methods, we compare \ourmodel trained on noise-simulated graphs with a variant trained on real topology predictions. The latter uses predictions generated by the official TopoNet Subset\_A checkpoint on the training split. Both variants use the same GNN architecture and node features described in Appendix~\ref{sec:implement-details}. For training on real predictions, predicted detections are matched to ground truth using the OpenLane-V2 matching procedure, and an edge is labeled positive only when both endpoints are matched and their corresponding ground-truth entities are connected. Table~\ref{tab:training_sources}(a) reveals different trends when evaluating on TopoNet and transferring to other methods. Training on TopoNet's real predictions achieves TOP$_{ll}$/TOP$_{lt}$ of 22.4/26.2 on TopoNet, exceeding simulated-graph training by 0.6 in TOP$_{ll}$ and 0.4 in TOP$_{lt}$. However, this advantage does not consistently extend to other prediction sources. Simulated-graph training yields higher scores on both TopoMLP and SMART. The results on TopoLogic are mixed: simulated training achieves slightly higher TOP$_{ll}$ (24.5 versus 24.4), while real-prediction training achieves slightly higher TOP$_{lt}$ (27.4 versus 27.2). Overall, simulated training performs better in five of the six comparisons across these three other methods. This pattern suggests that training on a specific method's predictions can benefit enhancement of that method, but may not provide a consistent advantage when transferring to others. Our noise simulation offers competitive performance across prediction sources while removing the need to generate predictions from a pretrained upstream model for enhancer training.

\section{Practical Usability Toward Real-World Applications}
\label{sec:practical_evaluation}

In this section, we provide additional experiments to further evaluate the practical usability of \ourmodel from multiple perspectives. We first analyze the robustness of \ourmodel under varying levels of topology corruption in Appendix~\ref{sec:noise_robustness} to examine whether it remains effective when the quality of the input topology varies substantially. We then evaluate temporal consistency across consecutive frames in Appendix~\ref{sec:temporal_consistency} to study whether the refined topology predictions remain consistent over time. Additionally, in Appendix~\ref{sec:safety_errors}, we analyze safety-relevant topology errors to examine whether \ourmodel corrects connections that are directly related to driving behaviors. We further evaluate scene-level ego-route recoverability in Appendix~\ref{sec:route_recoverability} to assess whether the refined topology graphs better preserve complete driving routes for downstream navigation and planning. We also evaluate downstream trajectory prediction with LaneGCN in Appendix~\ref{app:trajectory_prediction}. Finally, the computational efficiency of \ourmodel, including inference latency and peak GPU memory, is presented in Appendix~\ref{appendix:latency}.

\subsection{Robustness to Varying Levels of Topology Noise}
\label{sec:noise_robustness}

In practical autonomous driving systems, the quality of topology predictions
can vary substantially across upstream models and driving scenes. A topology
enhancement method should therefore remain effective not only for high-quality
predictions, but also when the input connectivity becomes increasingly
unreliable. Evaluating robustness under different levels of topology noise is
particularly important for \ourmodel because it is designed as a
source-agnostic enhancement module that can be applied to predictions from
different upstream methods. A controlled corruption experiment provides a
systematic way to stress-test this capability by gradually degrading the input
topology and measuring how much connectivity can still be recovered.

To directly evaluate the denoising capability of \ourmodel under different
levels of input topology quality, we conduct a controlled corruption experiment
on the Argoverse2 (Subset\_A) validation set. For each predicted topology score
$\hat{s}_{uv}$, we independently replace it with $1-\hat{s}_{uv}$ with
probability $p$. We evaluate
$p\in\{0.0,0.1,0.2,0.3,0.5,0.7,0.9\}$ on both TopoNet and SMART, which have
substantially different initial topology performance. The same trained
\ourmodel is applied at every corruption level without retraining or
adaptation. Here, $p=0$ corresponds to the uncorrupted predictions reported
in the main paper.

As shown in Table~\ref{tab:noise_robustness}, although $p=0.1$ corrupts only a
small fraction of topology scores, it already causes a substantial decrease in
TOP. This nonlinear degradation is expected because TOP is based on mean
Average Precision (mAP). Replacing a score $\hat{s}_{uv}$ with
$1-\hat{s}_{uv}$ can move confident false-positive connections toward the top
of the ranking while pushing true-positive connections downward. Since Average
Precision is particularly sensitive to errors near the top of the ranked list,
even a small amount of corruption can cause a large performance drop. As the
ranking becomes increasingly disrupted, additional corruption has a smaller
marginal effect. For $p \geq 0.1$, increasing $p$ progressively degrades the
input topology quality for both baselines and both relation types.

Nevertheless, \ourmodel consistently improves TOP$_{ll}$ and TOP$_{lt}$ across
all evaluated settings. For TopoNet, the improvement in TOP$_{ll}$ ranges from
0.77 to 10.90, while the improvement in TOP$_{lt}$ ranges from 0.47 to 2.00.
For SMART, \ourmodel improves TOP$_{ll}$ by 1.79 to 3.10 and TOP$_{lt}$ by
0.61 to 2.40 across the full range of input quality. Even at the most severe
corruption level, $p=0.9$, \ourmodel improves SMART from 0.33 to 2.20 on
TOP$_{ll}$ and from 2.20 to 2.81 on TOP$_{lt}$.

These results show that \ourmodel remains effective as the quality of the input
topology varies substantially. Its consistent recovery of performance under
controlled score corruption provides direct empirical evidence of its
topology-denoising capability. Specifically, \ourmodel can recover plausible
connectivity from degraded topology predictions by leveraging geometric,
semantic, and heterogeneous structural context, demonstrating that its
denoising behavior remains effective across a broad range of input topology
quality.

\begin{table}[t]
\centering
\caption{Robustness analysis under varying levels of topology noise on
Argoverse2 (Subset\_A). }
\label{tab:noise_robustness}
\resizebox{\textwidth}{!}{
\begin{tabular}{lccccc}
\toprule
\text{Model} &
\text{Noise} &
\multicolumn{2}{c}{\textbf{TOP$_{ll}$}} &
\multicolumn{2}{c}{\textbf{TOP$_{lt}$}} \\
\cmidrule(lr){3-4}
\cmidrule(lr){5-6}
&
&
\textbf{Before \ourmodel} &
\textbf{After \ourmodel} &
\textbf{Before \ourmodel} &
\textbf{After \ourmodel} \\
\midrule

\multirow{7}{*}{TopoNet}
& $p=0.0$ & 10.90 & \best{21.80} & 23.80 & \best{25.80} \\
& $p=0.1$ & 4.62 & \best{9.99} & 15.81 & \best{17.11} \\
& $p=0.2$ & 2.78 & \best{6.14} & 11.31 & \best{12.35} \\
& $p=0.3$ & 1.99 & \best{4.39} & 8.48 & \best{9.39} \\
& $p=0.5$ & 1.30 & \best{2.77} & 4.94 & \best{5.64} \\
& $p=0.7$ & 0.99 & \best{2.02} & 2.93 & \best{3.52} \\
& $p=0.9$ & 0.82 & \best{1.59} & 1.85 & \best{2.32} \\

\midrule

\multirow{7}{*}{SMART}
& $p=0.0$ & 37.00 & \best{40.10} & 33.00 & \best{35.40} \\
& $p=0.1$ & 12.65 & \best{15.03} & 21.84 & \best{23.58} \\
& $p=0.2$ & 6.59 & \best{8.95} & 15.61 & \best{17.00} \\
& $p=0.3$ & 4.01 & \best{6.26} & 11.46 & \best{12.68} \\
& $p=0.5$ & 1.81 & \best{3.89} & 6.49 & \best{7.41} \\
& $p=0.7$ & 0.99 & \best{2.78} & 3.89 & \best{4.61} \\
& $p=0.9$ & 0.33 & \best{2.20} & 2.20 & \best{2.81} \\

\bottomrule
\end{tabular}}
\end{table}

\subsection{Temporal Consistency of Refined Topology Predictions}
\label{sec:temporal_consistency}

In autonomous driving, topology predictions should remain stable across
consecutive frames when the underlying scene structure does not change.
Frequent changes in predicted connectivity or large fluctuations in topology
scores can lead to inconsistent scene understanding and potentially affect
downstream decision-making. We therefore evaluate whether the topology
refinement produced by \ourmodel improves the temporal consistency of predicted
connections across consecutive frames.

We evaluate TopoNet, TopoMLP, TopoLogic, and SMART on the Argoverse2
(Subset\_A) validation set. To establish correspondence across frames, we use
the persistent lane segment IDs inherited from the underlying Argoverse2 HD
map. For every ground-truth lane pair that remains co-visible for at least
$K=5$ consecutive frames, we compute an \emph{edge flip rate}, defined as the
fraction of adjacent frame pairs for which the predicted edge changes between
positive and negative. We additionally measure the temporal standard deviation
of topology scores to quantify the variation in connection confidence across
consecutive frames.

As shown in Table~\ref{tab:temporal_consistency},
\ourmodel reduces both the edge flip rate and the temporal variation of topology
scores across all four baselines. The lower edge flip rate indicates that the
refined connections are more consistent across consecutive frames,
while the lower temporal standard deviation indicates more stable topology
confidence scores. Such temporal stability is practically important because
downstream modules, such as trajectory prediction and path planning, rely on
a consistent understanding of the scene topology over time. Reducing
frame-to-frame fluctuations in predicted connectivity can help provide a more
stable topology graph to these downstream modules and reduce inconsistent
decisions caused by transient topology changes. Therefore, in addition to
improving topology accuracy, \ourmodel produces temporally more consistent
predictions that are better suited for use in an autonomous driving pipeline.

\subsection{Safety-Relevant Topology Error Analysis}
\label{sec:safety_errors}

\begin{table}[t]
\centering
\scriptsize
\caption{Safety-relevant topology errors corrected and introduced by
\ourmodel on Argoverse2 (Subset\_A), grouped into lane--lane and lane--traffic
relations. Lane--lane errors include missed ego successors and intersection
turns, while lane--traffic errors include missed associations with traffic
signals. Only relations whose endpoint detections are successfully matched to
ground truth are evaluated. ``Corrected'' denotes a previously missing
ground-truth connection recovered after refinement, while ``Broken'' denotes an
originally correct connection that becomes incorrect. }
\label{tab:safety_errors}
\resizebox{0.96\textwidth}{!}{
\begin{tabular}{lllrrr}
\toprule
\text{Model} &
\text{Relation} &
\text{Error Type} &
\textbf{Total} &
\textbf{Corrected} &
\textbf{Broken} \\
\midrule

\multirow{6}{*}{TopoNet}
& \multirow{2}{*}{Lane--Lane}
& Missed ego successor       & 243   & 60    & 0 \\
&
& Missed intersection turn   & 5,378 & 1,502 & 0 \\
\cmidrule(lr){2-6}
&
\multirow{4}{*}{Lane--Traffic}
& Missed red light           & 153   & 103   & 0 \\
&
& Missed green light         & 398   & 311   & 0 \\
&
& Missed yellow light        & 54    & 23    & 0 \\
&
& Missed turn signal         & 49    & 36    & 0 \\

\midrule

\multirow{6}{*}{TopoMLP}
& \multirow{2}{*}{Lane--Lane}
& Missed ego successor       & 243   & 13  & 0 \\
&
& Missed intersection turn   & 5,353 & 555 & 2 \\
\cmidrule(lr){2-6}
&
\multirow{4}{*}{Lane--Traffic}
& Missed red light           & 150   & 112 & 0 \\
&
& Missed green light         & 385   & 279 & 0 \\
&
& Missed yellow light        & 54    & 30  & 0 \\
&
& Missed turn signal         & 39    & 25  & 0 \\

\midrule

\multirow{6}{*}{TopoLogic}
& \multirow{2}{*}{Lane--Lane}
& Missed ego successor       & 244   & 7   & 0 \\
&
& Missed intersection turn   & 5,437 & 102 & 0 \\
\cmidrule(lr){2-6}
&
\multirow{4}{*}{Lane--Traffic}
& Missed red light           & 158   & 102 & 0 \\
&
& Missed green light         & 391   & 311 & 0 \\
&
& Missed yellow light        & 58    & 22  & 0 \\
&
& Missed turn signal         & 45    & 30  & 0 \\

\midrule

\multirow{6}{*}{SMART}
& \multirow{2}{*}{Lane--Lane}
& Missed ego successor       & 250   & 11  & 0 \\
&
& Missed intersection turn   & 5,902 & 489 & 16 \\
\cmidrule(lr){2-6}
&
\multirow{4}{*}{Lane--Traffic}
& Missed red light           & 150   & 107 & 0 \\
&
& Missed green light         & 398   & 316 & 0 \\
&
& Missed yellow light        & 63    & 43  & 0 \\
&
& Missed turn signal         & 47    & 42  & 0 \\

\bottomrule
\end{tabular}}
\end{table}

Aggregated topology metrics such as TOP measure overall connectivity
quality, but they do not indicate whether the corrected connections are
particularly relevant to driving behavior. In practical autonomous driving,
different topology errors can have very different consequences. For example,
a missing successor along the ego route can interrupt a feasible driving path,
while a missing lane--traffic association can prevent the topology graph from
correctly identifying the traffic control governing a lane. We therefore
examine whether the improvements produced by \ourmodel extend to topology
relations that are directly relevant to navigation and traffic-rule
understanding. For lane--lane topology, we consider \emph{missed ego
successors}, where a connection along the ego vehicle's driven path is missing,
and \emph{missed intersection turns}, where a valid turning connection at an
intersection or connector is absent. For lane--traffic topology, we evaluate missed associations with red, green,
and yellow traffic lights, as well as turn signals (i.e.,
\texttt{no\_left\_turn}, \texttt{no\_right\_turn}, \texttt{turn\_left} and
\texttt{turn\_right}). We evaluate only topology connections whose endpoints
are successfully detected and matched, thereby isolating topology refinement
from upstream detection quality. In Table~\ref{tab:safety_errors},
``Corrected'' denotes a previously missing ground-truth connection recovered
by \ourmodel, while ``Broken'' denotes an originally correct connection that
becomes incorrect after enhancement.

As shown in Table~\ref{tab:safety_errors}, \ourmodel recovers many
driving-relevant connections while introducing very few new errors. Across the
four baselines, it corrects 7--60 missed ego successors and 102--1,502 missed
intersection-turn connections, while breaking at most 16 intersection-turn
connections for SMART, 2 for TopoMLP, and none for TopoNet or TopoLogic. For lane--traffic
topology, \ourmodel recovers 102--112 missed red-light associations, 279--316
missed green-light associations, 22--43 missed yellow-light associations, and
25--42 missed turn-signals associations, without introducing
new errors in any of these traffic-element categories.

These results provide a more practically meaningful view of the improvements
than aggregated topology metrics alone. Recovering ego-route successors and
valid intersection turns helps preserve the connectivity of feasible driving
paths, while recovering lane--traffic associations provides more complete
information about the traffic controls and maneuver constraints governing
those paths. At the same time, the very small number of newly introduced
errors indicates that these gains are achieved without substantially
disrupting connections that were already correct. Overall, this analysis shows
that the improvements from \ourmodel extend to topology relations that are
directly relevant to how downstream autonomous driving modules interpret and
use the driving-scene topology.

\subsection{Scene-Level Ego-Route Recoverability for Navigation and Planning}
\label{sec:route_recoverability}
While correcting individual topology connections improves topology quality at local frames,
downstream navigation and planning require a continuous sequence of valid lane
transitions along a driving route. Even if most individual connections are
correct, a single erroneous transition can prevent the complete route from
being recovered. We therefore evaluate scene-level ego-route recoverability to
determine whether the local connectivity improvements produced by \ourmodel
translate into more complete driving routes for downstream navigation and
planning.

Since OpenLane-V2 provides frame-level topology annotations, we first
reconstruct a scene-level lane graph by aggregating topology predictions across
frames using the persistent lane IDs inherited from the underlying Argoverse2
HD map. To isolate topology quality from upstream detection errors, we evaluate
ego-route recoverability only when the lane detections required along the
driven route are successfully detected and matched. We then examine whether
the ego vehicle's true driving route can be recovered from the predicted
scene-level road topology graph. Specifically, a route is considered recoverable only if
every lane transition along the ego vehicle's complete driven trajectory,
typically spanning 70--110 meters, is correctly preserved in the predicted
graph. This provides a strict evaluation of route-level topology quality,
since a single missing or incorrect transition can prevent the complete route
from being recovered.

As shown in Table~\ref{tab:route_recoverability}, \ourmodel consistently
improves full-route recoverability across all four baselines, from 39.19\% to
68.92\% for TopoNet, 94.81\% to 100\% for TopoMLP, 92.41\% to 93.67\% for
TopoLogic, and 91.36\% to 98.77\% for SMART. These improvements show that the
benefits of topology refinement extend beyond correcting isolated edges and
lead to more complete scene-level driving paths.

This result has direct practical relevance because downstream navigation and
planning modules depend on consecutive lane connectivity to identify feasible
routes through a driving scene. A topology graph can achieve strong aggregate
performance while still being difficult to use if a small number of missing
connections interrupt an otherwise valid route. By increasing the fraction of
scenes in which the complete ego route can be recovered, \ourmodel produces
topology graphs that preserve the structural information needed to support
route-level reasoning more reliably. Although this experiment does not
directly evaluate a downstream planner, it provides complementary evidence
that the improvements from \ourmodel extend beyond aggregate topology metrics
to a structural property that is directly relevant to the practical use of
topology graphs for autonomous navigation and planning.


\subsection{Downstream Trajectory Prediction with LaneGCN}
\label{app:trajectory_prediction}
LaneGCN~\citep{liang2020learning} is a motion-forecasting model that represents road structure as a lane graph and combines graph convolutions with interactions between vehicles and the map to predict multiple possible future trajectories. Because its map encoder propagates information along lane connections, LaneGCN allows us to evaluate whether improved topology predictions benefit a downstream trajectory forecasting task. We therefore compare its prediction performance using TopoNet connectivity before and after enhancement by \ourmodel.

We adapt the official LaneGCN implementation to OpenLane-V2 Subset\_A keyframes and their associated Argoverse2 sensor tracks, predicting six trajectory modes over 3\,s from 2\,s of history at 10\,Hz. Keyframes without sufficient future observations are excluded. Prediction targets include the ego vehicle and other vehicles with complete trajectories, at least 2\,m of future motion, and current positions within the map region ($|x|<50$\,m, $|y|<25$\,m). Vehicles observed at the current step provide context. The training set contains 89,095 samples from 15,483 keyframes, and the validation set contains 19,073 samples from 3,306 keyframes.

To assess the effect of connectivity enhancement, both settings (before and after \ourmodel) use the same TopoNet lane detections and geometry, differing only in inter-lane successor/predecessor connections. We train a separate LaneGCN predictor for each setting and evaluate it on the corresponding validation maps using the same target samples. For each sample, we select the predicted mode with the lowest endpoint error and report its endpoint error (minFDE$_6$) and average displacement error (minADE$_6$). We also report MR$_6$, the fraction of samples with minFDE$_6>2$\,m. Table~\ref{tab:lanegcn_all} in the main paper reports the aggregate comparison, while Table~\ref{tab:lanegcn_subsets} separates ego and agent targets.

\begin{table}[htbp]
\centering
\small
\caption{LaneGCN performance for ego and agent targets using TopoNet connectivity before and after \ourmodel. Both conditions use the same detected lane geometry and paired validation samples. Lower values are better.}
\label{tab:lanegcn_subsets}
\begin{tabular}{lrrrrr}
\toprule
& & \multicolumn{2}{c}{minFDE$_6$ (m) $\downarrow$} & \multicolumn{2}{c}{MR$_6$ (\%) $\downarrow$} \\
\cmidrule(lr){3-4}\cmidrule(lr){5-6}
Subset & Samples & Before \ourmodel & After \ourmodel & Before \ourmodel & After \ourmodel \\
\midrule
Ego & 3,306 & 0.853 & \best{0.802} & 8.44 & \best{6.11} \\
Agents & 15,767 & 0.618 & \best{0.587} & 4.38 & \best{3.06} \\
\bottomrule
\end{tabular}
\end{table}

As shown in Table~\ref{tab:lanegcn_all}, enhancing TopoNet connectivity reduces minADE$_6$ from 0.271 to 0.269\,m, minFDE$_6$ from 0.659 to 0.624\,m, and MR$_6$ from 5.09\% to 3.59\%. Table~\ref{tab:lanegcn_subsets} shows that the reductions in endpoint error and miss rate hold for both ego and agent targets. For the ego vehicle, minFDE$_6$ decreases from 0.853 to 0.802\,m and MR$_6$ from 8.44\% to 6.11\%; for other agents, they decrease from 0.618 to 0.587\,m and from 4.38\% to 3.06\%, respectively. These results demonstrate that enhanced lane connectivity improves LaneGCN's best-of-six forecasting accuracy and reduces its miss rate under this evaluation protocol.

\subsection{Efficiency Analysis of Inference}
\label{appendix:latency}

To assess the computational cost of \ourmodel, we profile its inference pipeline on the full OpenLane-V2 Subset\_A validation split (4,806 frames), using TopoNet as the upstream topology model. The measurements cover traffic-element feature extraction, GNN message passing, and score prediction, excluding upstream TopoNet inference. We use a single NVIDIA H200 GPU with FP32 precision and a batch size of 1, and average measurements over validation frames. TopoNet produces a fixed number of queries, yielding 200 lane nodes and 100 traffic-element (TE) nodes per frame. The resulting graphs contain 39,800 candidate lane--lane edges and 20,000 candidate lane--TE edges, giving the graph module a consistent workload across scenes.

\begin{table}[h]
\centering
\caption{Per-component inference latency and peak allocated GPU memory of \ourmodel on OpenLane-V2 Subset\_A (200 lane nodes and 100 TE nodes per frame). All the numbers are averaged across evaluated frames.}
\label{tab:latency_breakdown}
\label{tab:memory_breakdown}
\begin{tabular}{lrr}
\toprule
Component & \textbf{Latency (s)} & \textbf{Peak GPU Memory (MB)} \\
\midrule
Feature Extraction & 1.4829 & 2,458 \\
GNN Message Passing & 0.0025 & 1,509 \\
Score Prediction & 0.0007 & 1,225 \\
\midrule
\textbf{Total} & \best{1.4861} & \best{2,458} \\
\bottomrule
\end{tabular}
\end{table}

\textbf{Inference Latency.}
As shown in Table~\ref{tab:latency_breakdown}, traffic-element feature extraction accounts for 1.4829\,s of the total 1.4861\,s per frame. The DINOv2-ViT-L/14 forward pass takes approximately 1.4829\,s. GNN message passing and score prediction, including fusion with TopoNet's topology scores, require only 2.5\,ms and 0.7\,ms, respectively, totaling 3.2\,ms for nearly 60,000 candidate connections. The dominant computational cost therefore comes from the visual feature extractor. Table~\ref{tab:ablation_backbone} shows that replacing DINOv2-ViT-L with ResNet-50 yields comparable topology performance: TOP$_{ll}$/TOP$_{lt}$ changes from 21.84/25.80 to 21.80/25.50, a reduction of only 0.04 and 0.30, respectively. Given this small performance gap, reusing ResNet-50 features already computed by TopoNet could retain most of the enhancement gains while avoiding a separate DINOv2 forward pass, which alone accounts for approximately 99.8\% of the measured latency.

\textbf{Memory Usage.}
Table~\ref{tab:latency_breakdown} reports peak GPU memory allocated by PyTorch during each inference stage, including resident model weights and intermediate tensors. The average per-frame peak of the enhancement pipeline is 2,458\,MB, with feature extraction accounting for the largest memory demand. DINOv2-ViT-L/14 contains approximately 300 million parameters, whereas the GNN contains only 0.29 million. During GNN message passing and score prediction, peak allocated memory reaches 1,509\,MB and 1,225\,MB, respectively. These values include allocations already resident in GPU memory, together with the intermediate tensors required by each stage, and therefore do not represent the memory used exclusively by the corresponding component. In particular, message passing requires edge-level attention coefficients and messages in addition to the GNN parameters. Because the stages share resident allocations and their temporary tensors need not coexist, the overall peak reflects the maximum concurrent memory usage rather than the sum of the stage peaks.

\section{Limitation and Future Works}
We have to admit that, relying on camera images to refine lane–traffic connectivity introduces a limitation, as it restricts the applicability of \ourmodel to settings where camera images are available. This is a deliberate design choice because reasoning about lane–traffic connectivity requires semantic understanding of traffic elements. For example, the enhancement model must distinguish whether a traffic element corresponds to a "turn left" sign or a "go straight" sign in order to correctly refine lane–traffic connections. An alternative to camera images would be to use the semantic labels of traffic elements directly. However, such labels are typically unavailable at inference time and require an additional perception or annotation pipeline to obtain. In contrast, camera images are already available in most autonomous driving systems and naturally provide the semantic cues needed for refinement without relying on ground-truth semantic annotations. Moreover, this modality requirement only affects lane–traffic connectivity refinement. When camera images are unavailable, \ourmodel can still be applied to lane–lane connectivity, which does not depend on image features. Future work could explore alternative semantic representations or additional modalities to further broaden the applicability of lane–traffic refinement.

\section{Additional Qualitative Results}
\label{app:more_qualitative_results}
\begin{figure*}[htbp]
\centering
\captionsetup[subfigure]{labelformat=empty}
\begin{subfigure}[t]{0.31\textwidth}
  \centering\includegraphics[width=\linewidth]{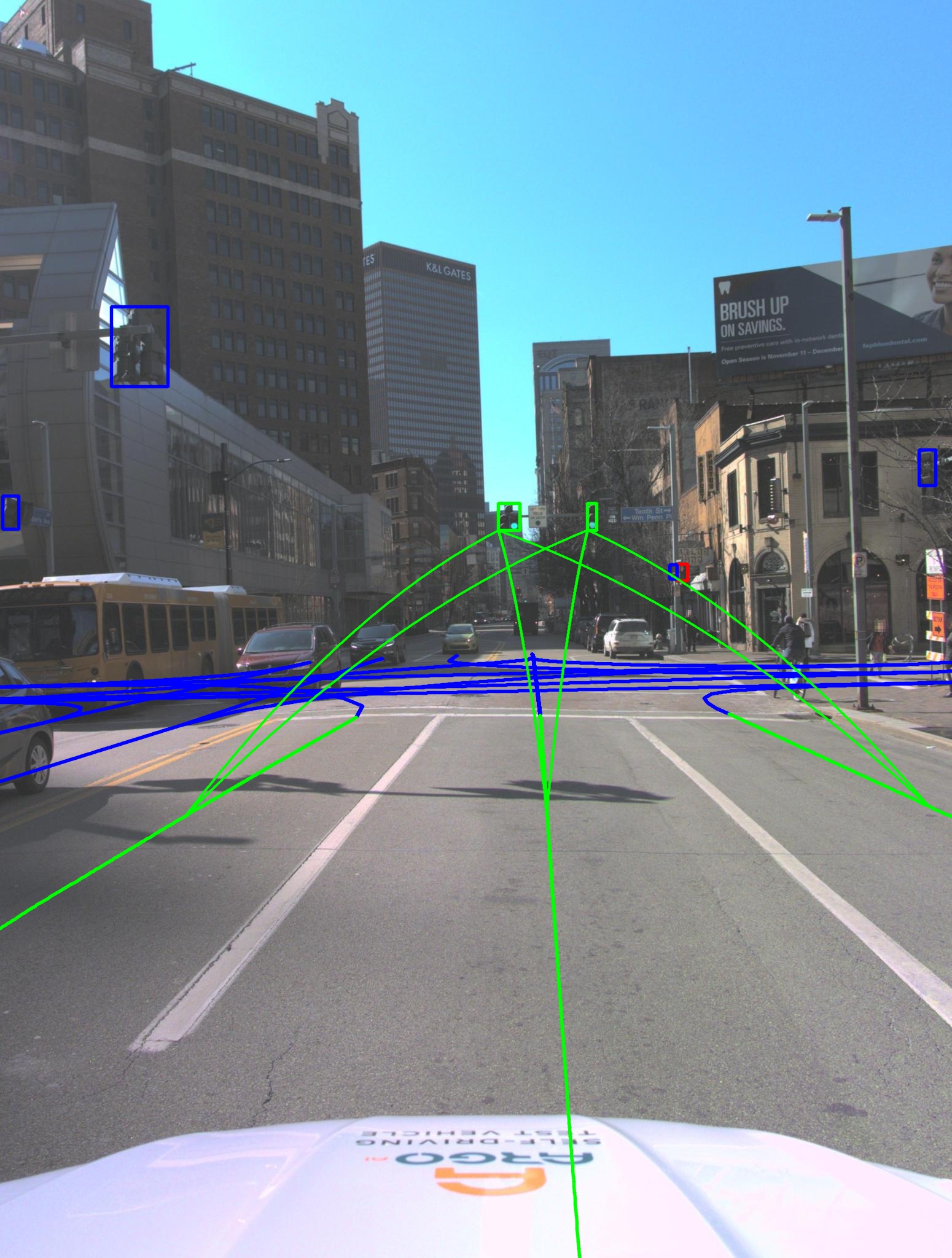}
  \caption{Ground-truth}
\end{subfigure}\hfill
\begin{subfigure}[t]{0.31\textwidth}
  \centering\includegraphics[width=\linewidth]{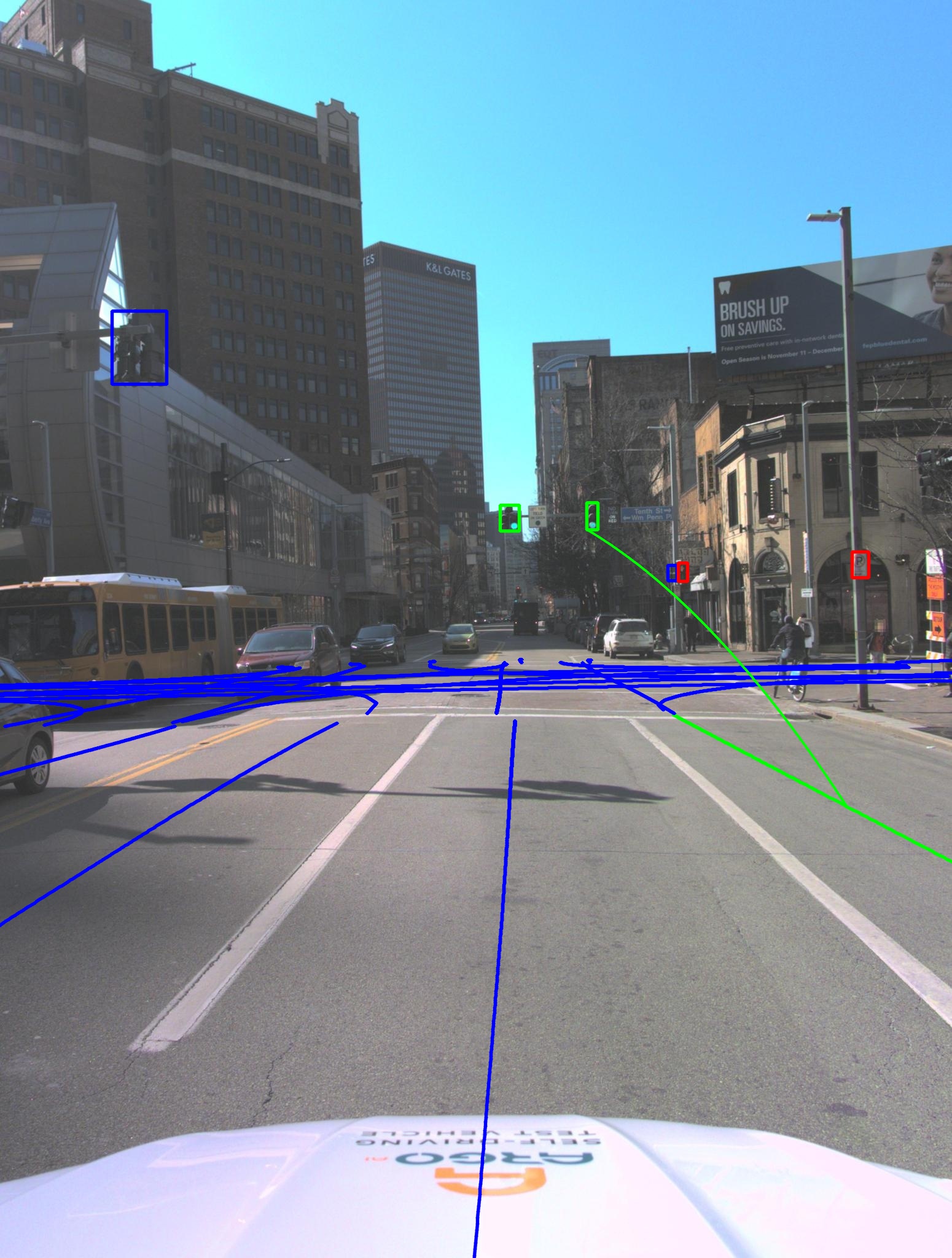}
  \caption{SMART}
\end{subfigure}\hfill
\begin{subfigure}[t]{0.31\textwidth}
  \centering\includegraphics[width=\linewidth]{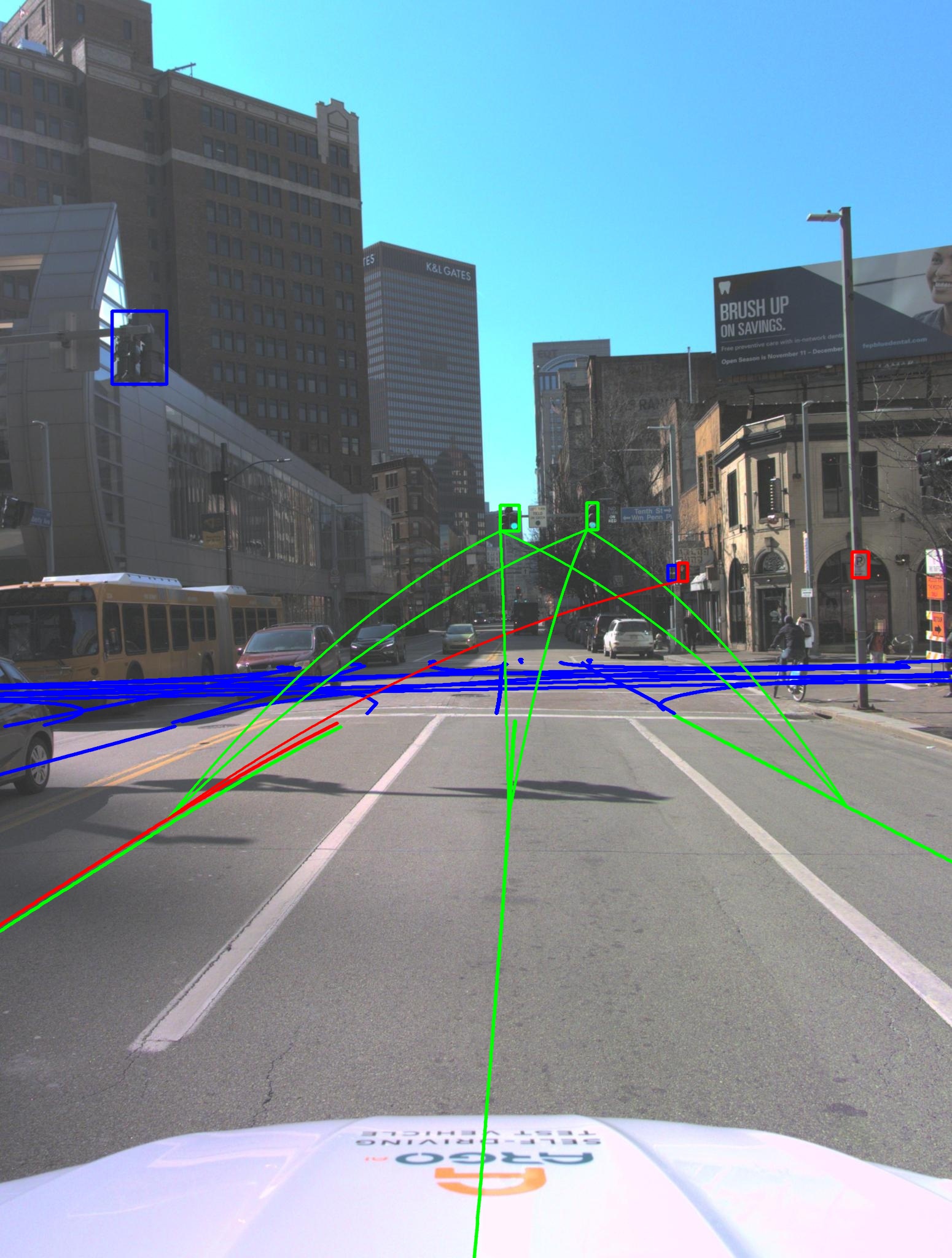}
  \caption{SMART + \ourmodel}
\end{subfigure}

\vspace{6pt}

\begin{subfigure}[t]{0.235\textwidth}
  \centering\includegraphics[width=\linewidth]{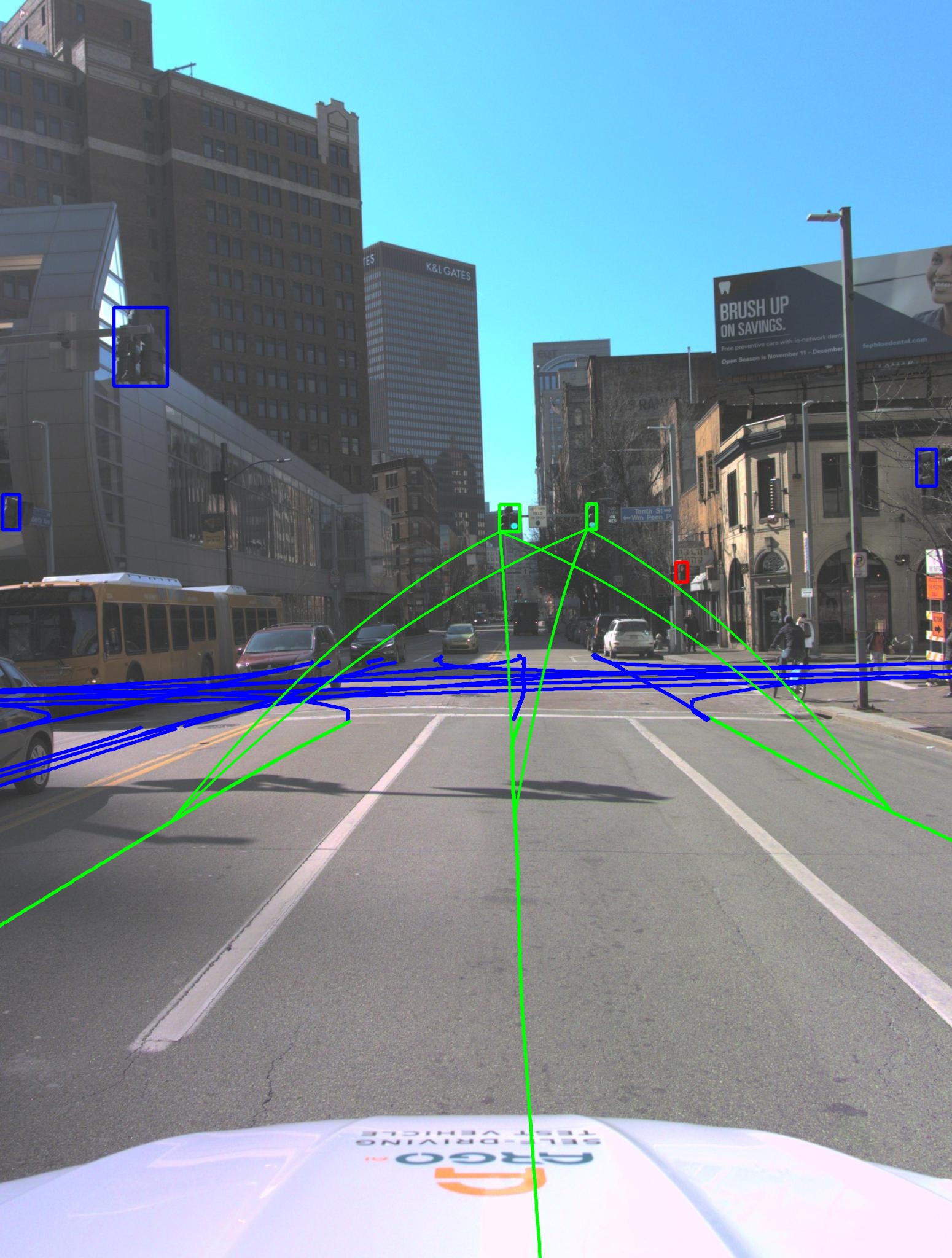}
  \caption{Topo2D }
\end{subfigure}\hfill
\begin{subfigure}[t]{0.235\textwidth}
  \centering\includegraphics[width=\linewidth]{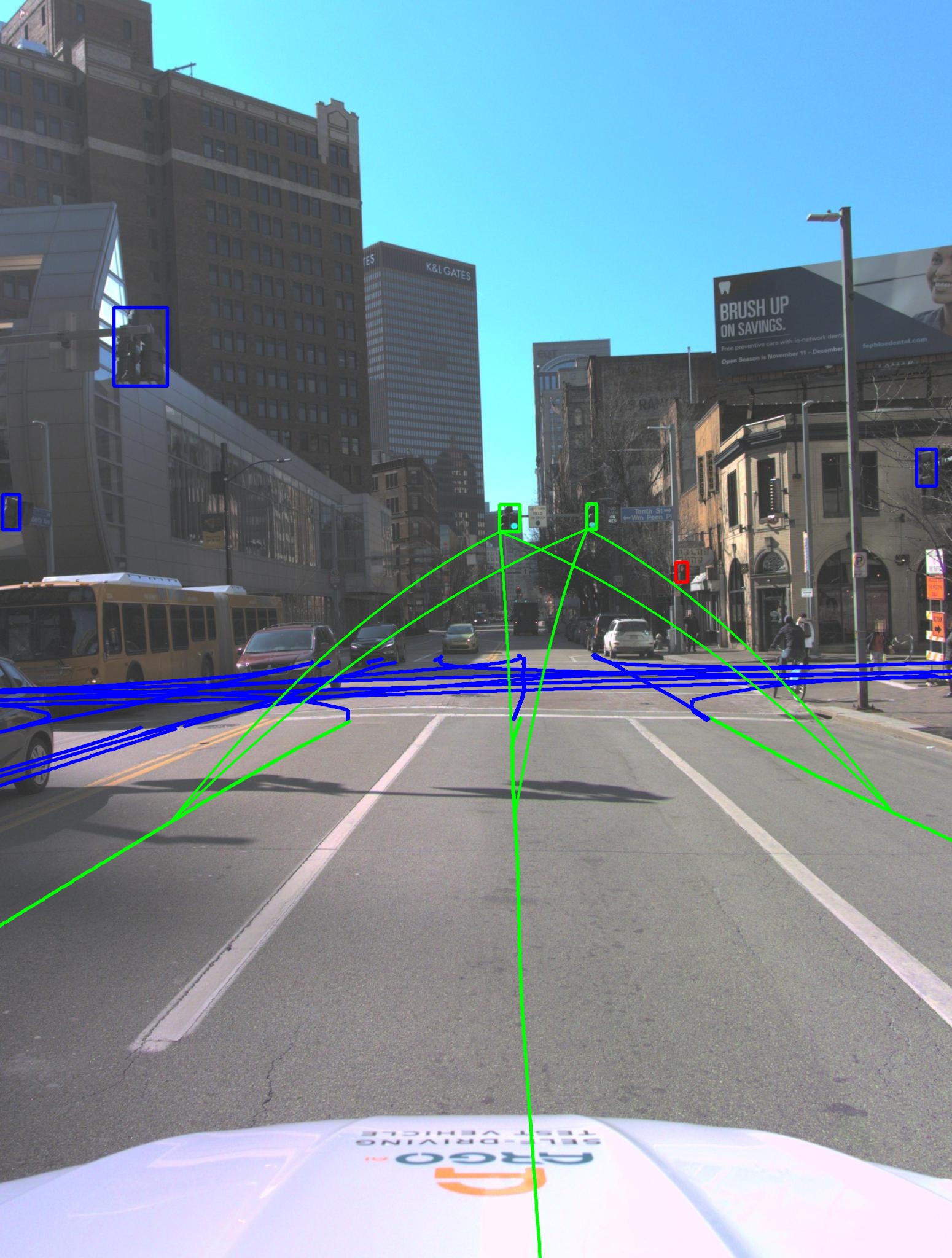}
  \caption{\centering Topo2D + \ourmodel}
\end{subfigure}\hfill
\begin{subfigure}[t]{0.235\textwidth}
  \centering\includegraphics[width=\linewidth]{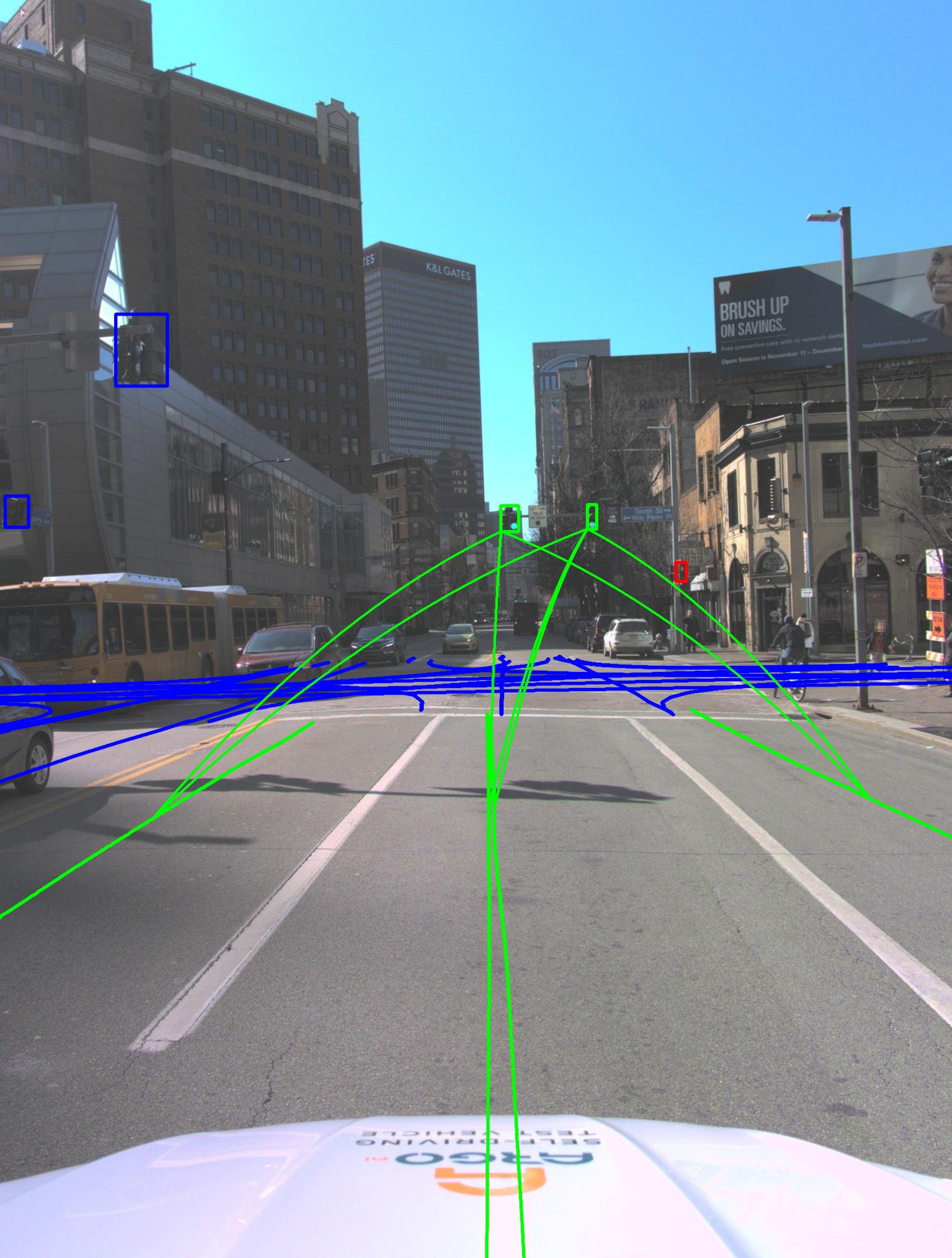}
  \caption{TopoLogic}
\end{subfigure}\hfill
\begin{subfigure}[t]{0.235\textwidth}
  \centering\includegraphics[width=\linewidth]{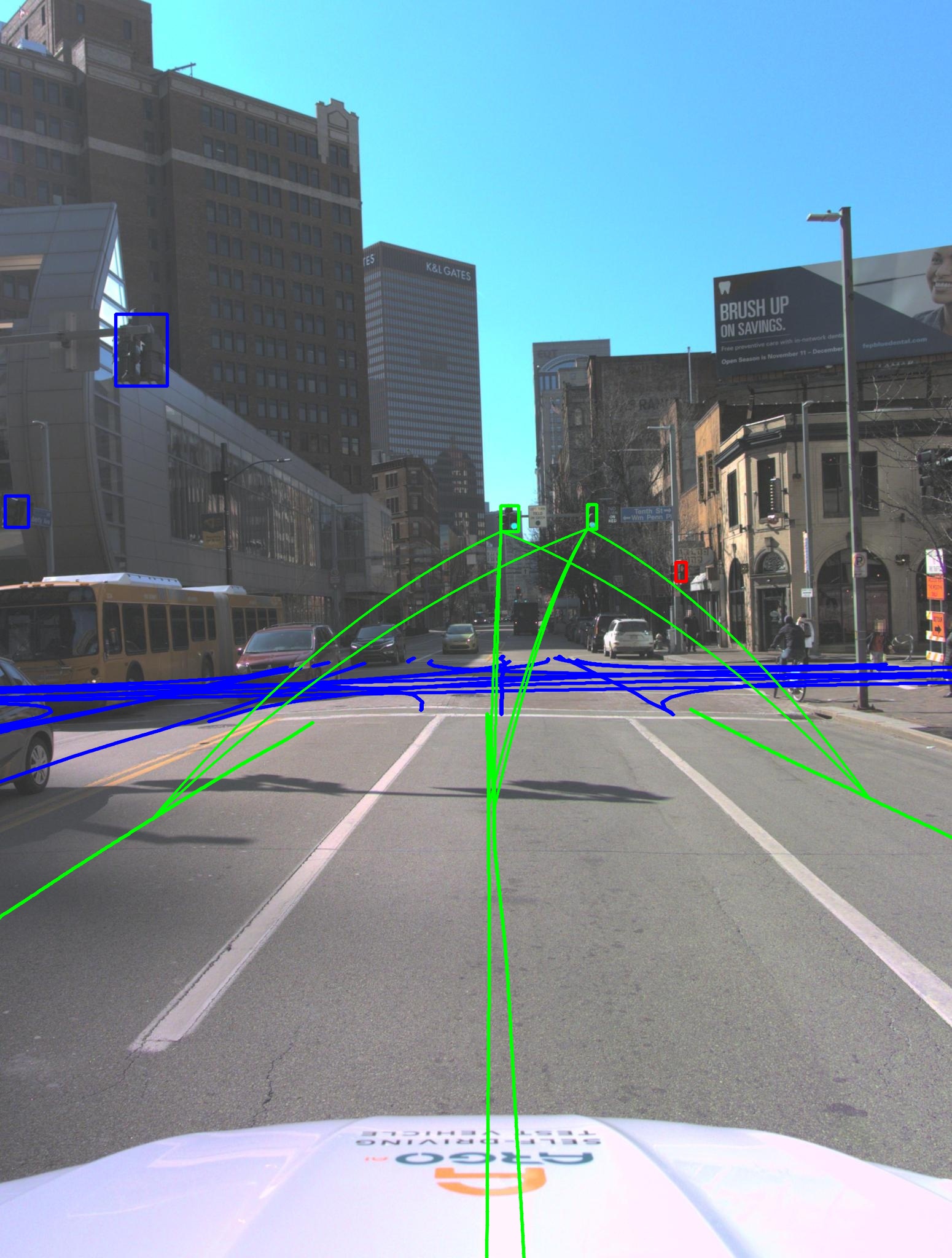}
  \caption{\centering TopoLogic + \ourmodel}
\end{subfigure}

\vspace{6pt}

\begin{subfigure}[t]{0.235\textwidth}
  \centering\includegraphics[width=\linewidth]{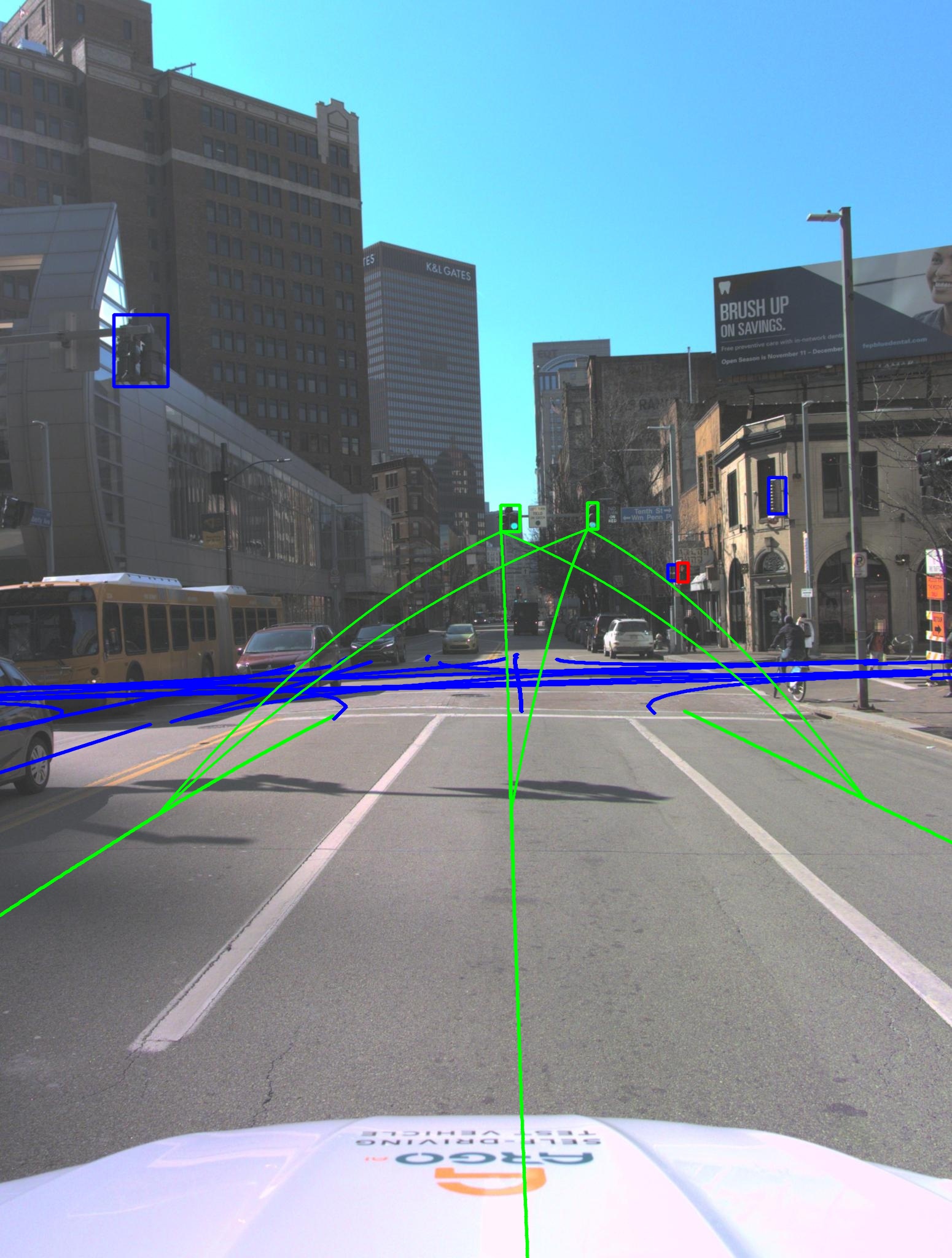}
  \caption{TopoMLP}
\end{subfigure}\hfill
\begin{subfigure}[t]{0.235\textwidth}
  \centering\includegraphics[width=\linewidth]{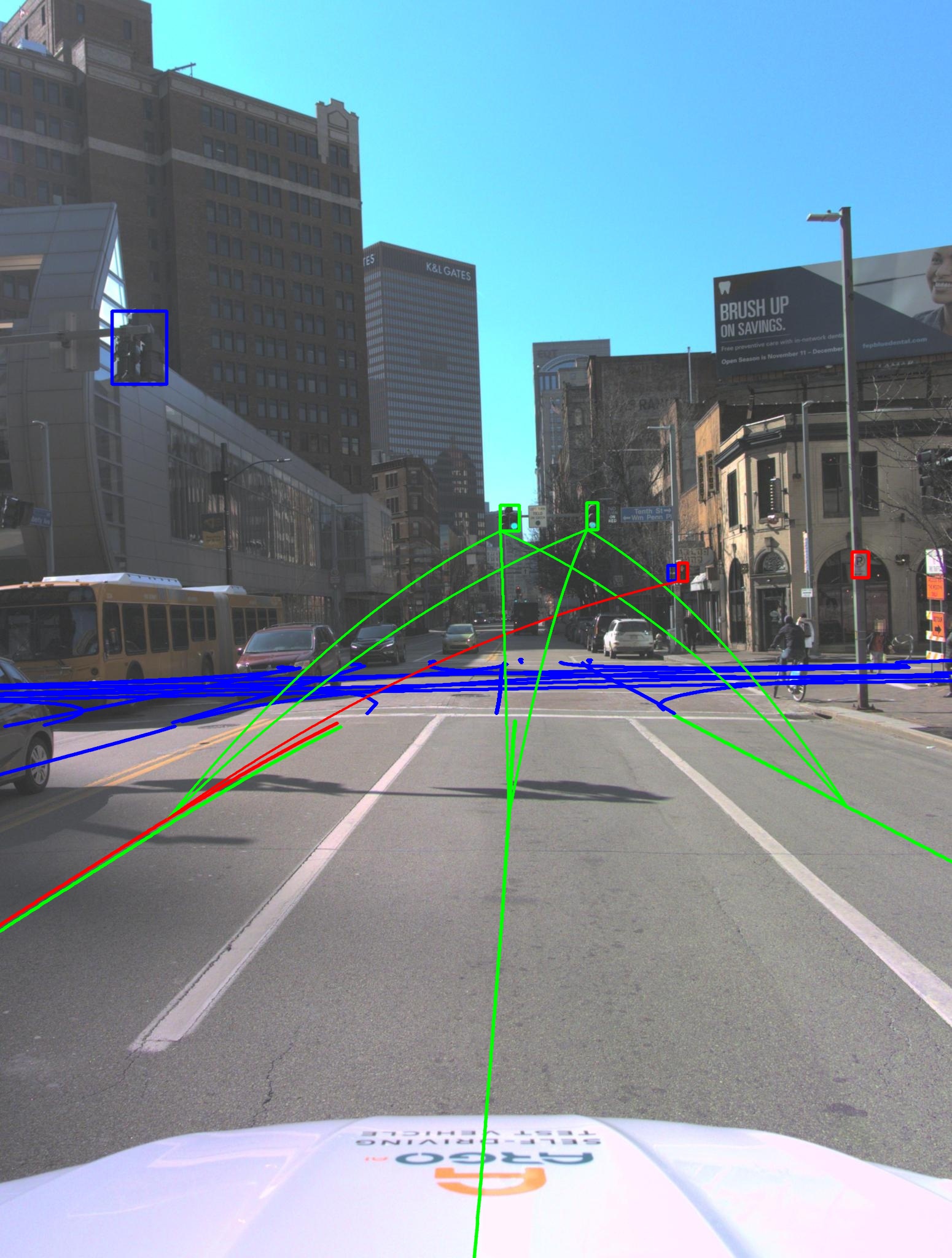}
  \caption{\centering TopoMLP + \ourmodel}
\end{subfigure}\hfill
\begin{subfigure}[t]{0.235\textwidth}
  \centering\includegraphics[width=\linewidth]{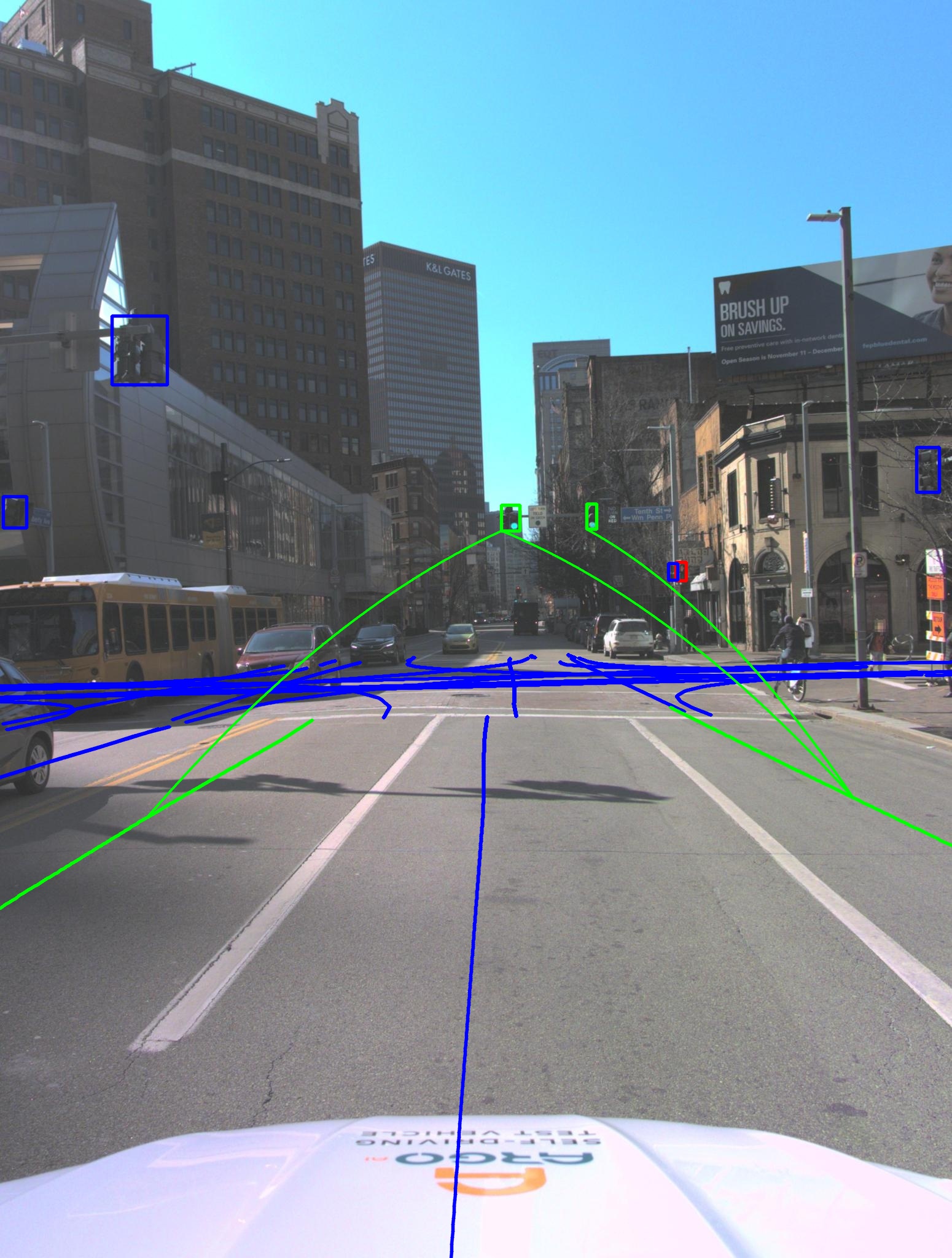}
  \caption{TopoNet}
\end{subfigure}\hfill
\begin{subfigure}[t]{0.235\textwidth}
  \centering\includegraphics[width=\linewidth]{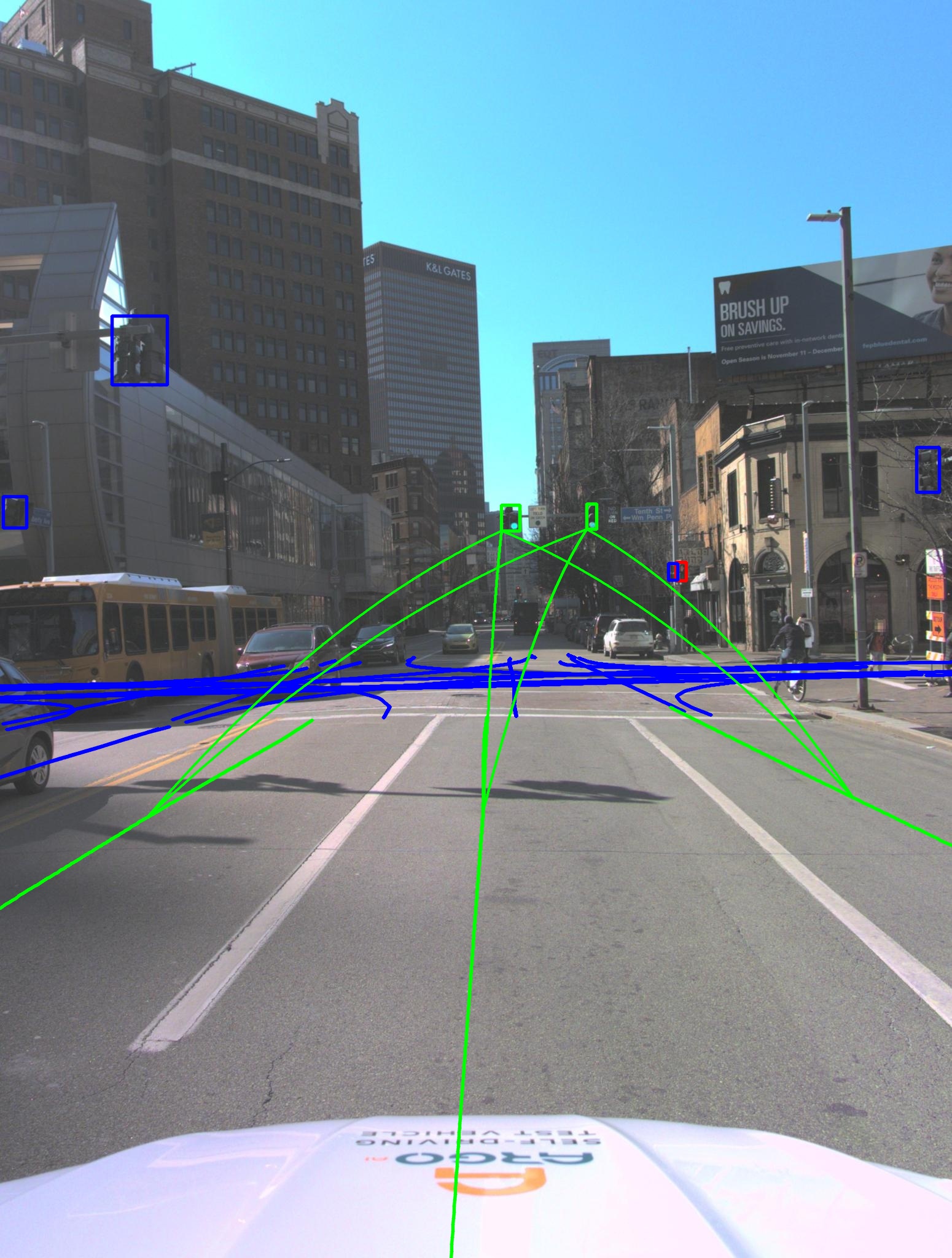}
  \caption{\centering TopoNet + \ourmodel}
\end{subfigure}

\caption{Qualitative results for scene 10023: predictions before and after refinement with \ourmodel across baselines.}
\label{fig:appx_qual_10023}
\end{figure*}

\begin{figure*}[htbp]
\centering
\captionsetup[subfigure]{labelformat=empty}
\begin{subfigure}[t]{0.31\textwidth}
  \centering\includegraphics[width=\linewidth]{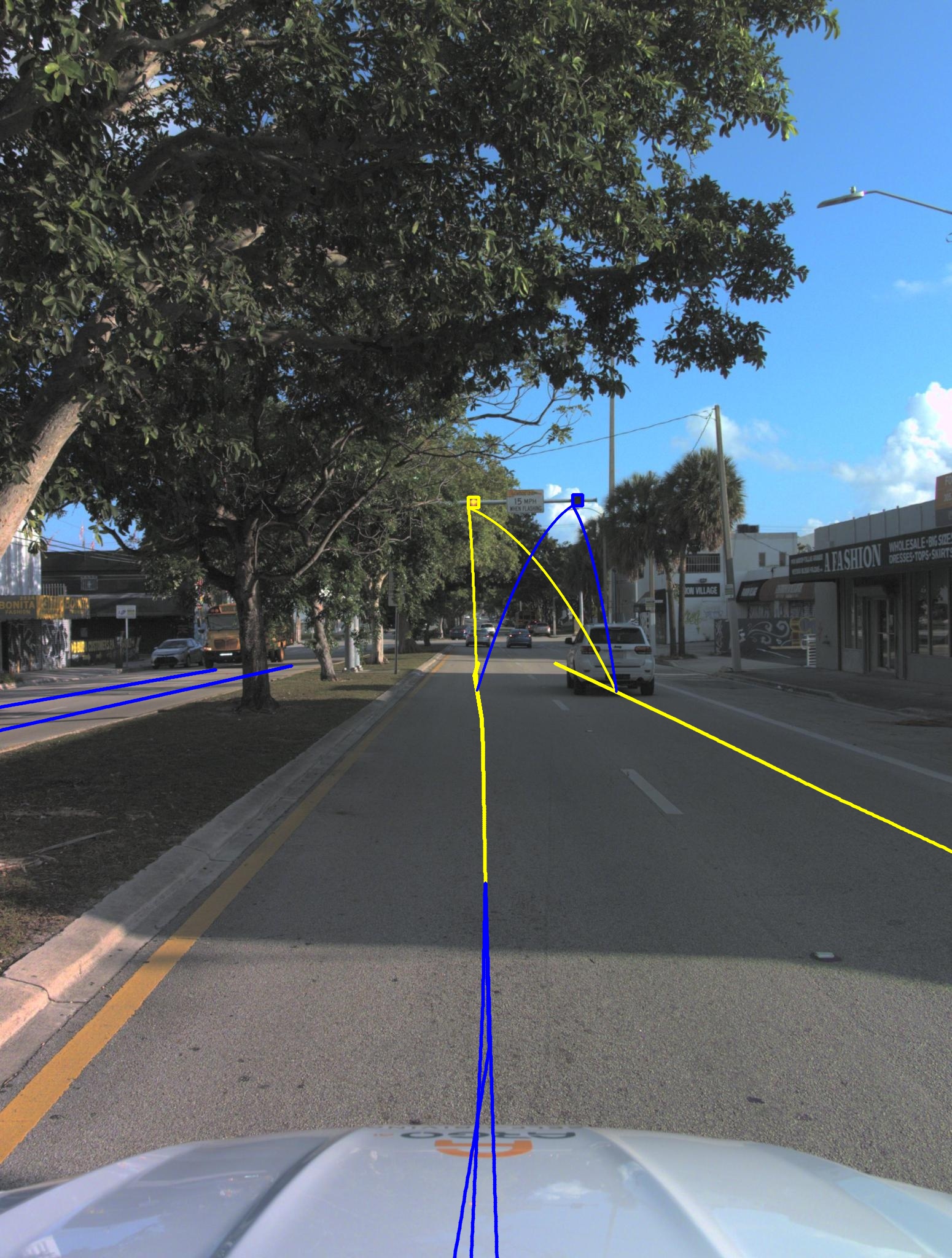}
  \caption{Ground-truth}
\end{subfigure}\hfill
\begin{subfigure}[t]{0.31\textwidth}
  \centering\includegraphics[width=\linewidth]{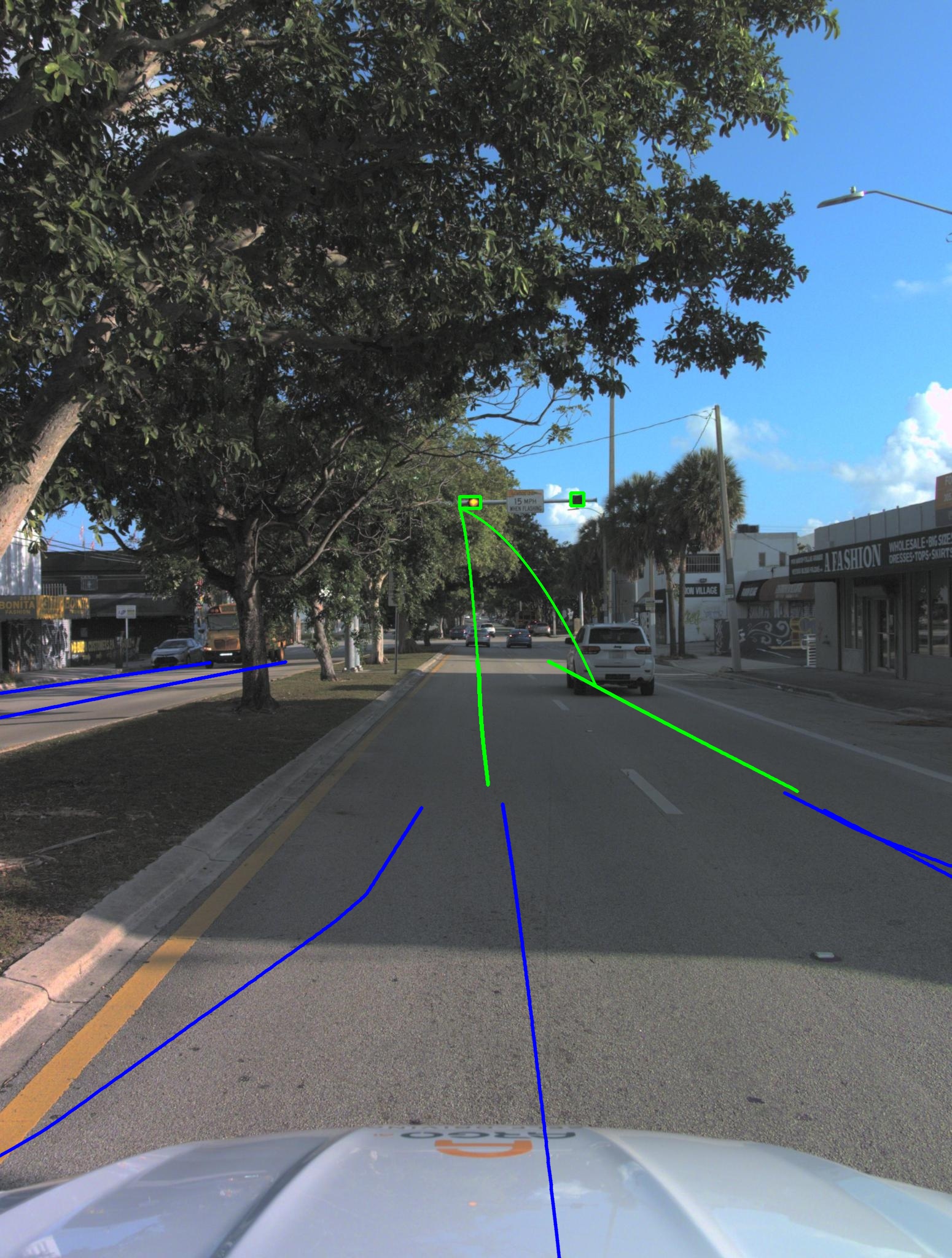}
  \caption{\centering SMART}
\end{subfigure}\hfill
\begin{subfigure}[t]{0.31\textwidth}
  \centering\includegraphics[width=\linewidth]{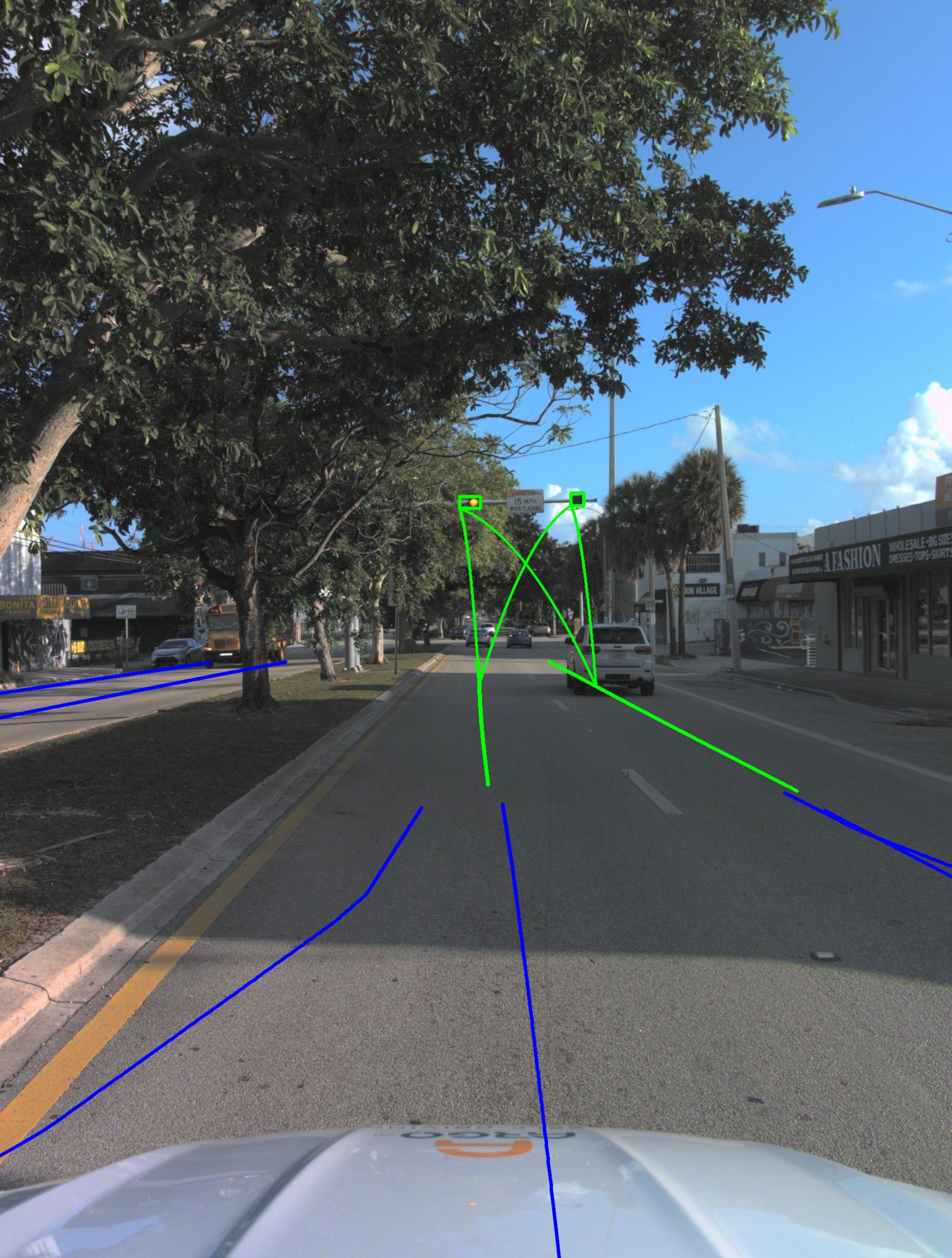}
  \caption{SMART + \ourmodel}
\end{subfigure}

\vspace{6pt}

\begin{subfigure}[t]{0.235\textwidth}
  \centering\includegraphics[width=\linewidth]{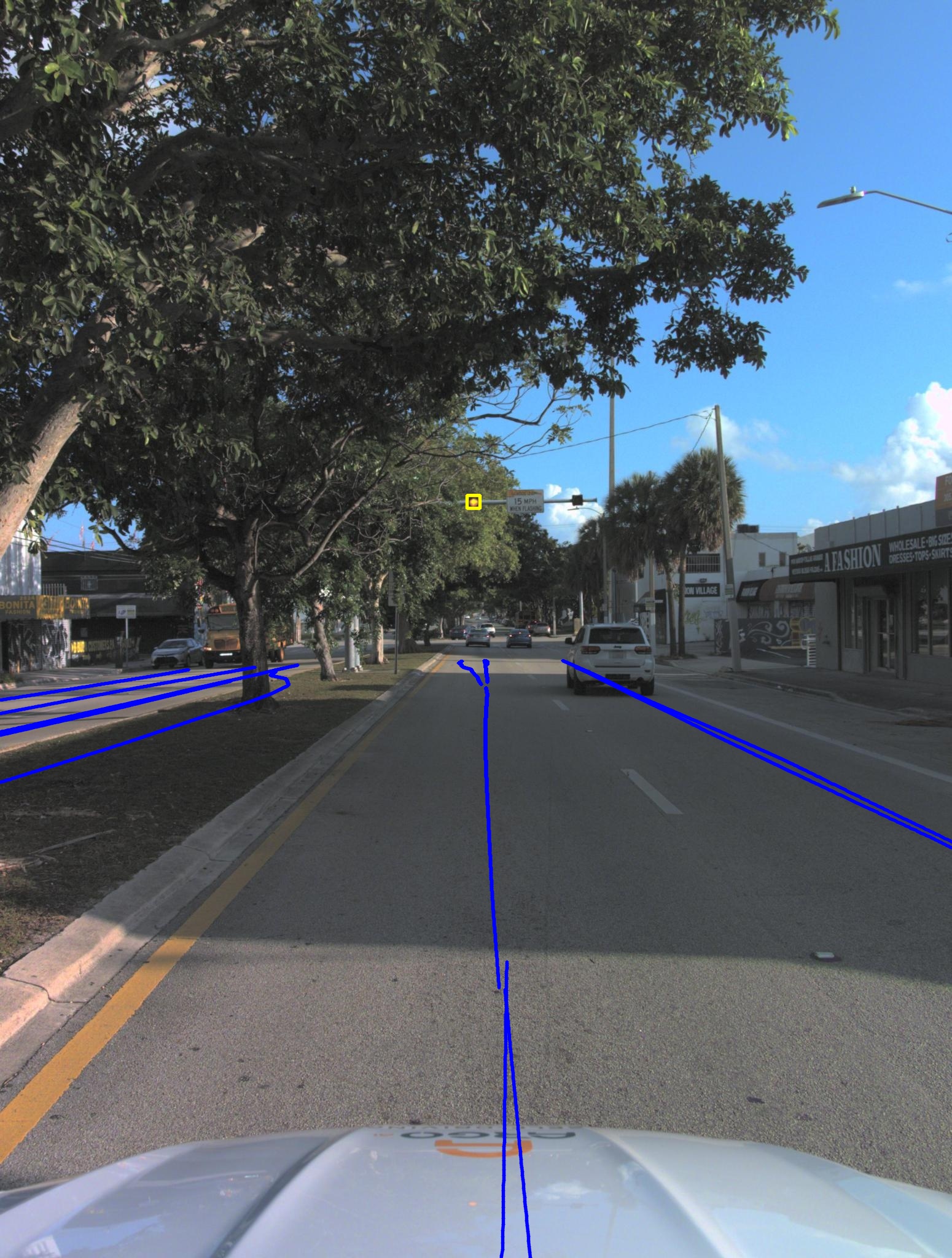}
  \caption{\centering Topo2D}
\end{subfigure}\hfill
\begin{subfigure}[t]{0.235\textwidth}
  \centering\includegraphics[width=\linewidth]{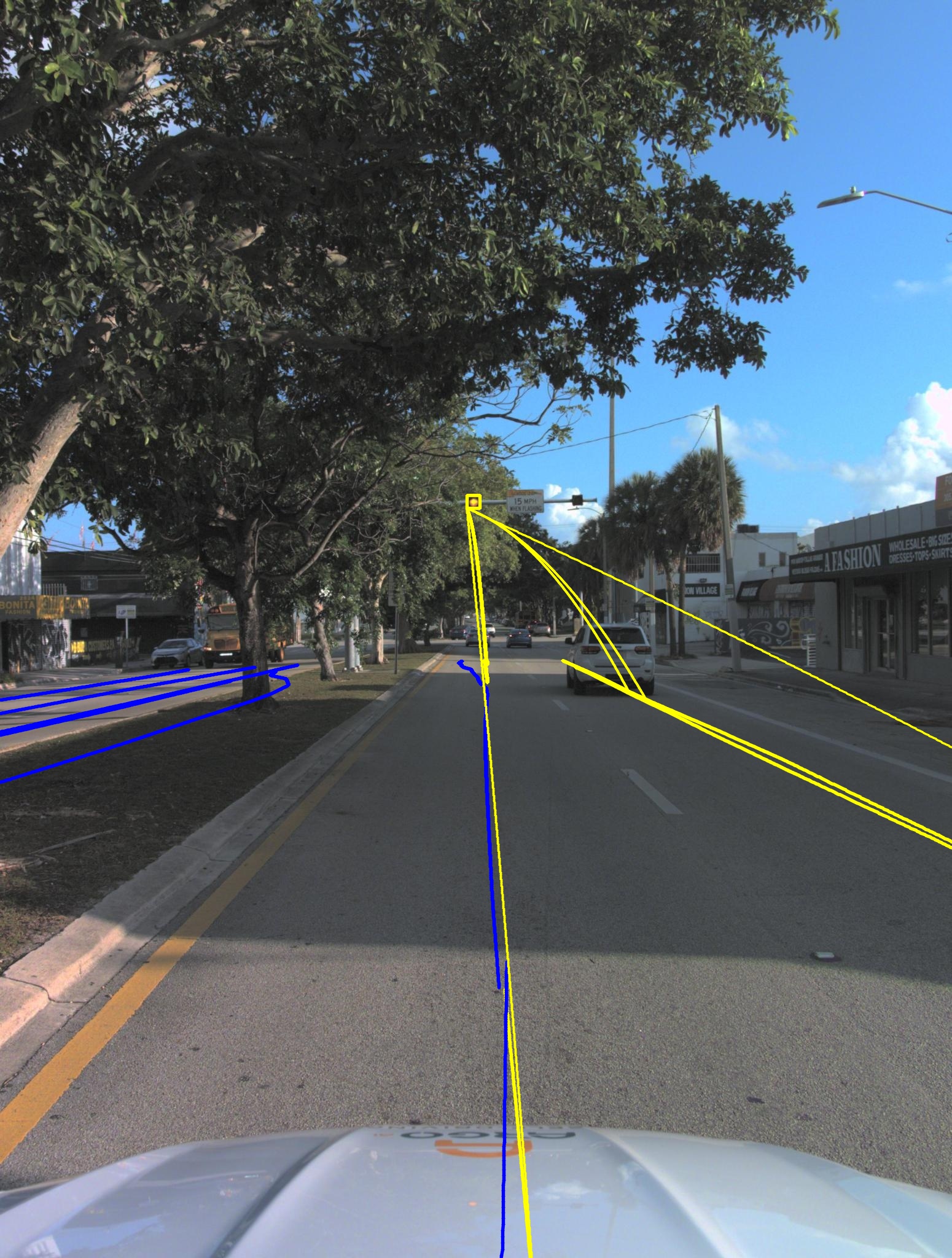}
  \caption{\centering Topo2D + \ourmodel}
\end{subfigure}\hfill
\begin{subfigure}[t]{0.235\textwidth}
  \centering\includegraphics[width=\linewidth]{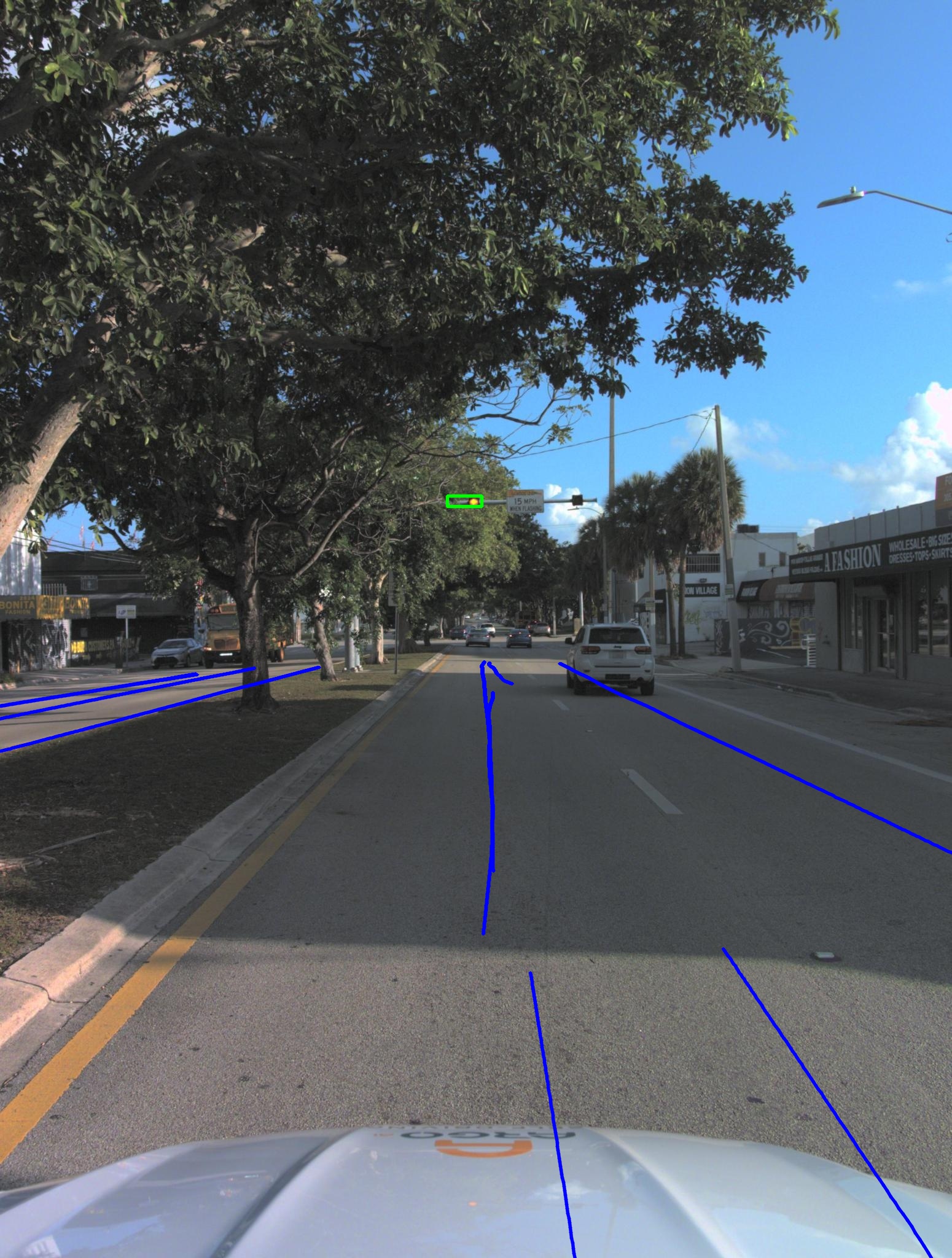}
  \caption{TopoLogic}
\end{subfigure}\hfill
\begin{subfigure}[t]{0.235\textwidth}
  \centering\includegraphics[width=\linewidth]{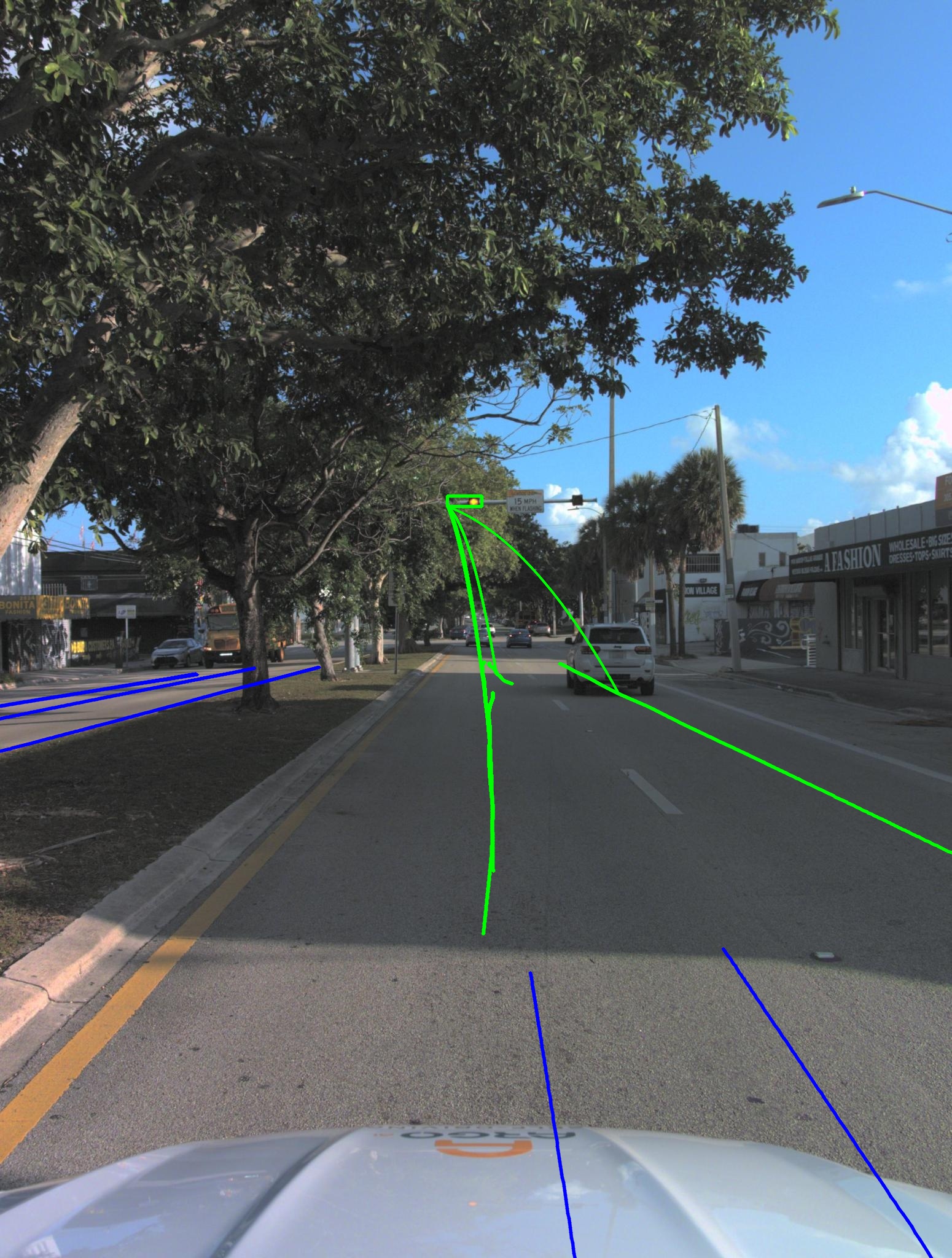}
  \caption{\centering TopoLogic + \ourmodel}
\end{subfigure}

\vspace{6pt}

\begin{subfigure}[t]{0.235\textwidth}
  \centering\includegraphics[width=\linewidth]{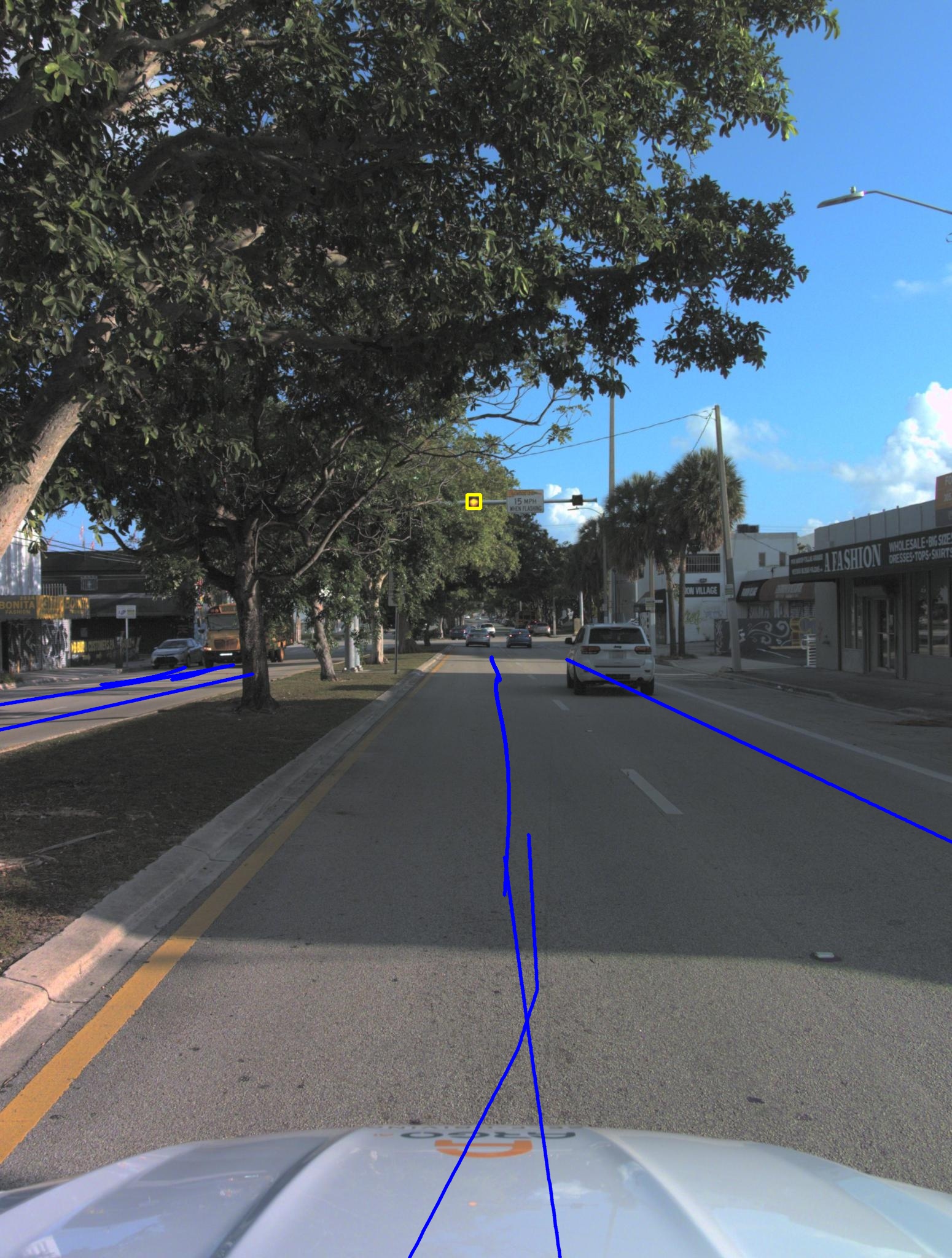}
  \caption{TopoMLP}
\end{subfigure}\hfill
\begin{subfigure}[t]{0.235\textwidth}
  \centering\includegraphics[width=\linewidth]{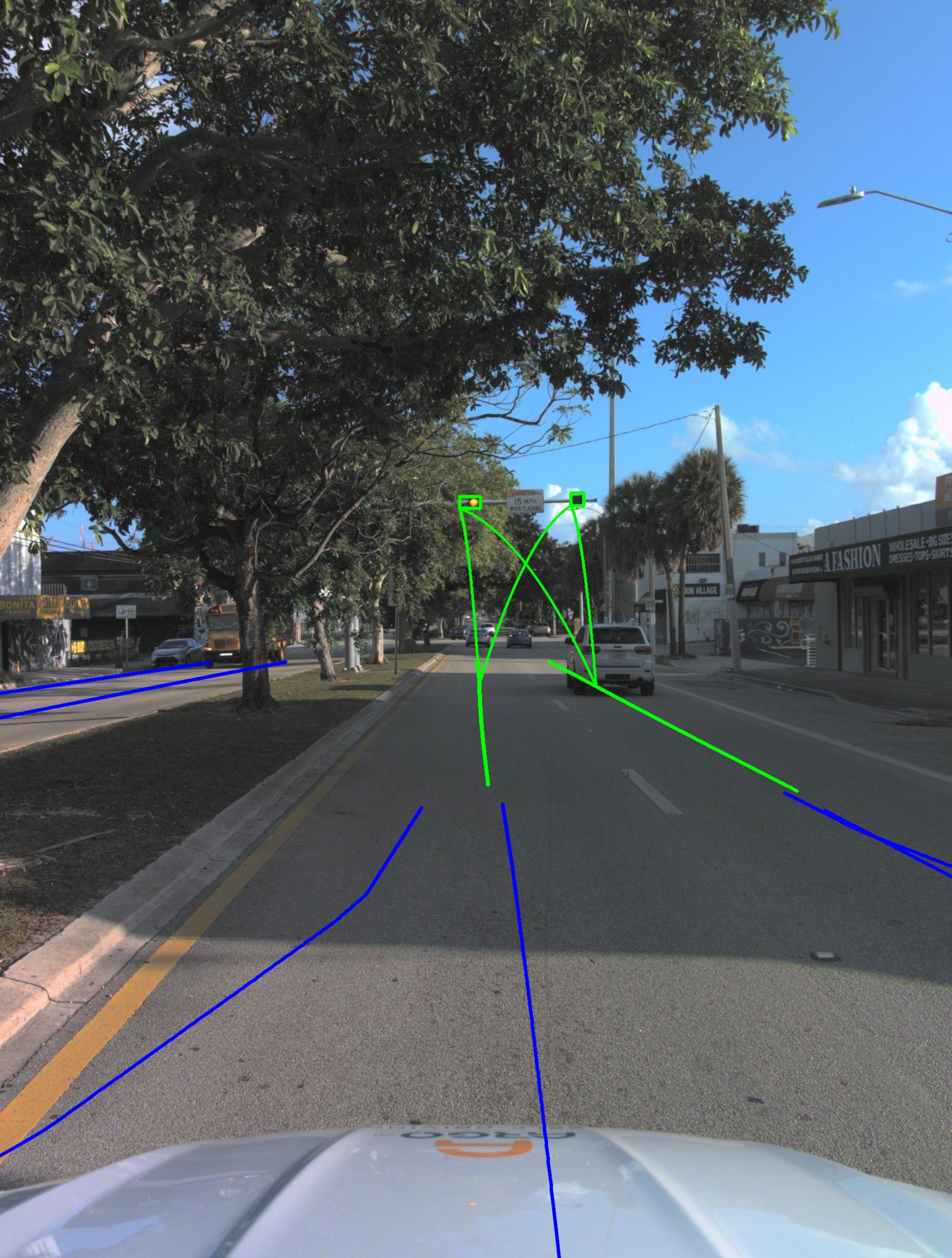}
  \caption{\centering TopoMLP + \ourmodel}
\end{subfigure}\hfill
\begin{subfigure}[t]{0.235\textwidth}
  \centering\includegraphics[width=\linewidth]{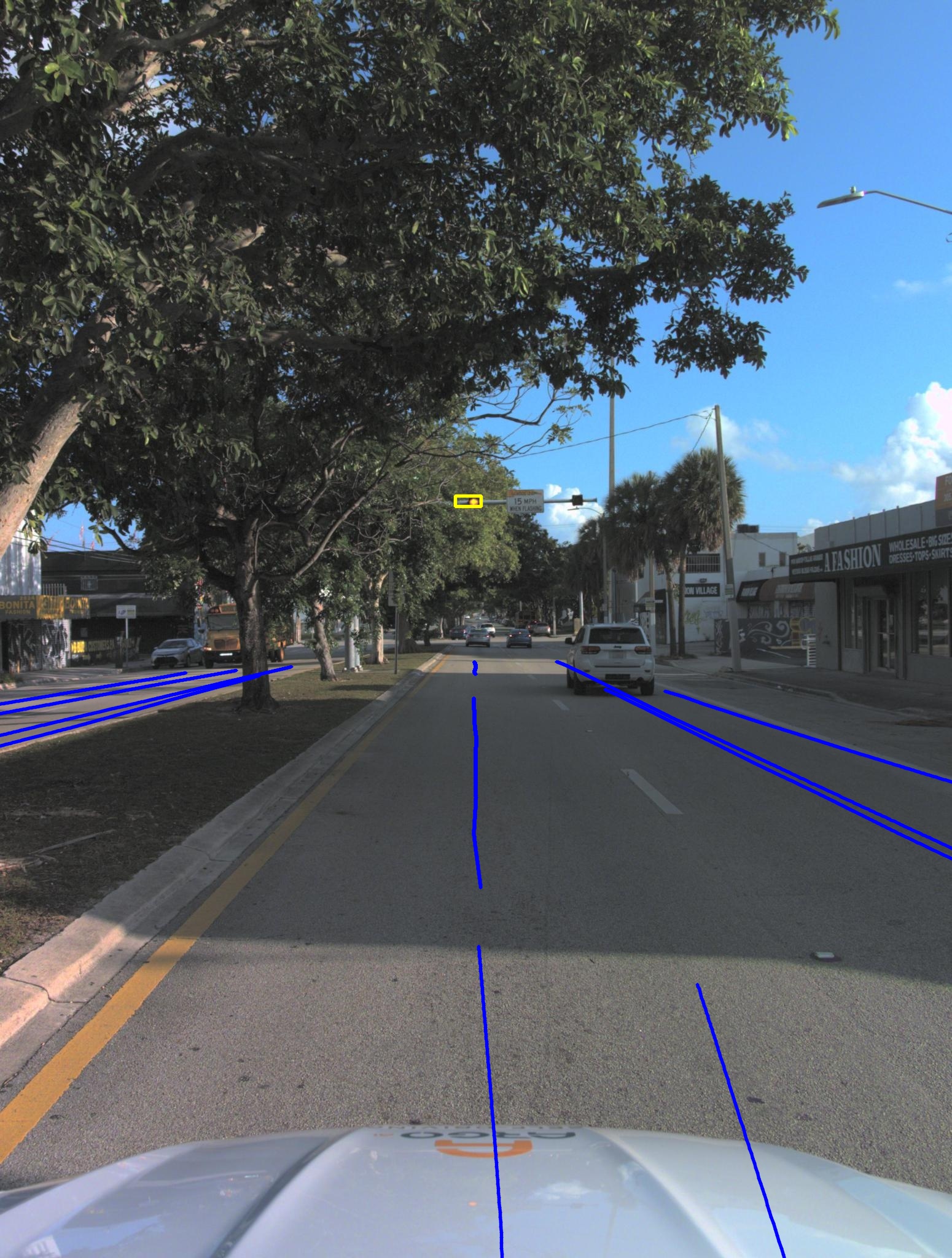}
  \caption{TopoNet}
\end{subfigure}\hfill
\begin{subfigure}[t]{0.235\textwidth}
  \centering\includegraphics[width=\linewidth]{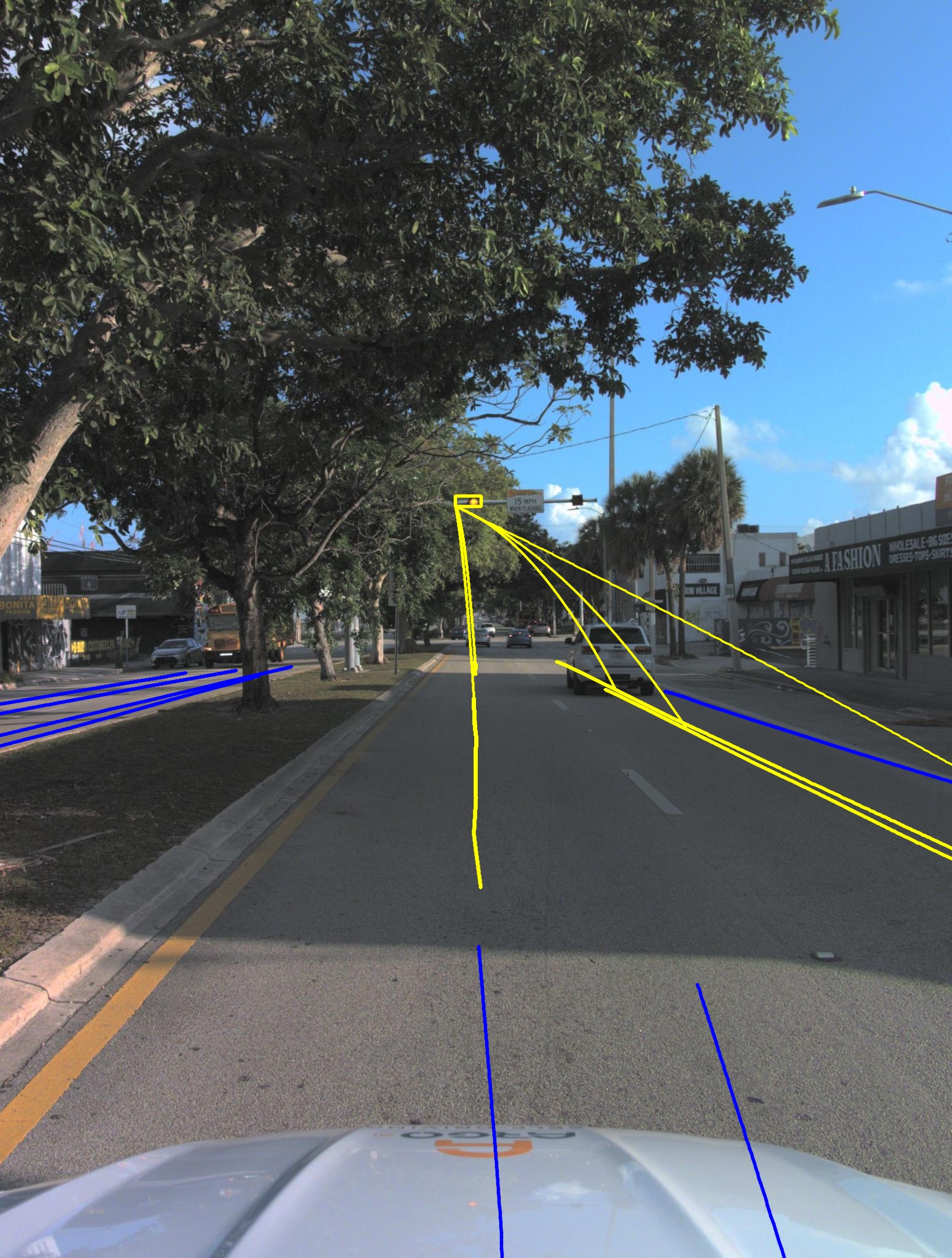}
  \caption{\centering TopoNet + \ourmodel}
\end{subfigure}

\caption{Qualitative results for scene 10013: predictions before and after refinement with \ourmodel across baselines (SMART, Topo2D, TopoLogic, TopoMLP, TopoNet). All panels correspond to the same ground-truth view (left).}
\label{fig:appx_qual_10013}
\end{figure*}

\begin{figure*}[htbp]
\centering
\captionsetup[subfigure]{labelformat=empty}
\begin{subfigure}[t]{0.31\textwidth}
  \centering\includegraphics[width=\linewidth]{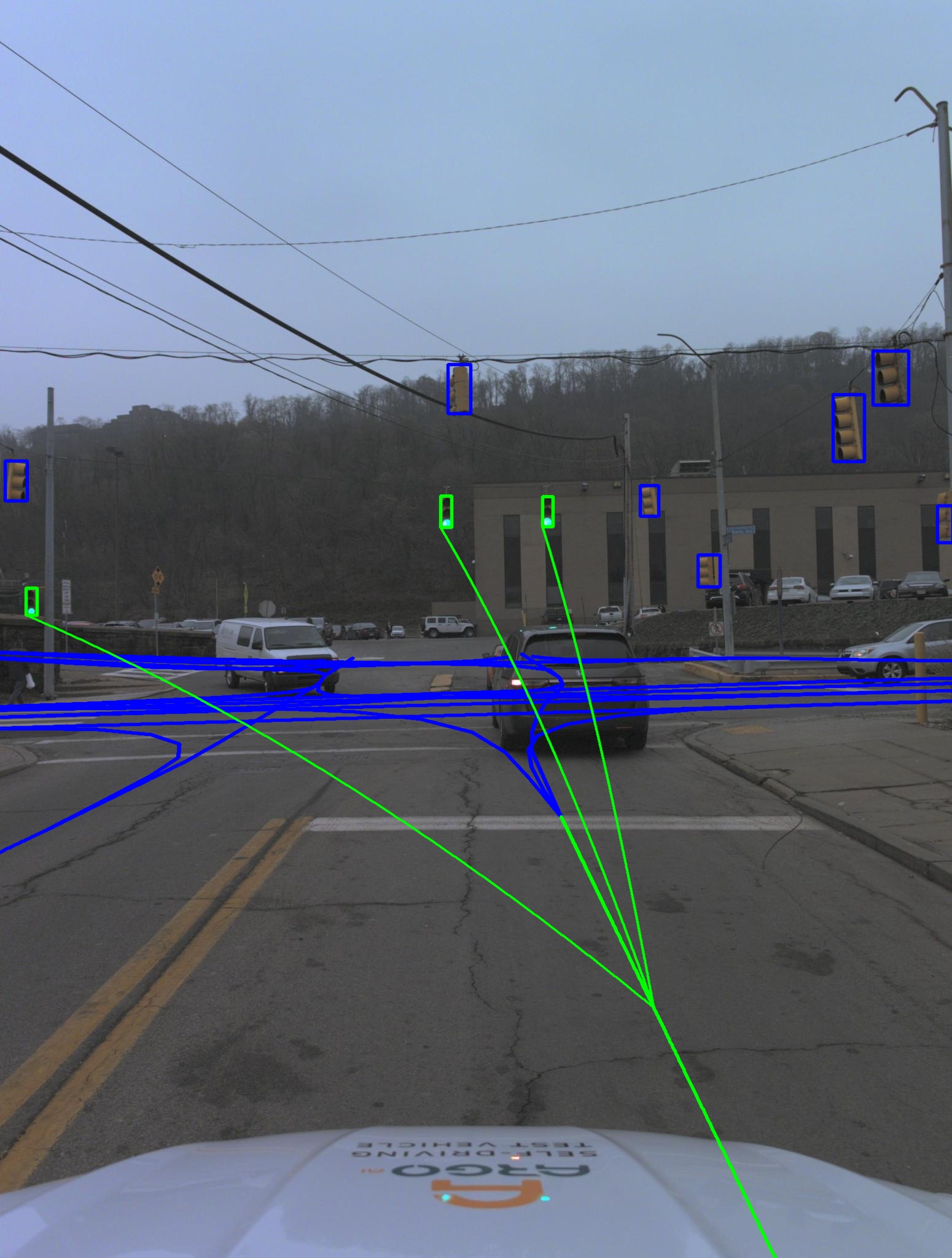}
  \caption{Ground-truth}
\end{subfigure}\hfill
\begin{subfigure}[t]{0.31\textwidth}
  \centering\includegraphics[width=\linewidth]{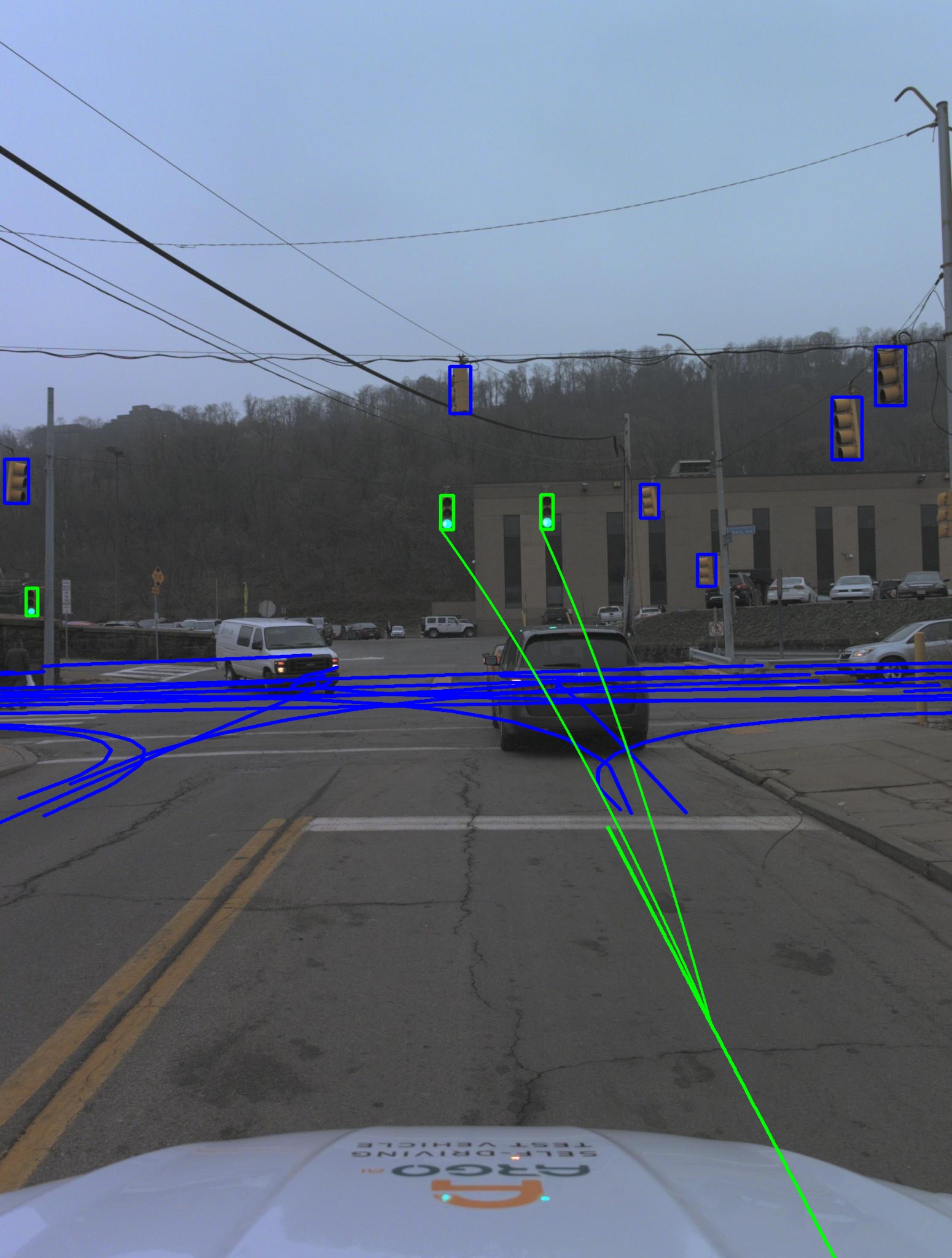}
  \caption{\centering SMART}
\end{subfigure}\hfill
\begin{subfigure}[t]{0.31\textwidth}
  \centering\includegraphics[width=\linewidth]{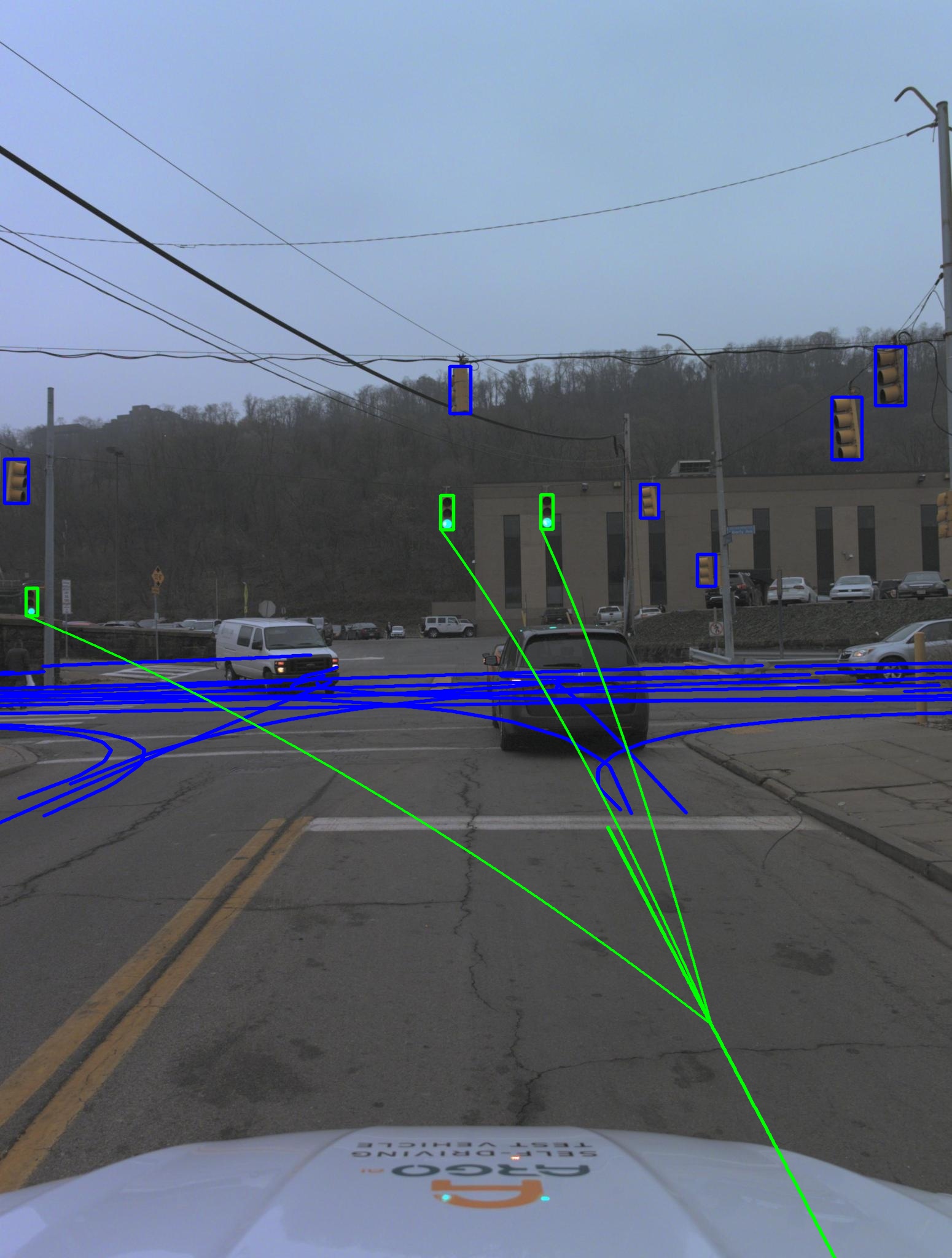}
  \caption{SMART + \ourmodel}
\end{subfigure}

\vspace{6pt}

\begin{subfigure}[t]{0.235\textwidth}
  \centering\includegraphics[width=\linewidth]{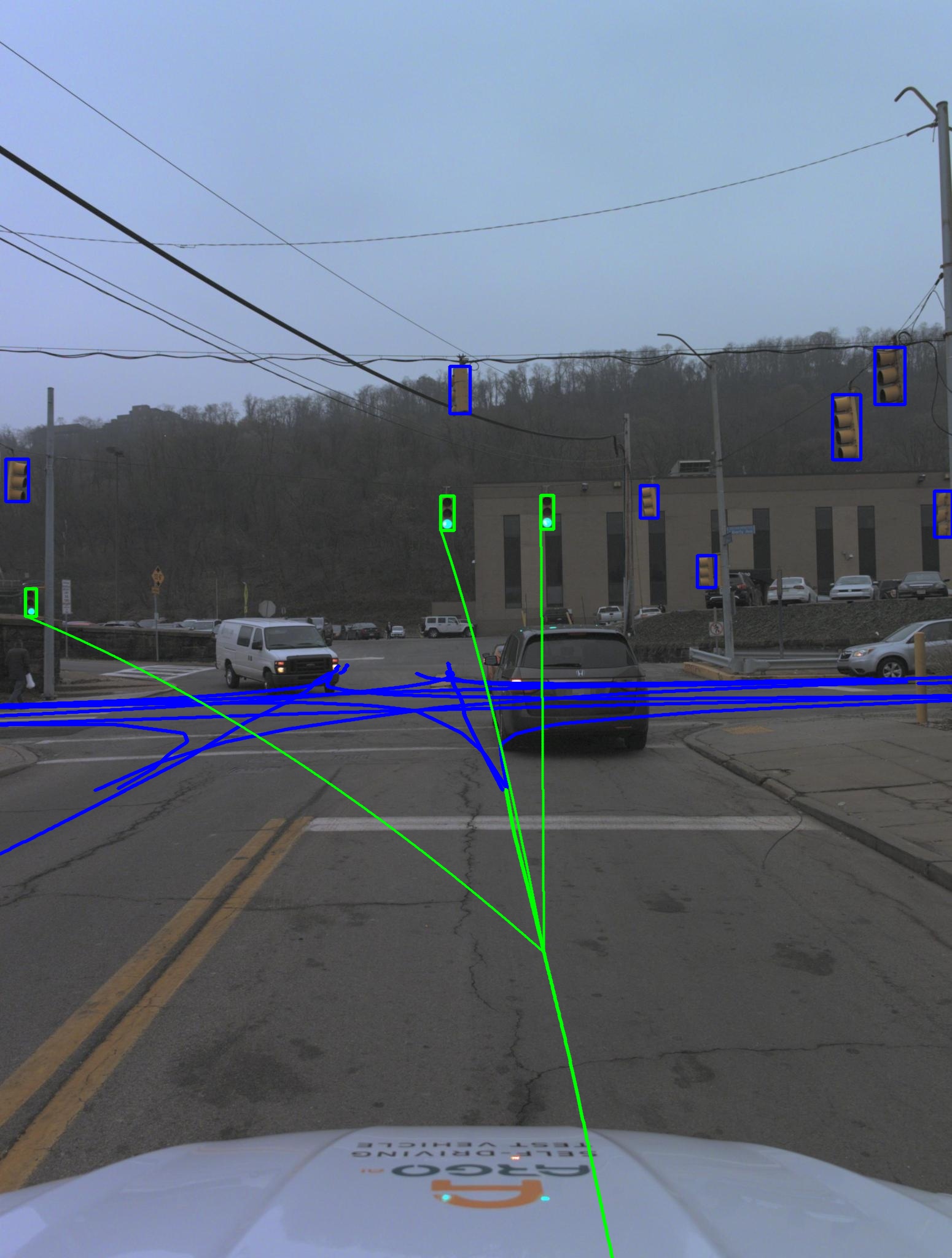}
  \caption{\centering Topo2D}
\end{subfigure}\hfill
\begin{subfigure}[t]{0.235\textwidth}
  \centering\includegraphics[width=\linewidth]{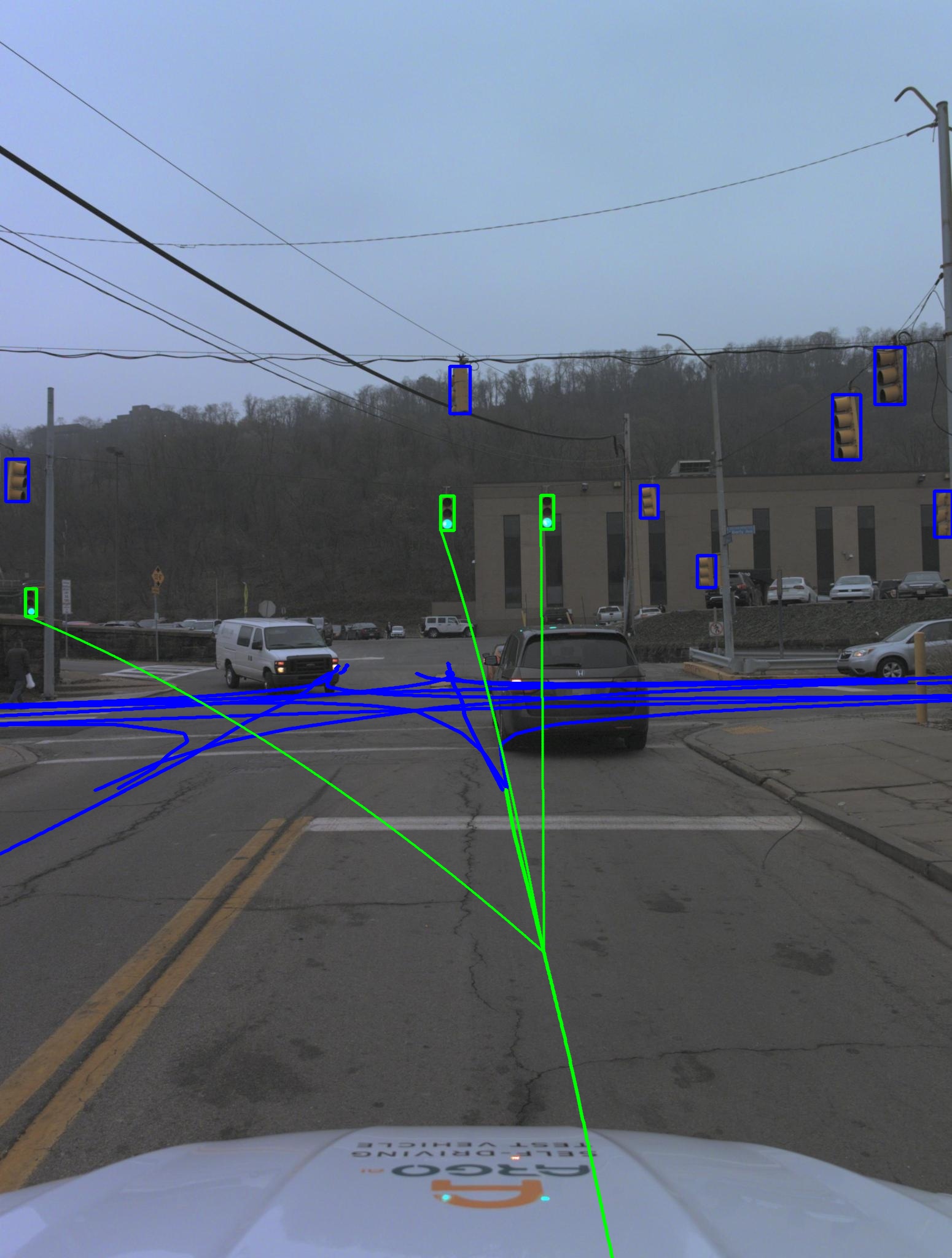}
  \caption{\centering Topo2D + \ourmodel}
\end{subfigure}\hfill
\begin{subfigure}[t]{0.235\textwidth}
  \centering\includegraphics[width=\linewidth]{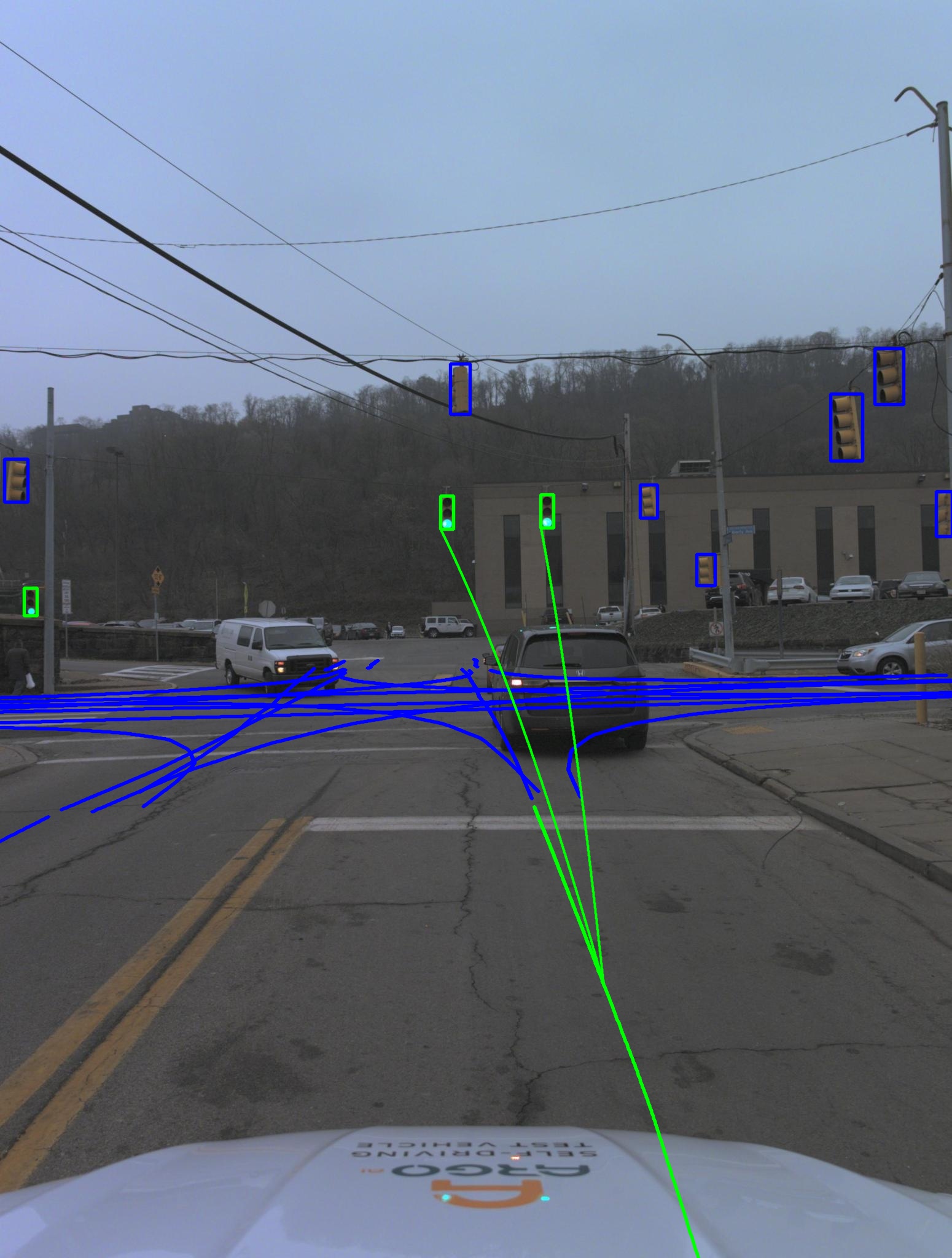}
  \caption{TopoLogic}
\end{subfigure}\hfill
\begin{subfigure}[t]{0.235\textwidth}
  \centering\includegraphics[width=\linewidth]{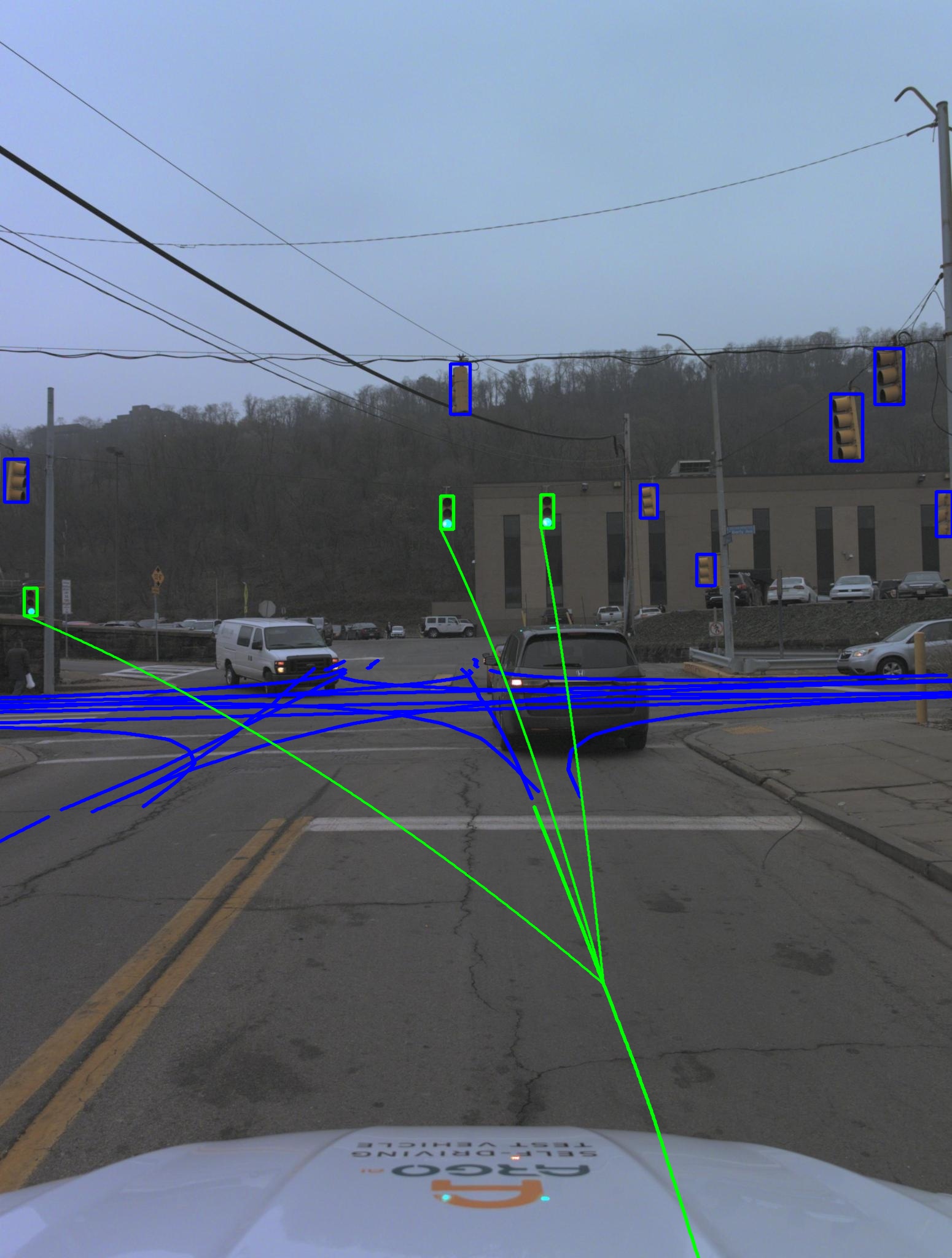}
  \caption{\centering TopoLogic + \ourmodel}
\end{subfigure}

\vspace{6pt}

\begin{subfigure}[t]{0.235\textwidth}
  \centering\includegraphics[width=\linewidth]{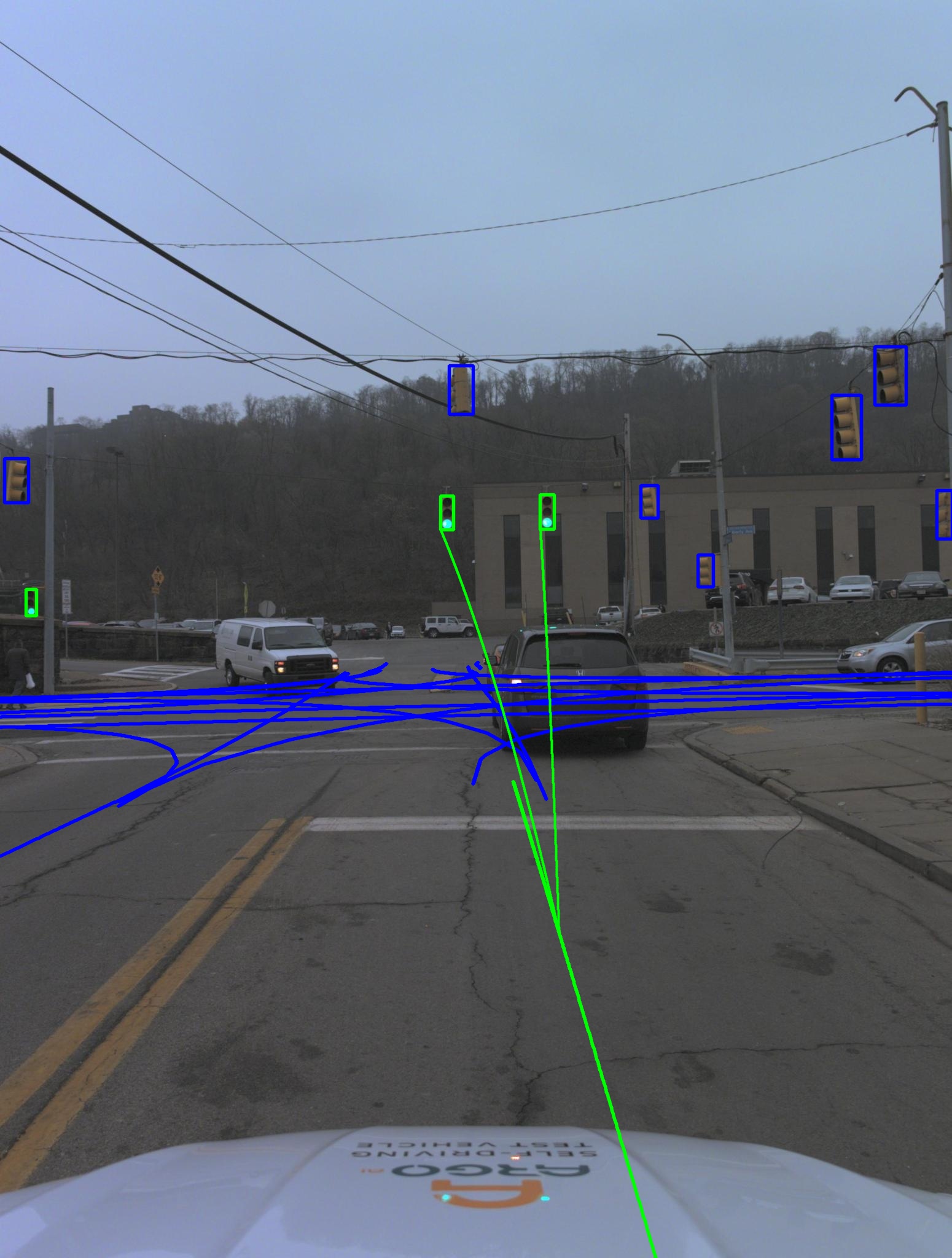}
  \caption{TopoMLP}
\end{subfigure}\hfill
\begin{subfigure}[t]{0.235\textwidth}
  \centering\includegraphics[width=\linewidth]{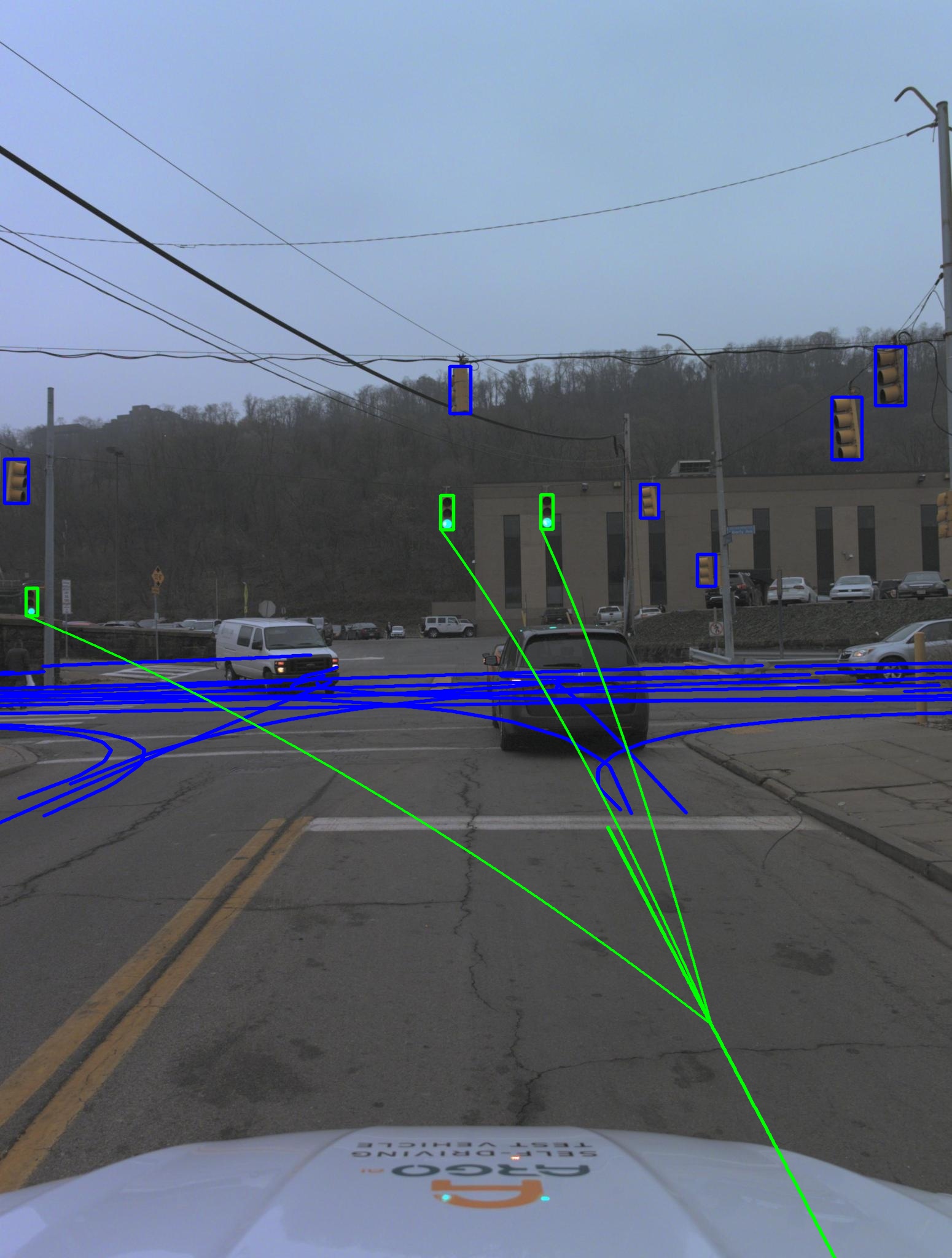}
  \caption{\centering TopoMLP + \ourmodel}
\end{subfigure}\hfill
\begin{subfigure}[t]{0.235\textwidth}
  \centering\includegraphics[width=\linewidth]{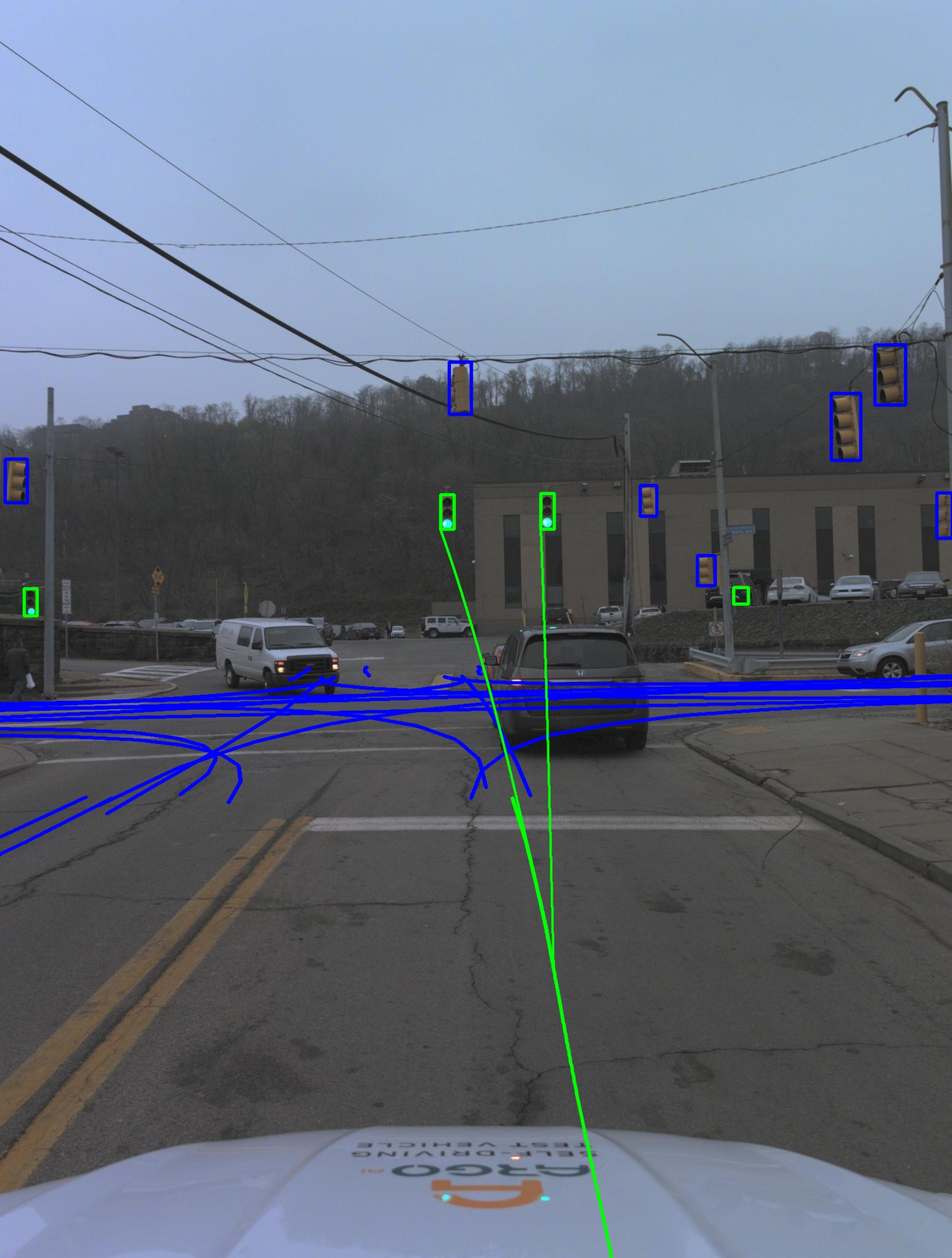}
  \caption{TopoNet}
\end{subfigure}\hfill
\begin{subfigure}[t]{0.235\textwidth}
  \centering\includegraphics[width=\linewidth]{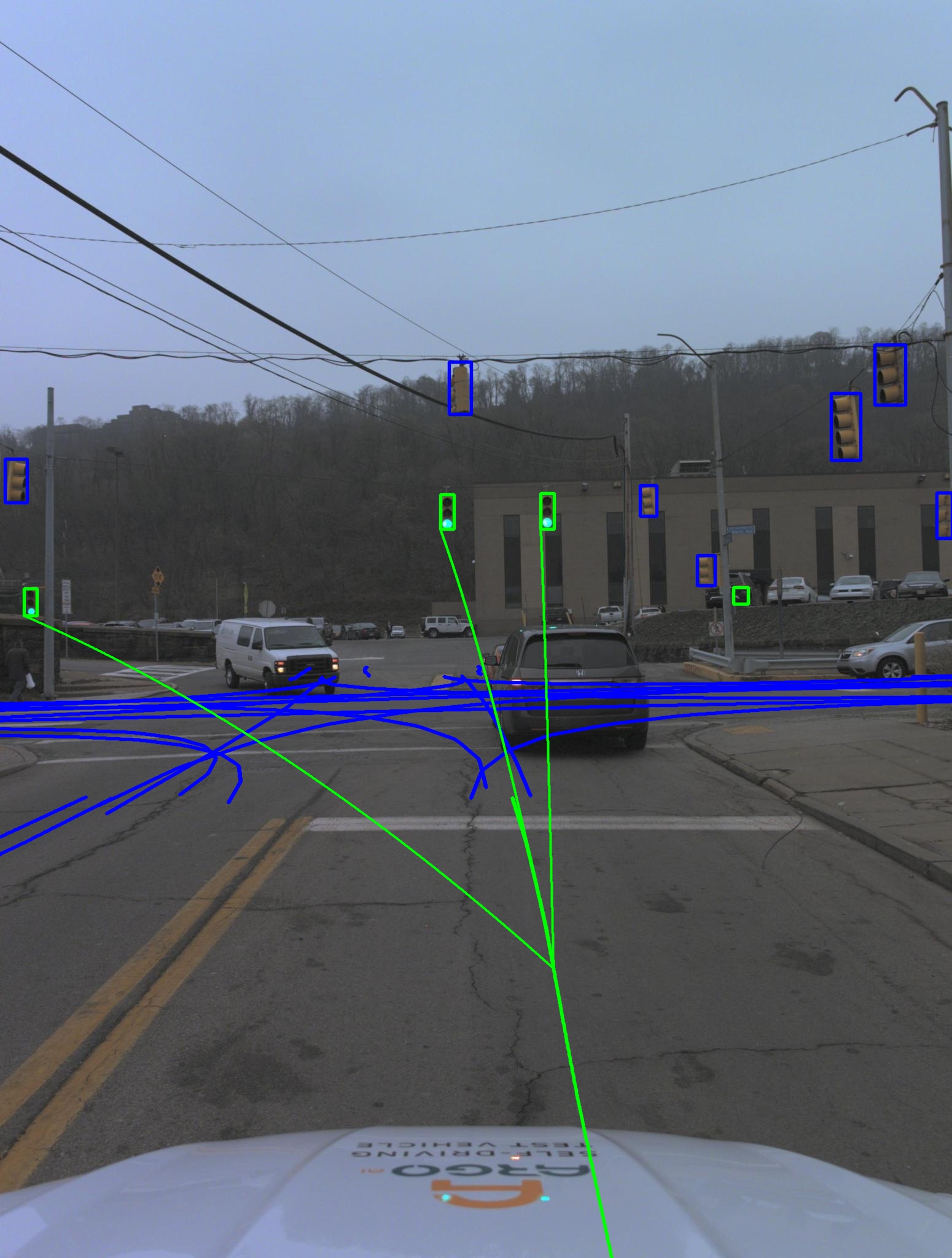}
  \caption{\centering TopoNet + \ourmodel}
\end{subfigure}

\caption{Qualitative results for scene 10021: predictions before and after refinement with \ourmodel across baselines (SMART, Topo2D, TopoLogic, TopoMLP, TopoNet). All panels correspond to the same ground-truth view (left).}
\label{fig:appx_qual_10021}
\end{figure*}

\begin{figure*}[htbp]
\centering
\captionsetup[subfigure]{labelformat=empty}
\begin{subfigure}[t]{0.31\textwidth}
  \centering\includegraphics[width=\linewidth]{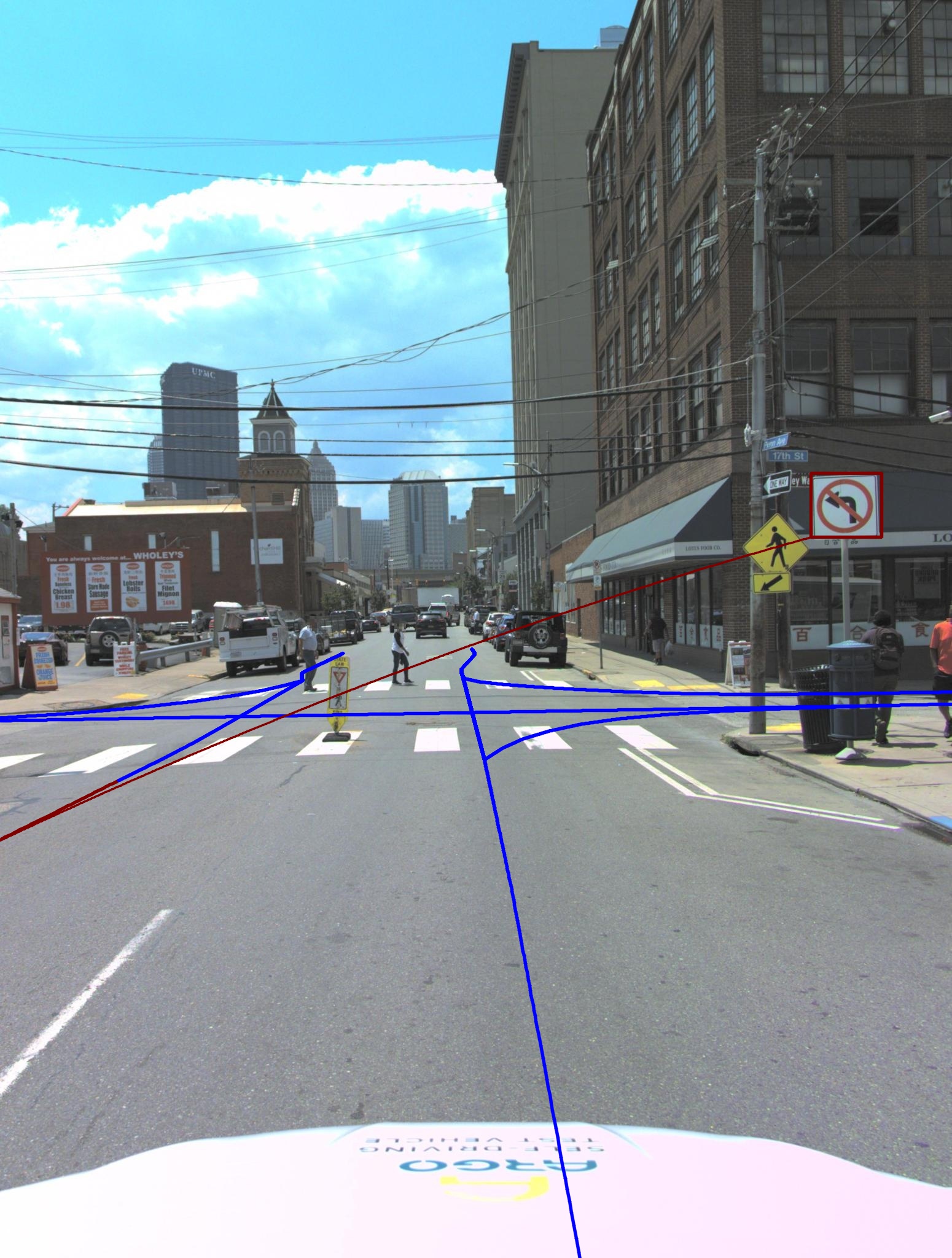}
  \caption{Ground-truth}
\end{subfigure}\hfill
\begin{subfigure}[t]{0.31\textwidth}
  \centering\includegraphics[width=\linewidth]{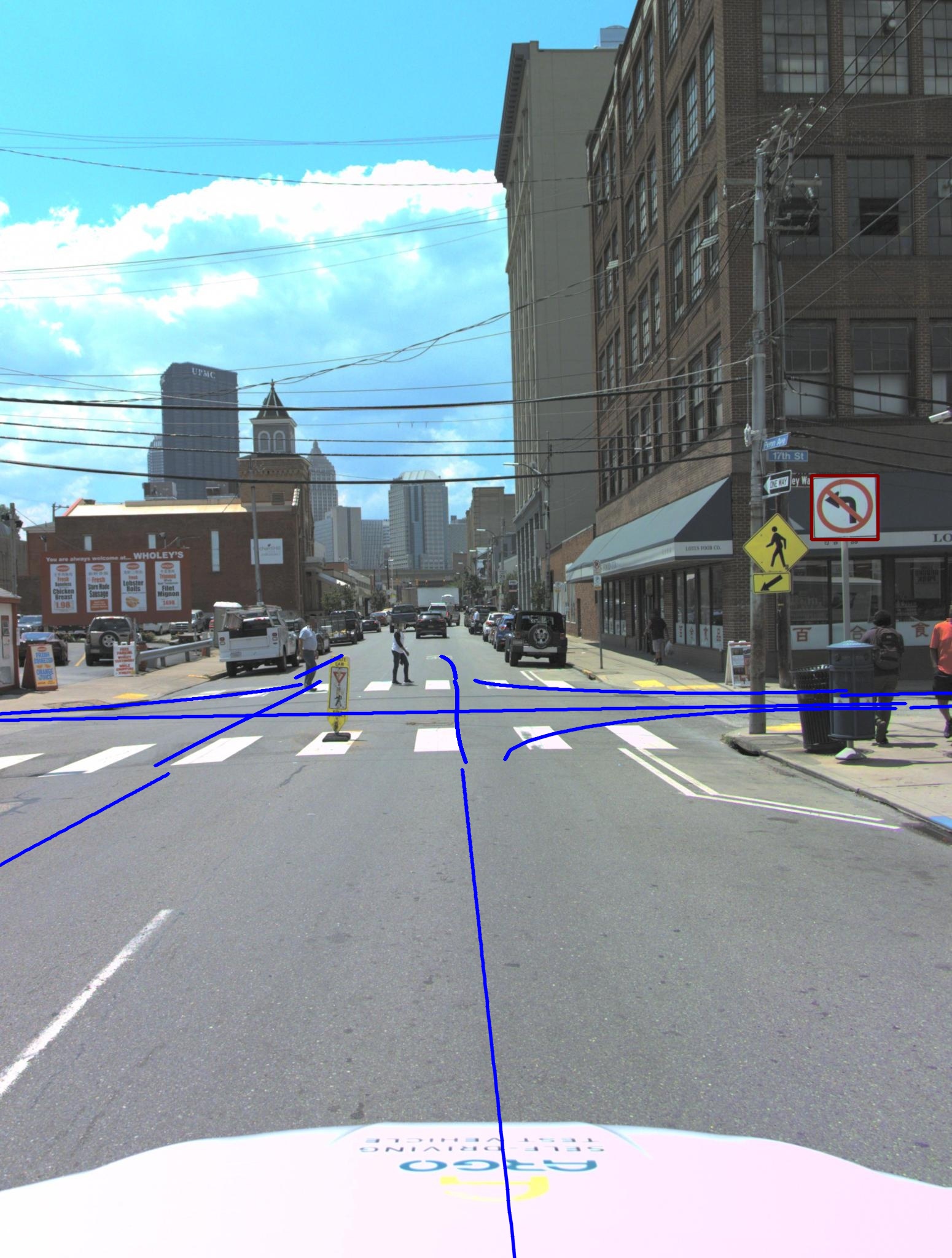}
  \caption{SMART}
\end{subfigure}\hfill
\begin{subfigure}[t]{0.31\textwidth}
  \centering\includegraphics[width=\linewidth]{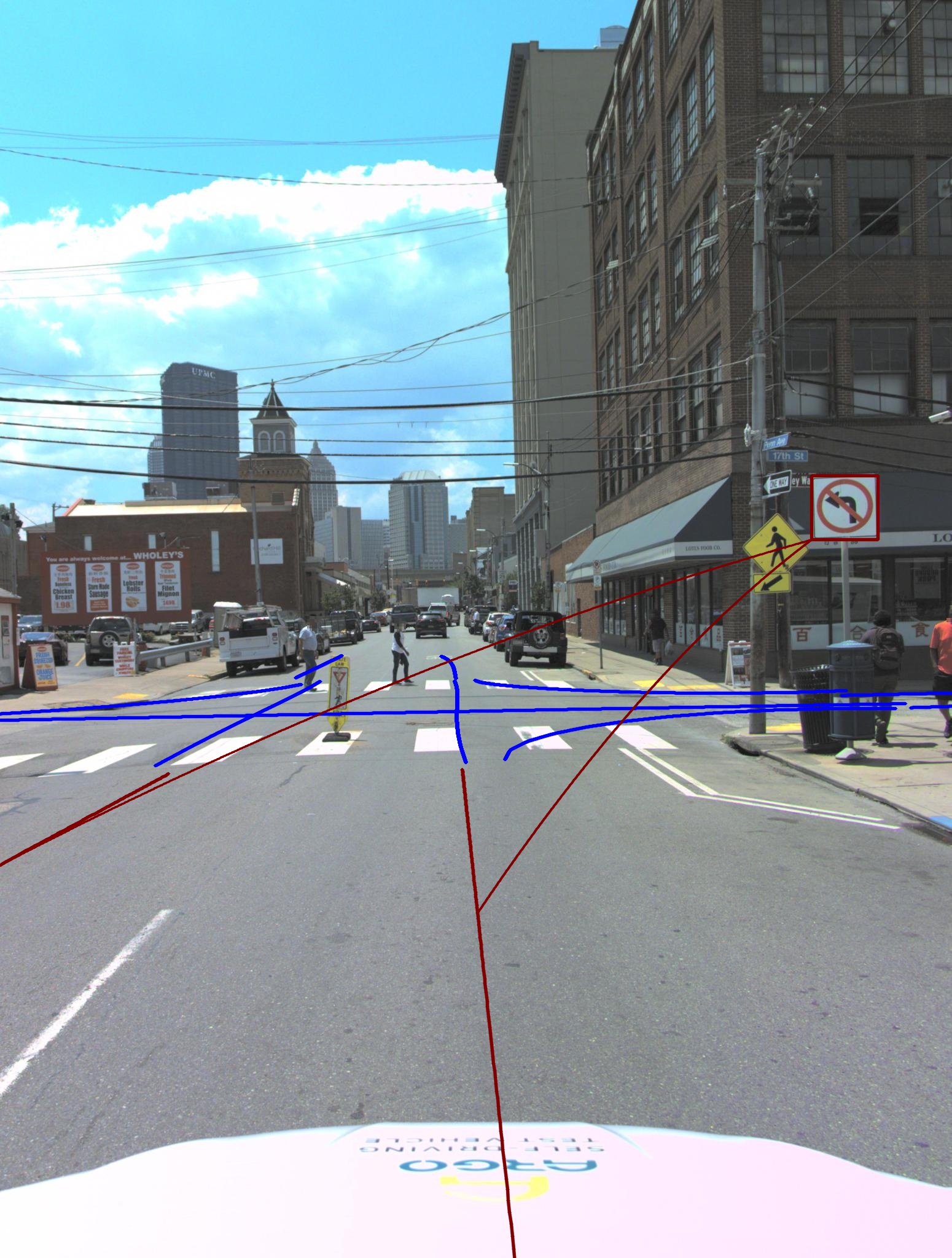}
  \caption{SMART + \ourmodel}
\end{subfigure}

\vspace{6pt}

\begin{subfigure}[t]{0.235\textwidth}
  \centering\includegraphics[width=\linewidth]{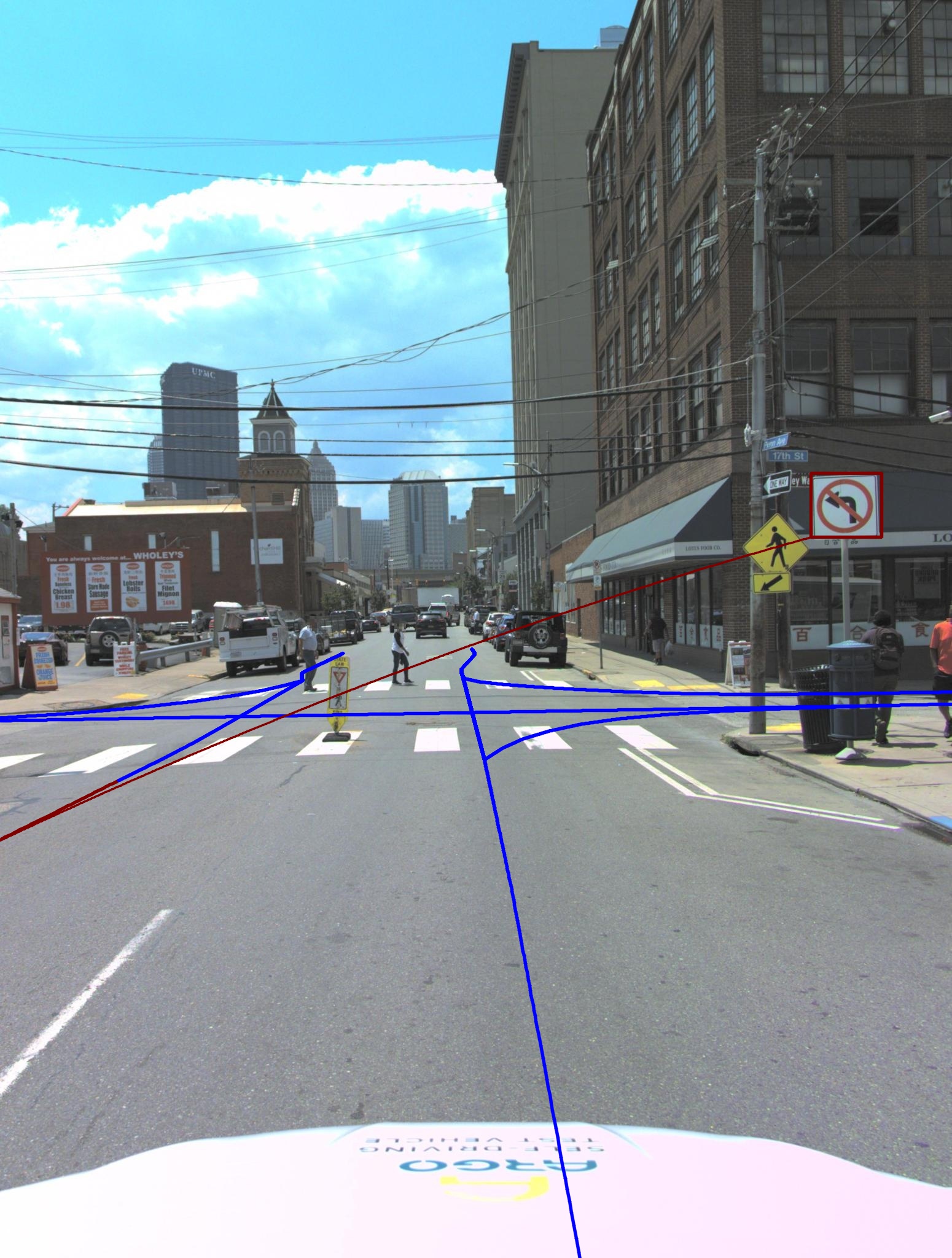}
  \caption{Topo2D}
\end{subfigure}\hfill
\begin{subfigure}[t]{0.235\textwidth}
  \centering\includegraphics[width=\linewidth]{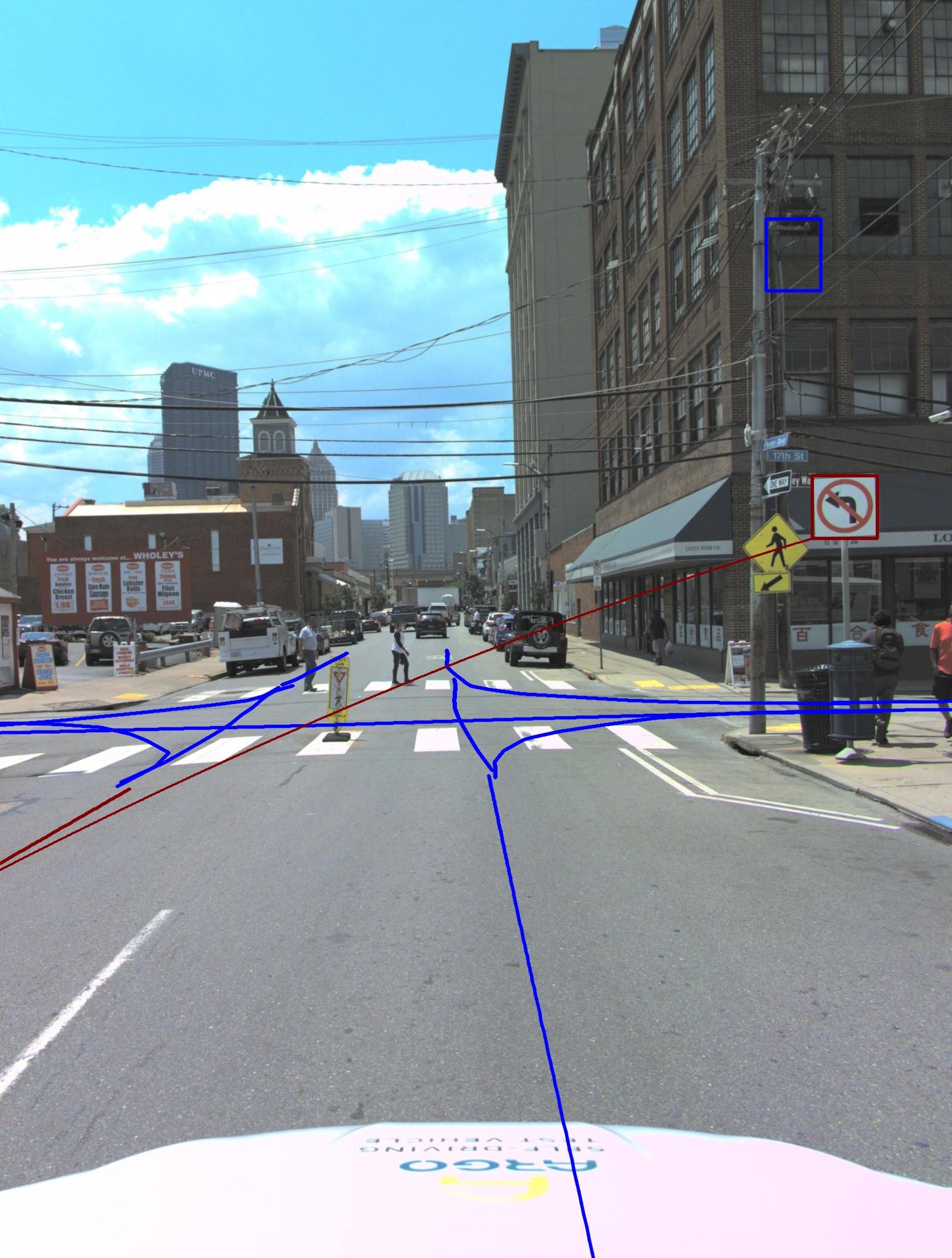}
  \caption{\centering Topo2D + \ourmodel}
\end{subfigure}\hfill
\begin{subfigure}[t]{0.235\textwidth}
  \centering\includegraphics[width=\linewidth]{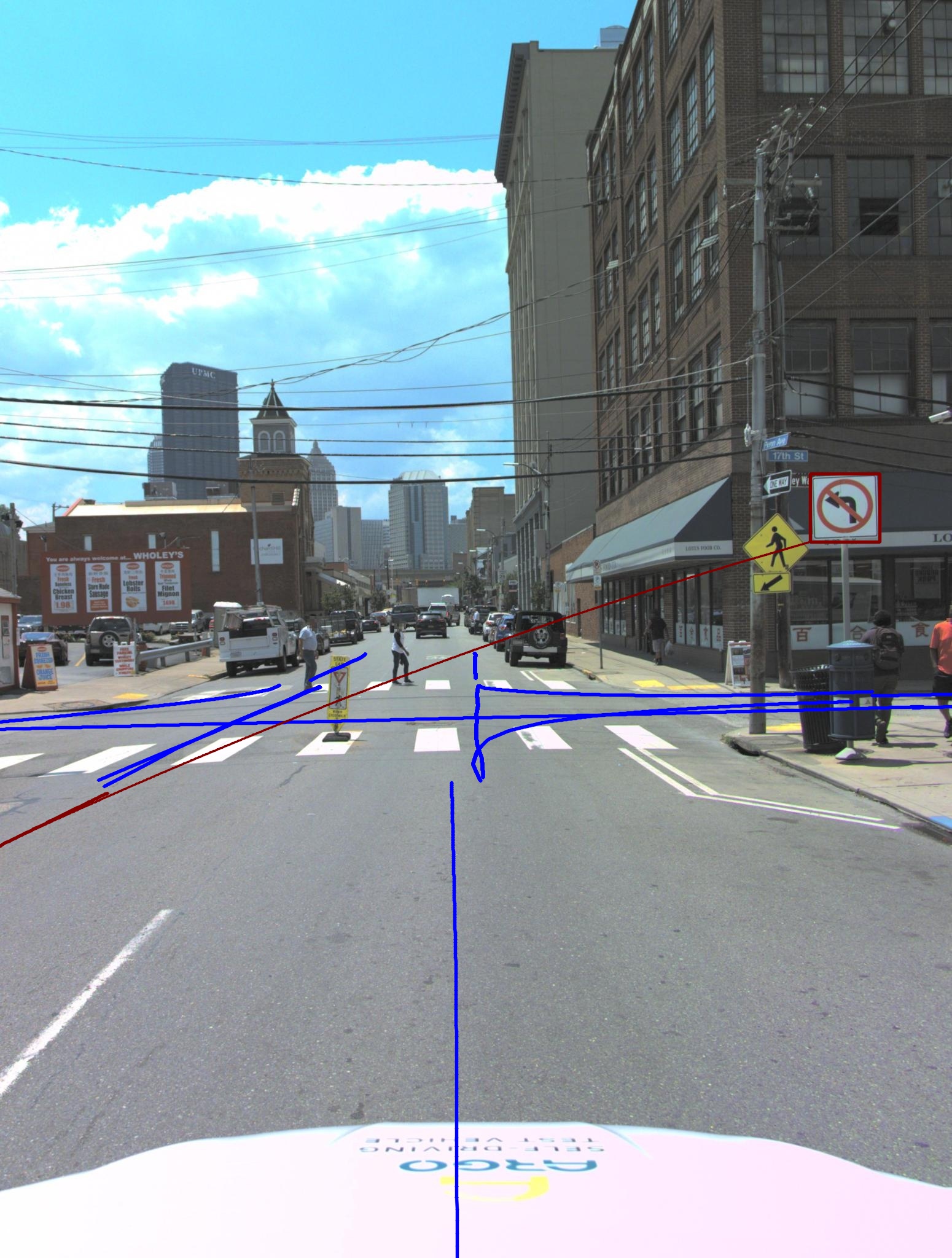}
  \caption{TopoLogic}
\end{subfigure}\hfill
\begin{subfigure}[t]{0.235\textwidth}
  \centering\includegraphics[width=\linewidth]{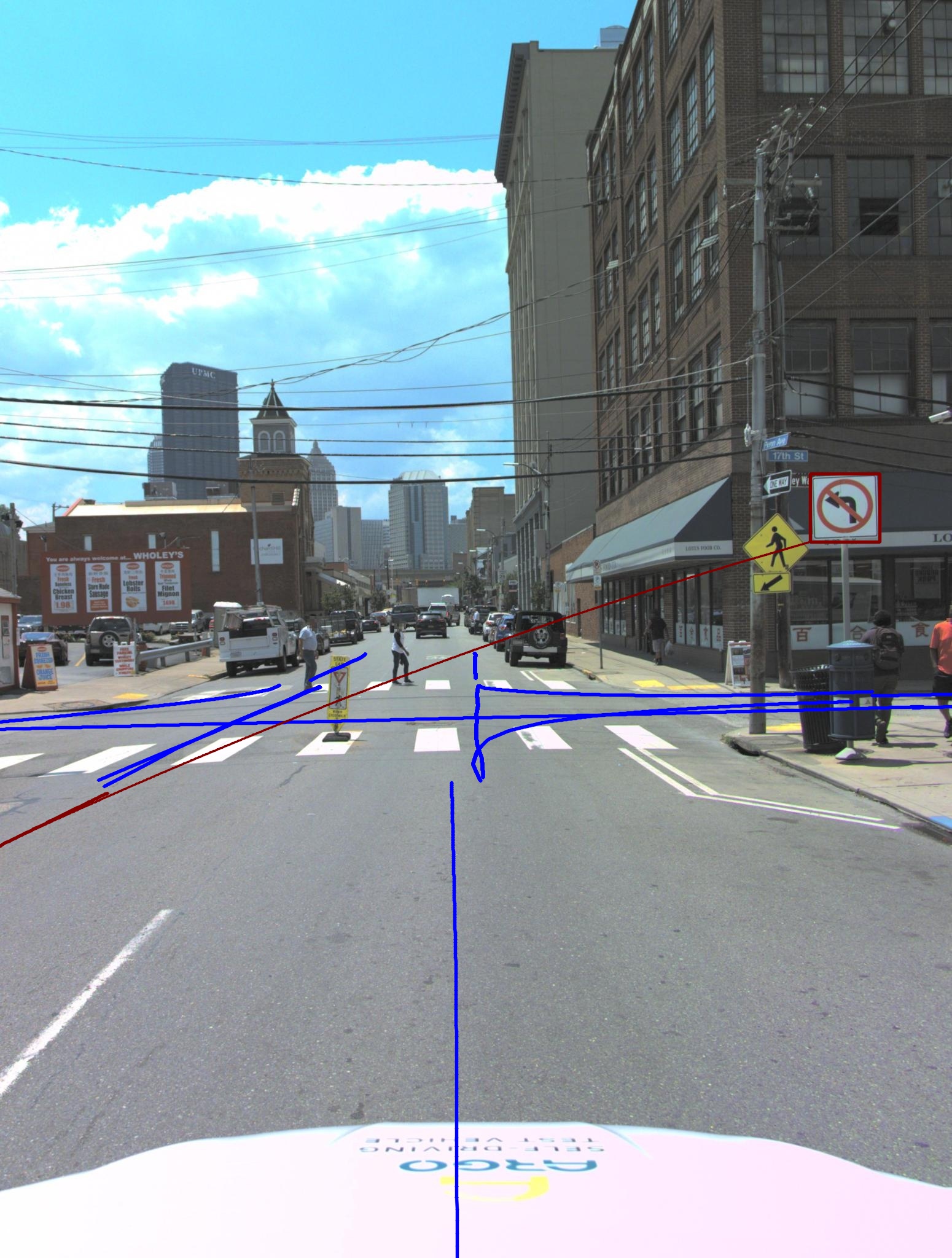}
  \caption{\centering TopoLogic + \ourmodel}
\end{subfigure}

\vspace{6pt}

\begin{subfigure}[t]{0.235\textwidth}
  \centering\includegraphics[width=\linewidth]{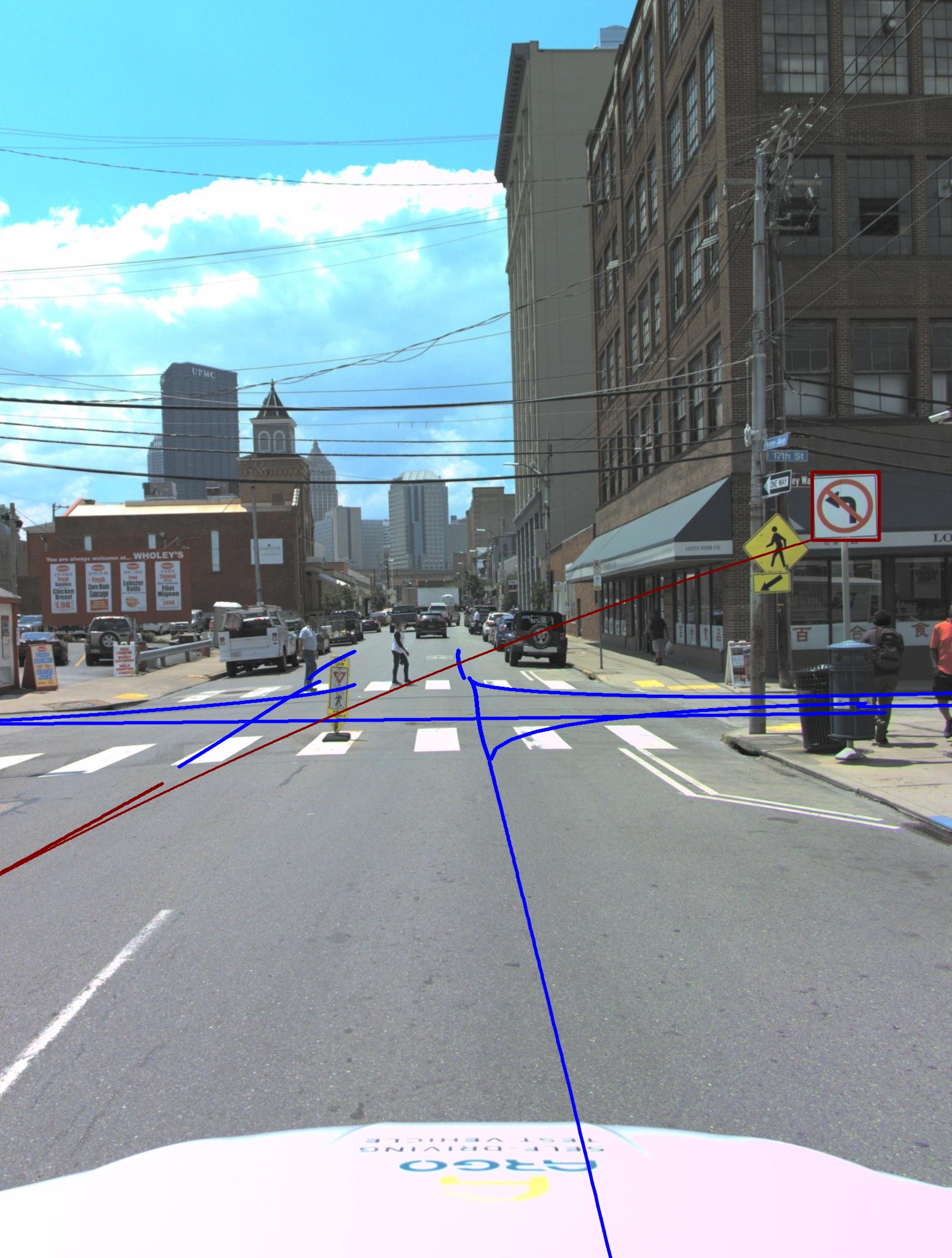}
  \caption{TopoMLP}
\end{subfigure}\hfill
\begin{subfigure}[t]{0.235\textwidth}
  \centering\includegraphics[width=\linewidth]{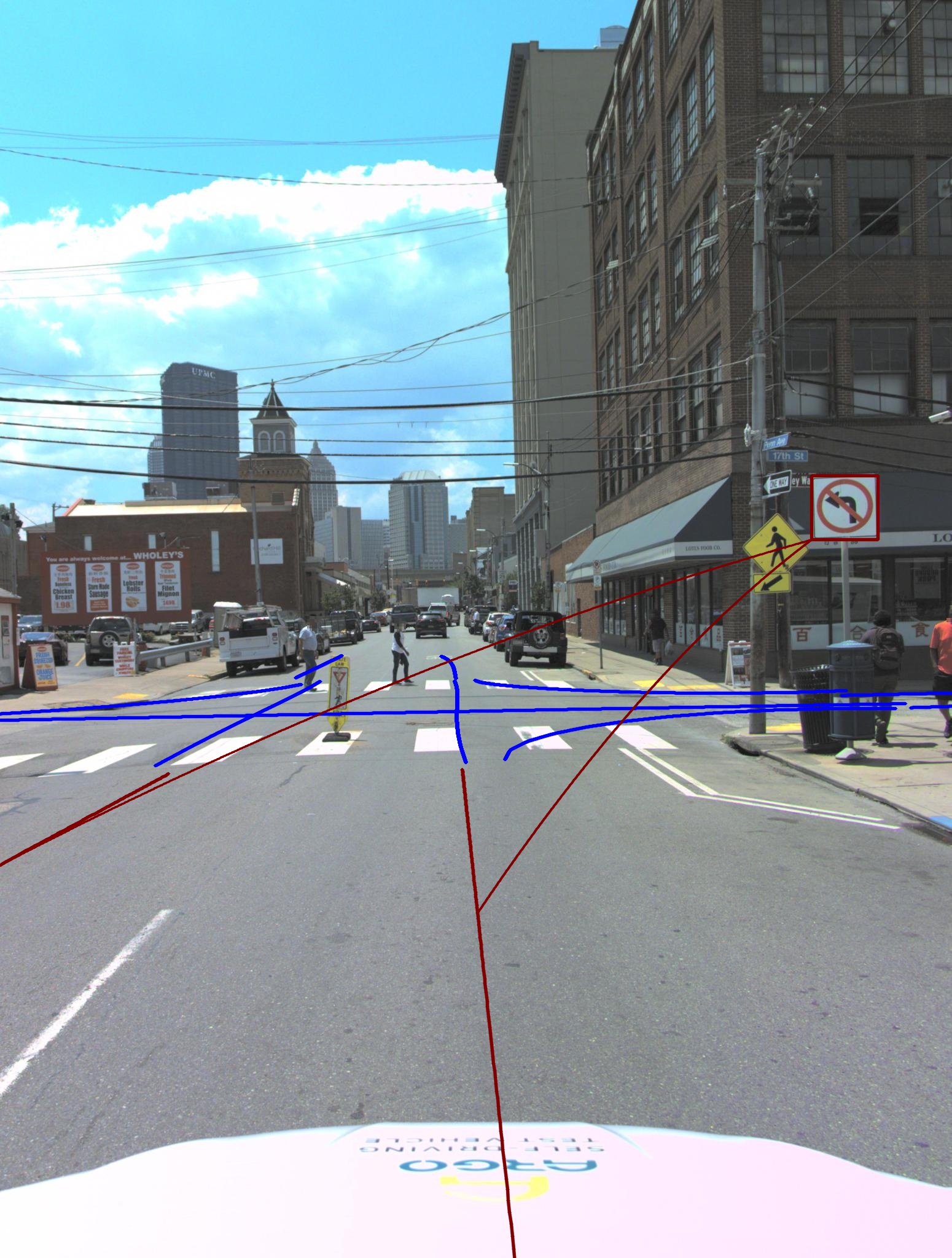}
  \caption{\centering TopoMLP + \ourmodel}
\end{subfigure}\hfill
\begin{subfigure}[t]{0.235\textwidth}
  \centering\includegraphics[width=\linewidth]{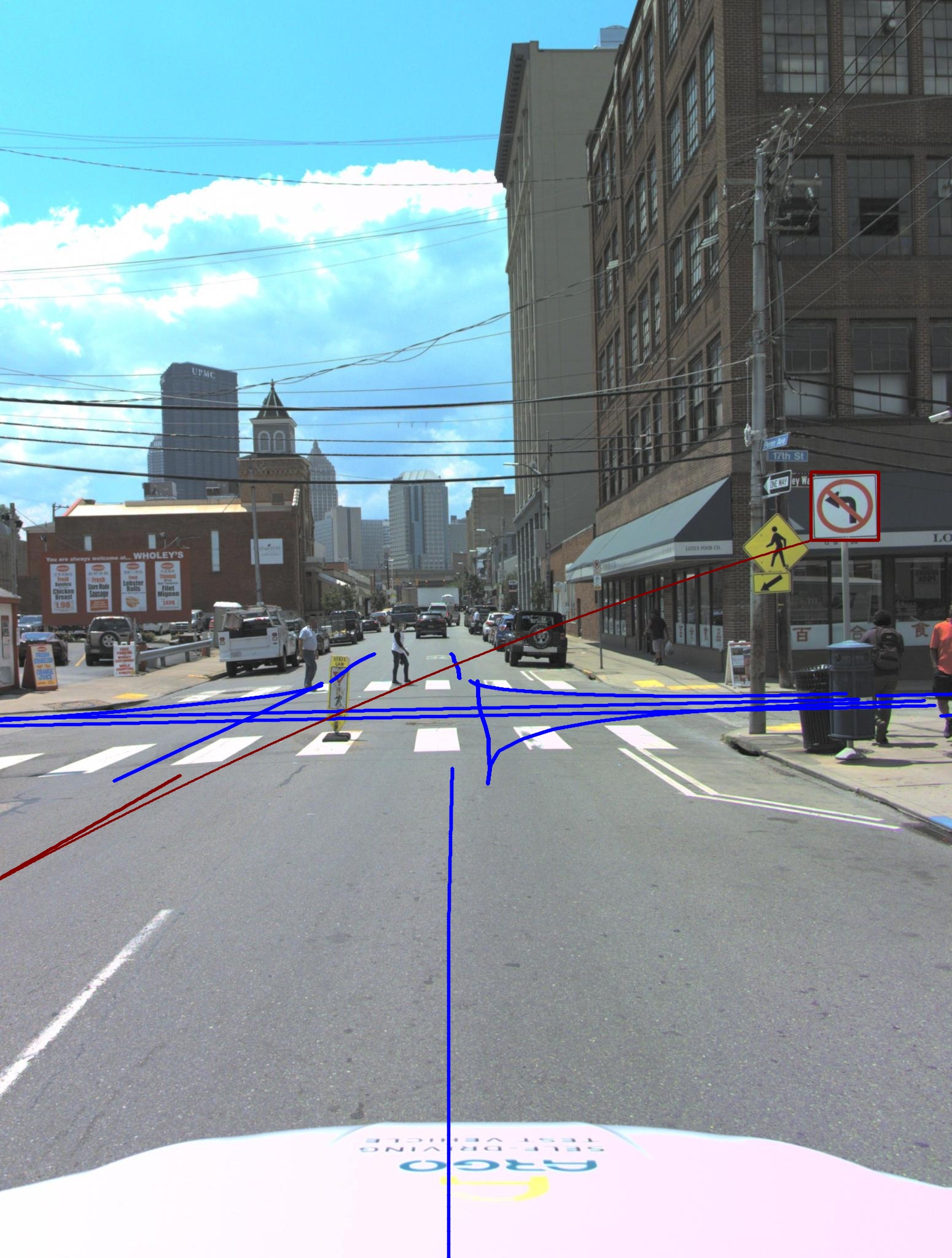}
  \caption{TopoNet}
\end{subfigure}\hfill
\begin{subfigure}[t]{0.235\textwidth}
  \centering\includegraphics[width=\linewidth]{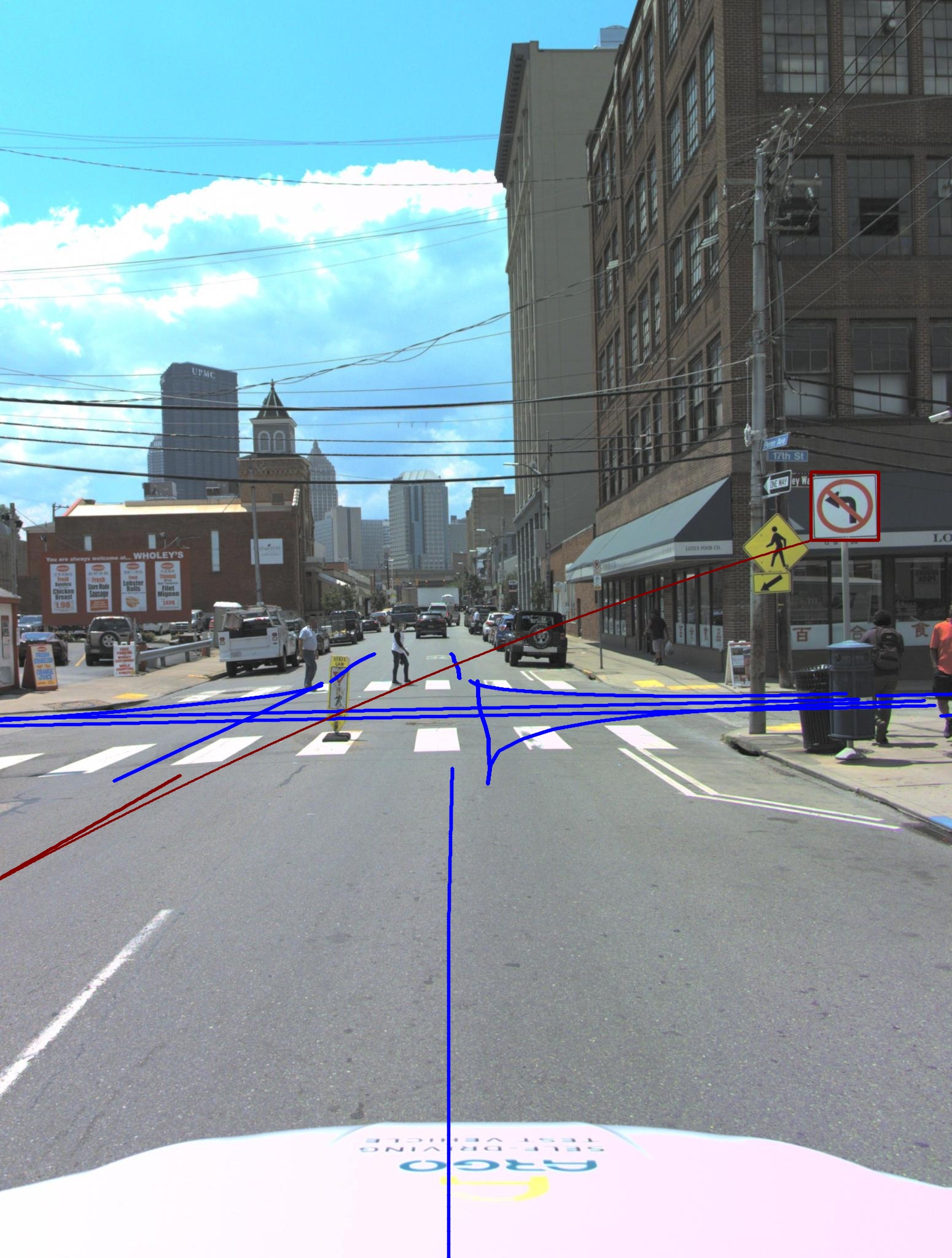}
  \caption{\centering TopoNet + \ourmodel}
\end{subfigure}

\caption{Qualitative results for scene 10030: predictions before and after refinement with \ourmodel across baselines.}
\label{fig:appx_qual_10030}
\end{figure*}

\end{document}